\documentclass[11pt]{amsart}
\usepackage{natbib}
\usepackage[paperwidth=169mm,paperheight=239mm,left=36mm,right=38mm,top=40mm,bottom=50mm,headsep=10mm]{geometry}
\usepackage{graphicx}
\usepackage{amssymb}
\usepackage{amsmath}
\usepackage{natbib}
\usepackage{amsthm}
\usepackage{aliascnt}
\usepackage{amsfonts}
\usepackage{epstopdf}
\usepackage[linesnumbered,%ruled,
vlined]{algorithm2e}
\usepackage{xargs}
\usepackage{bbm}
\usepackage{mathtools}
\usepackage{stmaryrd}
\usepackage{ifthen}
\usepackage{float}
\usepackage{upgreek}
\usepackage{mathrsfs}
\usepackage[textwidth=4cm, textsize=footnotesize]{todonotes}
\usepackage{hyperref}
\usepackage{cleveref}
\usepackage{subcaption}
\title[Graphical posterior predictive classifier]
{Graphical posterior predictive classifier:  Bayesian model averaging with particle Gibbs}
\author{Tatjana Pavlenko and Felix Leopoldo Rios}
\date{}

\newtheorem{theorem}{Theorem}
\newtheorem{proposition}{Proposition}
\newaliascnt{definition}{theorem}
\newtheorem{definition}[definition]{Definition}
\aliascntresetthe{definition}

\newaliascnt{lemma}{theorem}

\aliascntresetthe{lemma}
\crefname{lemma}{lemma}{lemmas}
\Crefname{Lemma}{Lemma}{Lemmas}

\newaliascnt{remark}{theorem}

\aliascntresetthe{remark}

\newcommand\ci{\perp\!\!\!\perp}
\newcommand{\gen}[1]{J_{#1}}
\newcommandx{\bk}[2][1=]{
\ifthenelse{\equal{#1}{}}
{\kernel{R}_{#2}}
{\kernel{R}_{#2} \langle #1 \rangle}
}

\newcommand{\class}{c}
\newcommand{\cliqueletter}{Q}
\newcommand{\clset}[1]{
\ifthenelse{\equal{#1}{}}
{\mathcal{\cliqueletter}}
{\mathcal{\cliqueletter}(#1)}
}

\newcommandx{\comb}[2][1=]{
\ifthenelse{\equal{#1}{}}
{z_{#2}}
{z_{#1 \mid #2}}
}
\newcommand{\combkernel}[1]{\kernel{\Psi}_{#1}}
\newcommand{\combkernelpath}[1]{\bar{\kernel{\Psi}}_{#1}}
\newcommand{\combmeas}[1]{\psi_{#1}}
\newcommand{\combpart}[2]{z_{#1}^{#2}}
\newcommand{\combsp}[1]{\spc{Z}_{#1}}

\newcommandx{\ct}[1][1=]{\ifthenelse{\equal{#1}{}}{\mathsf{I}}{\mathsf{I}_{#1}}}

\newcommand{\disc}{\mathsf{Pr}}

\newcommandx{\df}[1][1=]{\ifthenelse{\equal{#1}{}}{\delta}{\delta_{#1}}}

\newcommand{\epart}[2]{\xi_{#1}^{#2}}
\newcommand{\eqdef}{\vcentcolon=}

\newcommand{\giwsymb}{\mathcal{NIW}}
\newcommandx{\graph}[1][1=]{\ifthenelse{\equal{#1}{}}{G}{G_{#1}}}
\newcommandx{\graphedgeset}[1][1=]{\ifthenelse{\equal{#1}{}}{E}{E_{#1}}}
\newcommandx{\graphfd}[1][1=]{\ifthenelse{\equal{#1}{}}{\mathcal{G}}{\mathcal{G}_{#1}}}

\newcommandx{\graphnodeset}[1][1=]{\ifthenelse{\equal{#1}{}}{V}{V_{#1}}}
\newcommandx{\graphsp}[1][1=]{\ifthenelse{\equal{#1}{}}{\mathcal{G}}{\mathcal{G}_{#1}}}
\newcommandx{\graphtarg}[1][1=]{\ifthenelse{\equal{#1}{}}{\graphtargletter^\star}{\graphtargletter^\star \langle #1 \rangle}}
\newcommand{\graphtargletter}{\eta}

\newcommand{\hyperparamletter}{\vartheta}
\newcommandx{\hyperparam}[2][1=]{\ifthenelse{\equal{#1}{}}{\hyperparamletter_{#2}}{\hyperparamletter'_{#2}(#1)}}

\newcommandx{\I}[1][1=]{\ifthenelse{\equal{#1}{}}{I}{I^{(#1)}}}
\newcommand{\ind}[2]{I_{#1}^{#2}}
\newcommand{\intvect}[2]{ \{#1,\dots, #2\} }

\newcommandx{\jtnode}[1][1=]{\ifthenelse{\equal{#1}{}}{\cliqueletter}{\cliqueletter_{#1}}}

\newcommand{\jtnodeset}{\mathcal Q}
\newcommand{\kernel}[1]{\mathbf{#1}}

\newcommand{\law}{\mathcal L}

\newcommandx{\meanvec}[1][1=]{\ifthenelse{\equal{#1}{}}{{\mathbf m}}{{\mathbf m}_{#1}}}

\newcommand{\N}{N}
\newcommand{\Nmcmc}{M}
\newcommandx{\nckernel}[2][1=]{
\ifthenelse{\equal{#1}{}}
{\boldsymbol{\Gamma}_{#2}}
{\boldsymbol{\Gamma}_{#2} \langle #1 \rangle}
}
\newcommand{\nclasses}{k}
\newcommand{\normsymb}{\mathcal N}
\newcommand{\normdenssymb}[3]{{\mathcal N}}

\newcommandx{\niwnloc}[1][1=]{\ifthenelse{\equal{#1}{}}{\alpha}{\alpha_{#1}}}
\newcommandx{\niwnscalefactor}[1][1=]{\ifthenelse{\equal{#1}{}}{\beta}{\beta[#1]}}
\newcommandx{\niwiwloc}[1][1=]{\ifthenelse{\equal{#1}{}}{\Phi}{\Phi[#1]}}
\newcommandx{\niwiwdf}[1][1=]{\ifthenelse{\equal{#1}{}}{\delta}{\delta[#1]}}

\newcommandx{\niwnlocpred}[1][1=]{\ifthenelse{\equal{#1}{}}{\alpha}{\alpha_{#1}}}
\newcommandx{\niwnscalefactorpred}[1][1=]{\ifthenelse{\equal{#1}{}}{\beta}{\beta[#1]}}
\newcommandx{\niwiwlocpred}[1][1=]{\ifthenelse{\equal{#1}{}}{\Phi}{\Phi[#1]}}
\newcommandx{\niwiwdfpred}[1][1=]{\ifthenelse{\equal{#1}{}}{\delta}{\delta[#1]}}

\newcommand{\niwnmu}{{\boldsymbol \mu}}
\newcommand{\niwnnu}{\nu}
\newcommand{\niwiwtau}{{\boldsymbol \Phi}}
\newcommand{\niwiwalpha}{\vartheta}

\newcommand{\niwnnust}{\nu + n}
\newcommand{\niwnmust}{ \frac{\niwnnu\niwnmu + n \testdatameanest }{ \niwnnustsymb } }
\newcommand{\niwiwtaust}{\niwiwtau + \sumsq + \frac{\niwnnu n}{\niwnnustsymb}( \niwnmu - \testdatameanest)( \niwnmu - \testdatameanest)'}
\newcommand{\niwiwalphast}{\niwiwalpha +n }

\newcommand{\niwnnustsymb}{\niwnnu^\ast}
\newcommandx{\niwnmustsymb}[1][1=]{\ifthenelse{\equal{#1}{}}{\niwnmu^\ast}{\niwnmu_{#1}^\ast}}
\newcommandx{\niwiwtaustsymb}[1][1=]{\ifthenelse{\equal{#1}{}}{\niwiwtau^\ast}{\niwiwtau_{#1}^\ast}}

\newcommand{\niwiwalphastsymb}{\niwiwalpha^\ast}
\newcommand{\niwconstsymb}{\kappa}

\newcommand{\ntrees}[1]{\mu(#1)}
\newcommand{\1}{\mathbbm{1}}
\newcommand{\p}{p}
\newcommandx{\param}[1][1=]{\ifthenelse{\equal{#1}{}}{\theta}{\theta_{#1}}}
\newcommandx{\paramd}[1][1=]{\ifthenelse{\equal{#1}{}}{\theta}{\theta(#1)}}

\newcommand{\paramsp}{\Theta}

\newcommandx{\partarg}[2][1=]{
    \ifthenelse{\equal{#1}{}}
    {\eta^{\N}_{#2}}
    {\eta_{#2}^{\ast, \N}}
}
\newcommandx{\partinit}[1][1=]{
       \ifthenelse{\equal{#1}{}}
       {\kappa}
       {\kappa^\ast \langle #1 \rangle}
}
\newcommandx{\perm}[2][2=]{
\ifthenelse{\equal{#2}{}}
{s_{#1}}
{s_{#1}(#2)}
}
\newcommand{\PG}[1]{\kernel{P}_{#1}^\N}

\newcommandx{\precmatrestr}[1][1=]{\ifthenelse{\equal{#1}{}}{\mathcal P}{\mathcal P_{#1}}}
\newcommandx{\covmat}[1][1=]{\ifthenelse{\equal{#1}{}}{\mathbf \Sigma}{{\mathbf{\Sigma}_{#1}}}}

\newcommandx{\covmatest}[1][1=]{\ifthenelse{\equal{#1}{}}{\mathbf W}{\mathbf W^{(#1)}}}

\newcommandx{\prop}[2][1=]{
\ifthenelse{\equal{#1}{}}
{\kernel{\propletter}_{#2}}
{\kernel{\propletter}_{#2}^\ast \langle #1 \rangle}
}
\newcommand{\propletter}{K}
\newcommandx{\proppath}[2][1=]{
\ifthenelse{\equal{#1}{}}
{\bar{\kernel{\propletter}}_{#2}}
{\bar{\kernel{\propletter}}_{#2} \langle #1 \rangle}
}

\newcommand{\randind}[1]{\iota}

\newcommand{\retrosupp}[1]{\mathsf{S}_{#1}}

\newcommand{\rmd}{d}
\newcommandx{\SMCdist}[2][1=]{
\ifthenelse{\equal{#1}{}}
{\varrho_{#2}^\N}
{\varrho_{#2}^\N \langle #1 \rangle}
}
\newcommandx{\scalmat}[1][1=]{\ifthenelse{\equal{#1}{}}{v}{v_{#1}}}
\newcommandx{\sep}[1][1=]{\ifthenelse{\equal{#1}{}}{S}{S_{#1}}}
\newcommandx{\sepset}[1][1=]{\ifthenelse{\equal{#1}{}}{\mathcal{S}}{\mathcal{S}(#1)}}
\newcommand{\spc}[1]{\mathsf{#1}}
\newcommand{\sumsq}{\mathbf S}

\newcommandx{\targ}[3][1=, 3=]{
   \ifthenelse{\equal{#3}{}}
   {
       \ifthenelse{\equal{#1}{}}
       {\eta_{#2}}
       {\eta^\ast \langle #2 \rangle}
   }
   {
       \ifthenelse{\equal{#1}{}}
       {\eta_{#2} \langle #3 \rangle}
       {\eta_{#2}^\ast \langle #3 \rangle}
   }
}
\newcommandx{\targMCMC}[3][1=]{
\ifthenelse{\equal{#1}{}}
{\eta_{#2}^{\N, #3}}
{\eta^* \langle #2 \rangle^{\N, #3}}
}

\newcommand{\testdatasize}{n}
\newcommandx{\testdataindset}[1][1=]{\ifthenelse{\equal{#1}{}}{N}{n_{#1}}}

\newcommandx{\testdataindsubset}[1][1=]{\ifthenelse{\equal{#1}{}}{s}{s_{#1}}}

\newcommandx{\testdataset}[1][1=]{
\ifthenelse{\equal{#1}{}}
{\mathbf x}
{\mathbf x^{(#1)}}
}
\newcommandx{\testdatapnt}[1][1=]{\ifthenelse{
\equal{#1}{}}
{{\bf x}}
{{\bf x}_{#1}}
}

\newcommandx{\testdatapntrv}[1][1=]{\ifthenelse{\equal{#1}{}}{X}{X_{#1}}}
\newcommandx{\testdatapntsubsetrv}[1][1=]{\ifthenelse{\equal{#1}{}}{\mathbf X}{X_{#1}}}
\newcommandx{\testdatameanest}[1][1=]{\ifthenelse{\equal{#1}{}}{\mathbf {\bar x}}{\mathbf {\bar x}_{#1}}}

\newcommand{\tpreddf}[1]{\niwnnustsymb + 1 - #1}

\newcommand{\tpredprec}[1]{\frac{\niwnnustsymb(\niwnnustsymb+1-#1)}{\niwnnustsymb+1}{\niwiwtaustsymb}^{-1}}

\newcommand{\tpreddfsymb}{\delta^\ast}

\newcommand{\tpredprecsymb}{{\boldsymbol \Upsilon^\ast}}

\newcommand{\tdensconst}[3]{\frac{\Gamma[(#3 + #2)/2]|#1|^{1/2}}{\Gamma(#3/2)(#3\pi)^{#2/2}}}
\newcommand{\tdens}[5]{\tdensconst{#1}{#2}{#3}  \bigg[1 + (#4 - #5)'#1(#4 - #5) \bigg]^{-(#3 + #2)/2} }
\newcommand{\tdenssymb}{{\mathcal T}}

\newcommandx{\trdataindsubset}[1][1=]{\ifthenelse{\equal{#1}{}}{t}{t_{#1}}}

\newcommandx{\trdataset}[1][1=]{\ifthenelse{\equal{#1}{}}{\mathbf z}{\mathbf z^{(#1)}}}
\newcommandx{\trdatapnt}[1][1=]{\ifthenelse{\equal{#1}{}}{\mathbf z}{\mathbf z_{#1}}}
\newcommandx{\trdatameanest}[1][1=]{\ifthenelse{\equal{#1}{}}{\mathbf {\bar \trdatapnt}}{\mathbf {\bar \trdatapnt}_{#1}}}

\newcommandx{\tree}[1][1=]{\ifthenelse{\equal{#1}{}}{T}{T_{#1}}}
\newcommand{\treepart}[2]{T_{#1}^{#2}}
\newcommandx{\treerv}[1][1=]{\ifthenelse{\equal{#1}{}}{\mathcal T}{\mathcal T_{#1}}}
\newcommandx{\tr}[2][2=]{
\ifthenelse{\equal{#2}{}}
{\tilde{x}_{#1}}
{\tilde{x}_{#1}^{#2}}
}

\newcommandx{\trex}[2][2=]{
\ifthenelse{\equal{#2}{}}
{
x_{#1}
}
{x_{#1}^{#2}}
}

\newcommandx{\foo}[2][1=, 2=]{#1,#2}

\newcommandx{\trfd}[1][1=]{\ifthenelse{\equal{#1}{}}{\mathcal{T}}{\mathcal{T}_{#1}}}
\newcommand{\trgr}{g}
\newcommandx{\trsp}[1][1=]{\ifthenelse{\equal{#1}{}}{\mathcal{T}}{\mathcal{T}_{#1}}}
\newcommandx{\trtarg}[1][1=]{\ifthenelse{\equal{#1}{}}{\eta^\ast}{\eta^\ast \langle #1 \rangle}}

\newcommandx{\uk}[2][1=]{
\ifthenelse{\equal{#1}{}}
{\kernel{Q}_{#2}}
{\kernel{Q}_{#2} \langle #1 \rangle}
}
\newcommandx{\ungraphtarg}[1][1=]{\ifthenelse{\equal{#1}{}}{\gamma^\star}{\gamma^\star \langle #1 \rangle}}
\newcommandx{\untarg}[3][1=, 3=]{
   \ifthenelse{\equal{#3}{}}
   {
       \ifthenelse{\equal{#1}{}}
       {\gamma_{#2}}
       {\gamma^\ast \langle #2 \rangle}
   }
   {
       \ifthenelse{\equal{#1}{}}
       {\gamma_{#2} \langle #3 \rangle}
       {\gamma_{#2}^\ast \langle #3 \rangle}
   }
}
\newcommandx{\untrtarg}[1][1=]{\ifthenelse{\equal{#1}{}}{\gamma^\ast}{\gamma^\ast \langle #1 \rangle}}
\newcommand{\untargpath}[1]{\bar{\gamma}_{#1}}
\newcommandx{\unpartarg}[2][1=]{
    \ifthenelse{\equal{#1}{}}
    {\gamma_{#2}^\N}
    {\gamma_{#2}^\N  \langle #1 \rangle}
}
\newcommand{\wgt}[2]{\omega_{#1}^{#2}}
\newcommandx{\wgtfunc}[2][1=]{
\ifthenelse{\equal{#1}{}}
{w_{#2}}
{w_{#2} \langle #1 \rangle}
}

\newcommand{\iwishsymb}{\mathcal IW}

\newcommandx{\x}[2][2=]{
\ifthenelse{\equal{#2}{}}
{w_{#1}}
{w_{#1}^{#2}}
}
\newcommandx{\X}[2][2=]{
\ifthenelse{\equal{#2}{}}
{W_{#1}}
{W_{#1}^{#2}}
}
\newcommandx{\xt}[2][2=]{
\ifthenelse{\equal{#2}{}}
{\tilde{w}_{#1}}
{\tilde{w}_{#1}^{#2}}
}
\newcommandx{\Xt}[2][2=]{
\ifthenelse{\equal{#2}{}}
{\tilde{W}_{#1}}
{\tilde{W}_{#1}^{#2}}
}

\newcommand{\xsp}[1]{\mathcal{W}_{#1}}

\newcommandx{\y}[1][1=]{\ifthenelse{\equal{#1}{}}{y}{y_{#1}}}
\newcommandx{\Y}[1][1=]{\ifthenelse{\equal{#1}{}}{Y}{Y_{#1}}}
\newcommandx{\ysp}[1][1=]{\ifthenelse{\equal{#1}{}}{\mathsf{Y}}{\mathsf{Y}_{#1}}}
\newcommandx{\yfd}[1][1=]{\ifthenelse{\equal{#1}{}}{\mathcal{Y}}{\mathcal{Y}_{#1}}}
\newcommandx{\z}[2][2=]{
\ifthenelse{\equal{#2}{}}
{z_{#1}}
{z_{#1}^{#2}}
}
\newcommandx{\Z}[2][2=]{
\ifthenelse{\equal{#2}{}}
{\tree_{#1}}
{\tree_{#1}^{#2}}
}

\newcommand{\zpart}[2]{T_{#1}^{#2}}

\newcommand{\indstate}[4]{#1\ci #2 \,|\, #3}

\theoremstyle{definition}

\begin{document}
\maketitle
\begin{abstract}
    % -*- root: main.tex -*-
In this study, we present a multi-class graphical Bayesian predictive classifier that incorporates the uncertainty in the model selection into the standard Bayesian formalism.
For each class,  the dependence structure underlying  the observed features is represented by a set of decomposable Gaussian graphical models.
Emphasis is then placed on the \emph{Bayesian model averaging}
%on the general  Bayesian predictive inference
which takes full account of the class-specific model uncertainty by averaging over the posterior graph model probabilities.
An explicit evaluation of the model probabilities is well known to be infeasible.
%Even though the decomposability assumption severely reduces the model space, the size of the class of decomposable  models is still immense, rendering  the explicit  Bayesian averaging over \emph{all} the models infeasible.
To address this issue, we consider the particle Gibbs strategy of \cite{nontempspmcmc} for posterior sampling from decomposable graphical models which utilizes the Christmas tree algorithm of \cite{cta} as proposal kernel.
%Our posterior predictive inference exploits the junction tree and clique-separator factorization properties of decomposable graphs.
We also derive a strong hyper Markov law which we call the \emph{hyper normal Wishart law} that allow to perform the resultant Bayesian calculations locally.
The proposed predictive graphical classifier reveals superior performance compared to the ordinary Bayesian predictive rule that does not account for the model uncertainty, as well as to a number of out-of-the-box classifiers.

\end{abstract}

\section{Introduction} % (fold)
\label{sec:introduction}
% -*- root: main.tex -*-
% (1) något är viktigt, förklara varför;
% (2) någon har gjort något för att hantera/lösa detta, förklara vad;
% (3) det man gjort är inte optimalt, förklara varför;
% (4) du har gjort något bättre, förklara vad och varför detta är bättre;
% (5) beskriv stukturen hos artikeln. Kan man inte skriva en introduktion till sin artikel enligt denna princip har man ju inte lyckats att göra något bidrag! Jag bifogar en länk till den senaste introduktionen jag skrev, om du vill ha något att titta på:

% https://arxiv.org/abs/1608.06851

A Bayesian supervised predictive classifier is presented for a multi-class classification problem where class distributions are represented  by graphical models.
%the conditional independence graph models.
The goal is to decide a class-membership of a new observation, and assess the uncertainty related to the decision rule conditional on all relevant information available.

Suppose that a set $\boldsymbol{n} = \{ 1,\dots,\testdatasize \}$ of data items is given where each item belongs
to one of $\nclasses$ source classes denoted by $\Pi_{1},\dots,\Pi_{\nclasses}$.  Working in a supervised setting, we further assume that all the eligible classes are a priori specified.   We introduce a discrete random variable,  $\mathcal{C}$,  with $\mathcal{C}=\class$ denoting membership of $\Pi_{\class}$, $\class=1,\dots, \nclasses$ and assume the class labels to be known for all the samples in $\boldsymbol{n}$.

Our Bayesian approach calls for a prior on $\mathcal{C}$, and we will assume that this prior, represented by the probability mass function $p_{\mathcal{C}}(\class)$,
$\sum_{\class \in C} p_{\mathcal{C}}(\class) =1$, is also known.
With $p_{\mathcal{C}}(\class)$ unknown, by treating $\mathcal C$ as a multinomial random variable, we can use the family of Dirichlet priors to obtain the posterior probabilities over an ensemble of classes.
%$C=\{1,\dots, \nclasses\}$ of $\nclasses$.

Associated with each item is a vector $\testdatapnt=(x_1,\dots,x_p)'$, a collection of $p$ continuous feature variables.
Each such vector assigned to a particular class $\Pi_{\class}$ is assumed to be generated from class-specific distributions with the density
$f(\testdataset \mid \param_\class)$, $\param_\class \in  \Theta_\class$, where $\Theta_\class$ is the associated parameter space.
Suppose further that we have observed the values of $\mathcal{C}$ and all the $\testdatapnt$'s for a random sample of $n$ items, yielding the \emph{training data}
$$\{(\mathcal{C}^{i},\testdatapnt^{i})\}_{i=1}^{\testdatasize}, \quad n=n_{1}+\dots +n_{\nclasses}. $$
Let $\boldsymbol{\mathcal{C}}^{n}$ denote the vector $(\mathcal{C}^{1},\dots,\mathcal{C}^{\testdatasize})$, and $\testdataset^{(\testdatasize)}$ be the sample vectors for items in the training data.  For any subset $\boldsymbol a\subset \boldsymbol n$, we use
$\testdatapnt^{(a)} \subset \testdataset^{(\testdatasize)}$ to denote the subsets of data for the corresponding samples.  In what follows, we generally use the upper indices to denote sample observations and save lower indices for structural properties of the observed  $\testdatapnt$).

%(do we need this? WE need, and we need a subset
%$\testdatapnt^{(a)} \subset \testdataset^{(\testdatasize)})$

Suppose further that we are presented with a  "future" [or  "fresh" or new "test"] observation, an item given by a vector
$\testdatapnt^{\testdatasize+1}$ of observed features. Given the above set up, we are interested in making  inference about the value of $\mathcal{C}^{\testdatasize+1}$ which is the main target of predictive inference, (i.e. decide class-membership) of the new observation on the basis of all the available data $(\boldsymbol{\mathcal{C}}^{n}, \testdataset^{(n)}, \testdatapnt^{\testdatasize+1})$.
This problem of classification can be formulated as a process of determining posterior probabilities for $\Pi_{1},\dots,\Pi_{\nclasses}$,
i.e. posterior distribution of $\mathcal{C}^{\testdatasize+1}$ for all the classes.
%For this purpose, we consider the \emph{Bayesian predictive classifier} which provides the following posterior distribution for the class memberships $\mathcal{C}_{\testdatasize+1}$ of
%$\testdatapnt_{\testdatasize+1}$
\begin{align}
    p_{\mathcal{C}^{\testdatasize+1}}  (\class \mid \testdatapnt^{\testdatasize+1}, \testdataset[\testdatasize_{c}], \boldsymbol{\mathcal{C}}^{\testdatasize}) =
\frac{f(\testdatapnt^{\testdatasize+1} \mid \class, \testdataset[\testdatasize_{\class}],  \boldsymbol{\mathcal{C}}^{\testdatasize}) \,
p_{\mathcal{C}^{\testdatasize+1}}(\class \mid \boldsymbol{\mathcal{C}}^{\testdatasize})}
{\sum \limits_{\class' \in  \{1,\dots,k\}}  f(\testdatapnt^{\testdatasize+1} \mid \class',\testdataset[\testdatasize_{\class'}],  \boldsymbol{\mathcal{C}}^{\testdatasize}) \,
p_{\mathcal{C}^{\testdatasize+1}}(\class' \mid \boldsymbol{\mathcal{C}}^{\testdatasize})  },
\label{eq:postdist}
\end{align}
where
\begin{align}
    f(\testdatapnt^{\testdatasize+1} \mid \testdataset[\testdatasize_{\class}], \class, \boldsymbol{\mathcal{C}}^{\testdatasize}) =  \int_{\paramsp_{\class}}
f(\testdatapnt^{\testdatasize+1} \mid \param_\class)f(\param_\class \mid \testdataset[\testdatasize_{\class}],
\boldsymbol{\mathcal{C}}^{\testdatasize})d\param_\class
\end{align}
is the predictive probability distribution of $\testdatapnt^{{\testdatasize+1}}$ with the class $\Pi_{\class}$, and
\begin{align}
p_{\mathcal{C}^{\testdatasize+1}}(\class\mid\boldsymbol{\mathcal{C}}^{\testdatasize}) = \int p(\class \mid \phi_\class)f(\phi_{\class} \mid \boldsymbol{\mathcal{C}}^{\testdatasize}) d \phi_{\class},
\end{align}
 is the prior probability of the class $\class$ parameterized by $\phi_\class$.

Further, the optimal classification rule corresponding to zero-one loss function (assigning zero cost to any
correct decision, and unit cost to any wrong decision) can be obtained by specifying the mode of the posterior distribution in
\eqref{eq:postdist}. This is referred to as the \emph{maximum a posteriori (MAP) criterion} that can be further simplified to
\begin{align}
  \class^\ast =  \arg \max_{\class \in \{1,\dots,k\}} p(\class \mid \testdatapnt^{\testdatasize+1}, \testdataset[\testdatasize], \boldsymbol{\mathcal{C}}^{\testdatasize})
    =
    \arg \max_{\class \in  \{1,\dots,k\}}f(\testdatapnt^{\testdatasize+1}  \mid \class, \testdataset[\testdatasize], \boldsymbol{\mathcal{C}}^{\testdatasize})p(\class \mid \boldsymbol{\mathcal{C}}^{\testdatasize}),
\end{align}
yielding the Bayesian predictive classifier which assigns $\testdatapnt^{\testdatasize+1}$ to $\Pi_{\class^\ast}$. The classification rule that assigns a class membership of a new observation $\testdatapnt^{\testdatasize+1}$ by using the MAP estimator of $p_{\mathcal{C}^{\testdatasize+1}} (\class \mid \testdatapnt^{\testdatasize+1}, \testdataset[\testdatasize], \boldsymbol{\mathcal{C}}^{\testdatasize})$,  is known to be optimal in a sense that  it minimizes the averaged risk of misclassification, see e.g. \cite{ripley2007pattern}.

Foundations of the general predictive Bayesian inference, with the focus on the predictive classification are considered in pioneering works by  Geisser, see
\cite{geisser1964posterior,geisser1966predictive,geisserpredictive}.
 Several such predictive classifiers have been later emerged in the literature and their performance properties have been studied.  
 Examples include e.g.  \cite{dawid1992conjugate} where the supervised Bayesian classification has been studied using natural conjugate priors for the  Gaussian distribution parameters and with infinitely many feature variables. Class specific predictive distributions are derived in such growing dimensions asymptotic framework along with the conditions that allows almost sure identification of the class
membership of a sample from an unknown origin.
%(under which asymptotically perfect classification is possible).

Further extensions of the posterior predictive strategy within classification framework are considered in the recent studies by \cite{Corander2013}, and \cite{nyman2016marginal} where the key component is the incorporation of the structural properties of the class-conditional  distributions into the Bayesian predictive inference.
These structural properties describe the qualitative manner in which information flows among the feature variables and are well-represented by a \emph{graphical model}, $G=(V,E)$ where the nodes $V$, representing random variables
%$x_1,\dots,x_p$
in the model, are connected by undirected edges $E \subseteq V\times V$, encoding the conditional independence properties of the multivariate distribution \cite{lauritzen1996}.
%where nodes, representing variables in the model, are connected by undirected edges, encoding the conditional independence properties of the multivariate distribution (Lauritzen, 1996).

In \cite{Corander2013}, the class-specific distributions are represented by a family of Gaussian graphical models (GGMs) with the block-diagonal structure, which is then conveniently merged with the conjugate Bayesian analysis. Due to the factorization of the posterior predictive densities over the graph structure, such type of assumption allows for local, within-blocks Bayesian updates, which in turn delivers an efficient solution to the high-dimensional classification problems, see details in \cite{Corander2013}.
\cite{nyman2016marginal}, consider classification of categorical data where the class of stratified decomposable graphical models is used for encoding feature dependencies.
This approach is shown to allow for a more detailed representation of the dependence structure, thus enhancing the classification process.

%[Add more examples, may be  (\cite{bellon2013}, unpublished  paper?]

The above mentioned classifies, while usually called predictive, are in fact derived under a single {\it known} graph $\graph$ defining the class-conditional distribution and hence do not generally obey the principles of the predictive Bayesian inference in the sense of \cite{geisserpredictive}.
%model class-conditional graphical model $G$ and hence do not generally obey the principles of the predictive Bayesian inference in the sense of Geisser (1993). graph model
 Conditioning on $\graph$, while essentially simplifying the posterior analysis, ignores the uncertainty the model themselves possess on the probabilistic classification.
 We argue therefore that the standard predictive classification formalism stated in equations (1)-(3) is flawed and suggest a generic framework for the Bayesian treatment of model uncertainty.
 Specifically, for each $\Pi_\class$, we consider the underlying class-specific graph, $\graph[\class]$ as \emph{unknown}, and characterize it through its posterior distribution which is in turn incorporated into the building of the predictive classifier.

 %strategy where the uncertainty about the class-specific graphical models is characterized through their posterior distributions, which are then incorporated into building of the predictive classifiers.

 %the posterior inference on the class-specific graphical models is incorporated into building of the predictive classifiers. (using the observed data)
 Bayesian inference concerning the graph $\graph$ is usually referred to a \emph{structure learning} and requires specification of a flexible but tractable family $\graphsp$ of possible graphs, capable of representing a variety of the conditional independence structure.
 %Inference about the underlying graph from observed data is usually referred to a \emph{Bayesian structure learning} and requires specification of a flexible but tractable family of graphical models, capable of representing a variety of the conditional independence structure.
 Although other types of graphical models exist, in the interests of tractability and scalability, we for each $\Pi_{\class}$, restrict the family $\graphsp$ to the set of \emph{undirected decomposable Gaussian graphical models}, and allow the graph structure freely vary across different classes within $\graphsp$.
 The family $\graphsp$ exhibits the special property that for each $\graph \in\graphsp$, a Markov distribution over $\graph$  can be derived recursively, by using Markov combinations of smaller components,  see details in Section 2.
Prior distributions over the graphical structure itself, which are termed for \emph{graph laws} are discussed in detail in \cite{StructuralMarkov}.
%Inference concerning the graph $G$, specifying only a family $\graphsp$ of possible graphs.

 The first key component of our suggested approach is the \emph{Bayesian model averaging}, where the posterior predictive distribution of $\testdatapnt^{\testdatasize+1}$ in (2) under each of the candidate models in $\graphsp$ is weighted by the posterior model probabilities.
 %thereby taking full account of the class-specific model uncertainty.
Theoretically, \cite{10.2307/2291017} show that Bayesian averaging over \emph{all} the models in this fashion provides better predictive accuracy  than using any single model; see Dawid (1986) and \cite{kass1995bayes} for a review of general Bayesian model averaging approach and a more recent review on the evaluation of Bayesian approaches for model uncertainty by \cite{clyde2004model}.
But even though the graph decomposability assumption severely reduces the model space, the size of the family of decomposable is still immense, rendering the explicit Bayesian averaging over all potential structures in $\graphsp$ infeasible.
To tackle this issue, we exploit the \emph{particle Markov chain Monte Carlo} (PMCMC) sampling strategy to obtain an approximation of the graph posterior distribution.
This brings us to the second key component of our approach, namely \emph{particle Gibbs with systematic refreshments} (PG) sampling  scheme which is designed for efficient posterior sampling from decomposable graphical models; see details in \cite{nontempspmcmc} and a companion paper \cite{cta}.
Combined, the two components constitute our new, inherently predictive classification procedure, called for the \emph{graphical posterior predictive} classifier.
For a general review of particle MCMC methods, see e.g. \cite{andrieu2010particle} and \cite{chopin:singh:2015}.
 %Putting together the above two components, we arrive to our novel classification procedure called for the \emph{graphical posterior predictive} classifier

In a series of papers,  \cite{corander2013inductive,corander2013have} and \cite{Cui2016} have addressed another important issue of the predictive inference in both supervised and semi-supervised settings, namely using \emph{marginal} and \emph{simultaneous} predictive classifiers.
Simultaneous classifiers require modeling of the \emph{joint} posterior predictive distribution of the unknown class variables for the test sample and are therefore computationally much more demanding than the standard, marginal ones which treat each fresh observation separately and independently on the other observations in the test data.

Performance properties of simultaneous and marginal supervised classifiers are extensively studied in \cite{Cui2016},    indicating that both approaches demonstrate asymptotically equal performance accuracy when the amount of training data goes to infinity \cite[Theorem 1, Theorem 2]{Cui2016}.
In the light of results, we in the current study focus on the marginal predictive strategy, noting however that our proposed graphical predictive classifier is fully suitable for the simultaneous Bayesian inference.

The structure  of the remaining part of the paper is as follows. 
In Section \ref{sub:graphical_models},  we review the properties of the hyper Markov laws over decomposable GGMs (\cite{dawid1993}), introduce the hyper normal inverse Wishart conjugate family of prior distributions and show that it is strong hyper Markov.
Using these results, we then in Section \ref{sec:graphical_bayesian_predictive_classifier} derive our graphical predictive classifier that is based on the predictive distribution (predictive score function) for $\testdatapnt^{\testdatasize+1}$ and incorporates Bayesian model averaging over $\graphsp$.
%derive the predictive  distribution for $\testdatapnt_{\testdatasize+1}$ ...  (incorporating the graph ... ) which is ... graphical predictive classifier.
In  Section \ref{sec:bayesian_graph_structure_learning_with_the_cta_particle_gibbs_sampler},  we derive the MCMC graph structure learning algorithm  which is needed for approximating the graph posterior at the model averaging step.
This algorithm,  recently suggested in  \cite{nontempspmcmc}  is based on the particle Gibbs graph structure learning that exploits the Christmas tree kernel presented in the companion paper \cite{cta}, and allows for an efficient posterior sampling form the decomposable GGMs.
Next in Section \ref{sec:simulation_study}, we provide the description and results of our numerical study.
We conclude in Section \ref{sec:summary} with some computational considerations and future research directions.
Technical derivations are found in Appendix \ref{appendix:a}.

% \Psi, \Xi, \Upsilon, \boldsymbol{\Upsilon}_{Q}, \psi  $$

\section{Review of graph theory and Markov properties} % (fold)
\label{sub:graphical_models}
% -*- root: main.tex -*-

For consistency and clarity of the presentation, %the following summarizes notations, terms and results of the graphical modeling which will be  used throughout the paper.]
we restate some definitions and results regarding graph theory and Markov properties which will be used later in this text.
For further details the reader is referred to \cite{lauritzen1996} and \cite{dawid1993}.

Let $\graph=(\graphnodeset,\graphedgeset)$ be an undirected graph with node set $\graphnodeset$ and edge set $\graphedgeset\subseteq\graphnodeset\times \graphnodeset$.
Two nodes $a,b\in \graphnodeset$ are said to be neighbors in $\graph$ if $(a,b)\in \graphedgeset$.
A subset $\jtnode \subseteq \graphnodeset$ is complete if every pair of nodes $(a,b)\in \graphedgeset$ are neighbors in the subgraph $\graph[\jtnode]$, induced by the nodes in $\jtnode$.
If $\jtnode$ is maximal, in the sense that it is not contained in any other complete set of nodes it is called a clique.
\begin{definition}
A \emph{decomposition} of $\graph$ is a pair $(A, B)$ such that $A\cup B = \graphnodeset$, $A \cap B$ is complete and separates $A$ from $B$, that is, every path between nodes in $A$ and $B$ must intersect $A\cap B$.
\end{definition}
Henceforth in this paper  we will assume that $\graph$ is \emph{decomposable} defined next.
% meaning that the cliques in $\graph$ can be order in a \emph{perfect order} $\jtnode[1],\dots,\jtnode[\jtorder]$ so that
% \begin{align}
%     \sep[i] =\jtnode[i] \cap (\cup_{j=1}^{i-1} \jtnode[j]) \subseteq \jtnode[\jtorder],
%     \label{eq:rip}
% \end{align}
% holds for every $k=1,\dots,i-1$.
% %The property \eqref{eq:rip} is called the running intersection property and is called perfect.
% The separators $\sep[i]$, are not necessarily distinct but for any perfect order, the set of separators is invariant.
% We will by $\jtnodeset$ denote the set of cliques in a graph $\graph$ and $\sepset$ the set of separators.
% Another characterization of decomposable graphs that will play a crucial role in the suggested classifier is the junction tree representation.

% \begin{definition}
% A graph is called decomposable if the cliques can be arranged in a so called \emph{junction tree}, meaning that for any pair of cliques $\jtnode[i],\jtnode[j]$ it holds that
% $\jtnode[i]\cap \jtnode[j] \subseteq \jtnode,$
% for every clique $\jtnode$ on the unique path between $\jtnode[i]$ and $\jtnode[j]$.
% \end{definition}

\begin{definition} \label{def:decomposable:graph}
    A graph $\graph$ is \emph{decomposable} if it is complete or if there
    exists a decomposition $(A, B)$ of $\graph$ such that
    $\graph[A]$ and $\graph[B]$ are decomposable.
\end{definition}

Decomposable graphs are sometimes alternatively termed \emph{chordal}
or \emph{triangulated}, as Definition~\ref{def:decomposable:graph} is
equivalent to the requirement that every cycle of length $4$ or more is
chorded. %\jimmy{Felix, add a reference.}
Decomposable graphs are also characterized by the fact that their cliques can be arranged in a so called \emph{junction tree}.

\begin{theorem}
A graph $\graph$ is decomposable if and only if its cliques can be arranged in a \emph{junction tree}, meaning that for any pair of cliques $\jtnode[i],\jtnode[j]$ it holds that
$\jtnode[i]\cap \jtnode[j] \subseteq \jtnode,$
for every clique $\jtnode$ on the unique path between $\jtnode[i]$ and $\jtnode[j]$.
\end{theorem}

% \begin{definition}
%     \label{def:junction:tree:property}
%     A tree $\tree = (\graphnodeset, \graphedgeset)$, where
%     $\graphnodeset = \{ \jtnode[1], \dots, \jtnode[k] \}$ with each
%     $\jtnode[i]$ being a subset of some finite set $W$, satisfies
%     the \emph{junction tree property} if for all $(\cl, \cl') \in \graphnodeset^2$,
%     the path $\cl \sim \cl' = \{ \jtnode[i_j] \}_{j = 1}^{\ell + 1}$ satisfies
%     $$
%         \cl \cap \cl' \subseteq \bigcap_{j = 1}^{\ell + 1} \jtnode[i_j].
%     $$
% \end{definition}

% \begin{theorem}
%     \label{eq:thm:junction:tree:representation}
%     A graph $\graph$ is decomposable if and only if there exists a tree
%     $\tree$ of cliques that satisfies the junction tree property.
% \end{theorem}

The  intersections of two neighboring cliques in a junction tree is called a \emph{separator}.
We denote the set of separators by $\sepset$ and the set of cliques by $\jtnodeset$.
Since the graph, $\graph$ underlying a junction tree is unique, $\jtnodeset$ and $\sepset$ can also be regarded as components of $\graph$ itself.
The space of all decomposable graphs with a given number of nodes, which should be understood from the context, is denoted by $\graphsp$. 
A junction tree for a specific graph $\graph$ is not unique in general, following the notation in \cite{doi:10.1198/jcgs.2009.07129} we denote the number of equivalent junction trees by $\ntrees{\graph}$.
On the other hand, each junction tree $\tree$ has exactly one underlying decomposable graph which we denote by $\trgr(\tree)$.
Let $\testdatapntrv = (\testdatapntrv[i])_{i\in \graphnodeset}$ be a random variable and let $f(\testdatapnt[A])$ denote a generic marginal density function of $\testdatapntrv[A]=(\testdatapntrv[i])_{i\in A}$, where $A \subseteq \graphnodeset$.
For $A,B,C\subseteq \graphnodeset$, we say that $\testdatapntrv[A]$ and $\testdatapntrv[B]$ are conditional independent given $\testdatapntrv[C]$ if the conditional 
density $f(\testdatapnt[A]|\testdatapnt[B], \testdatapnt[C])$ can be written as a function of $\testdatapnt[A]$  and $\testdatapnt[C]$ alone.
This relation is usually denoted by $\indstate{\testdatapntrv[A]}{\testdatapntrv[B]}{\testdatapntrv[C]}{}$.
Important properties of the conditional independence studied in \cite{10.2307/2984718} are usually described as the \emph{semi-graphoid axioms} stated next.
Let $X,Y$ and $Z$ be disjoint random variables, then
\begin{itemize}
    \item[(S1)] If $\indstate{X}{Y}{Z}{}$ then $\indstate{Y}{X}{Z}{}$ (symmetry)
    \item[(S2)] If $\indstate{X}{Y}{Z}{}$ and $U$ is a function of $X$, then $\indstate{Y}{U}{Z}{}$ (decomposition)
    \item[(S3)] If $\indstate{X}{Y}{Z}{}$ and $U$ is a function of $X$, then $\indstate{X}{Y}{(Z, U)}{}$ (weak union)
    \item[(S4)] If $\indstate{X}{Y}{Z}{}$ and $\indstate{X}{W}{(Y,Z)}{}$ then $\indstate{X}{(W,Y)}{Z}{}$ (contraction).
\end{itemize}
It is standard to relate one semi-graphoid to another in terms of the Markov properties, in this case the former would be induced by the separation statements in an undirected graph and the latter induced by the conditional independence statements in a probability distribution, see e.g. \cite{graphoids}.
\begin{definition}
A distribution for $\testdatapntrv$ is said to be \emph{Markov with respect to $\graph$} if for any decomposition $(A,B)$ it holds that
\begin{align}
    \indstate{\testdatapntrv[A]}{\testdatapntrv[B]}{\testdatapntrv[A\cap B]}{}.
    \label{eq:markprop}
\end{align}
\end{definition}
%The family of distributions satisfying \eqref{eq:markprop} will be denoted by $\markmod (\graph)$.
For distributions which are Markov with respect to a graph $\graph$, the density follows the so-called \emph{clique-separator-factorization} (CSF) property 
\begin{align}
    f(\testdatapnt) = \frac{\prod_{\jtnode \in \jtnodeset}f(\testdatapnt[\jtnode])}{\prod_{\sep \in \sepset}f(\testdatapnt[\sep])}.
    \label{eq:markovfact}
\end{align}
%where $f(\testdatapnt[\jtnode])$ and $f(\testdatapnt[\sep])$ are marginal densities over the cliques and the separators respectively.

In the multivariate Gaussian distribution, the underlying graph is determined by the precision matrix.
We let the normal distributions with mean $\meanvec\in \mathbb{R}^{|\graphnodeset|}$ and positive definite covariance $\covmat \in \mathbb{R}^{|\graphnodeset| \times |\graphnodeset|}$ be denoted by $\normsymb(\meanvec, \covmat)$.
For any graph $\graph$, we let $\precmatrestr[\graph]$ denote the space of all symmetric positive definite matrices $\covmat \in \mathbb{R}^{|\graphnodeset| \times |\graphnodeset|}$ such that $(\covmat^{-1})_{ij}=0$ whenever $(i,j)$ is not en edge in $\graph$.
Further, for any fixed $\meanvec$, the set of Gaussian distributions with $\covmat^{-1} \in \precmatrestr[\graph]$  corresponds those which are Markov with respect to $\graph$.
These models are typically called Gaussian graphical models (GGMs) or covariance selection models, see \cite{1972,10.2307/2241271}.

\cite{dawid1993} introduced an extension of the concept of Markov properties to distributions (referred to as \emph{laws}) defined over distributions, in parametric models determined by some random parameter $\param$.
For $A,B\subseteq \graphnodeset$, the notation $\param_A$ refers to the marginal distribution of $\testdatapntrv[A]$ and $\param_{A|B}$ determines the distribution of $(\testdatapntrv[A]|\testdatapntrv[B]=\testdatapnt[B])$. 
\begin{definition}
A law $\law(\param)$ is said to be \emph{(weak) hyper Markov} with respect to $\graph$ if for $A,B\subseteq \graphnodeset$
\begin{align*}
\indstate{\param_A}{\param_B}{\param_{A\cap B}}{}.
\end{align*}
It is said to be \emph{strong hyper Markov}  if
\begin{align*}
    \param_A \ci \param_{B|A}.
\end{align*}

\end{definition}
The strong hyper Markov property is of particular interest in Bayesian inference since it enables for posterior calculations to be performed locally using data collected separately for each clique, see \cite[Corollary 5.5]{dawid1993}.

The following property is crucial in the construction of hyper Markov laws.
For given two laws $\mathcal L_A$ and $\mathcal L_B$, defined for $\param_A$ and $\param_B$ respectively  where $A,B\subseteq \graphnodeset$, it is said that  $\mathcal L_A$ and $\mathcal L_B$ are \emph{hyper consistent} if they both have the same marginal law on $\param_{A\cap B}$.
Given a set of hyper consistent laws $\{\law_{\jtnode}:\jtnode \in \jtnodeset\}$, the unique hyper Markov law is constructed by hyper combinations of the clique-specific laws.

% subsection graphical_models (end)

% section graphical_predictive_classification (end)
\section{Graphical predictive classification} % (fold)
\label{sec:graphical_bayesian_predictive_classifier}
% -*- root: main.tex -*-
% subsection predictive_distributions_for_ggm (end)
In order to derive a predictive distribution for $\testdatapnt[\testdatasize+1]$, we first restrict the attention to the single class case, implying that $\testdatasize_\class=\testdatasize$ and focus attention on one clique $\jtnode \in \jtnodeset$.
In such setting, the data are assumed to be sampled from $\normsymb(\meanvec_{\jtnode}, \covmat_{\jtnode})$ and can be summarized by the joint density
\begin{align}
    f(\testdataset[{\testdatasize}]_{\jtnode} | \meanvec[\jtnode], \covmat_{\jtnode}) =  (
    2\pi)^{-\testdatasize q /2}|\covmat_{\jtnode}|^{-\testdatasize/2} \exp \bigg \{-\frac{1}{2} \sum_{i=1}^\testdatasize(\testdatapnt_{\jtnode}^i -\testdatameanest_{\jtnode})'\covmat_{\jtnode}^{-1}(\testdatapnt_{\jtnode}^i -\testdatameanest_{\jtnode}) \bigg \},
\end{align}
where $\testdatameanest_{\jtnode} = \frac{1}{\testdatasize}\sum_{i=1}^\testdatasize \testdatapnt_{\jtnode}^{i}$ and $|\jtnode|=q$.
To simplify the notation, we temporarily drop the subscript $\jtnode$.
It is well known from Bayesian theory that the conjugate prior for $(\meanvec,\covmat)$ is \emph{normal inverse Wishart}
which we denote by $\giwsymb(\niwnmu, \niwnnu ,\niwiwtau, \niwiwalpha)$.
In this distribution, the marginal of $\covmat$ is \emph{inverse Wishart}, $\iwishsymb(\niwiwtau, \niwiwalpha)$ with $\niwiwalpha > q-1$ degrees of freedom and positive definite location matrix $\niwiwtau$. 
The conditional distribution of $\meanvec$ given  $\covmat$ is  $\normsymb(\niwnmu, \frac{1}{\niwnnu}\covmat)$.
The joint density is given by
\begin{align}
    \giwsymb(\meanvec, \covmat | \niwnmu, \niwnnu ,\niwiwtau, \niwiwalpha)=
    \frac{1}{\niwconstsymb (\niwiwalpha, \niwiwtau, \niwnnu)}
    | \covmat |^{-1/2}
    \exp\{-\frac{\niwnnu}{2} (\meanvec -\niwnmu)'\covmat^{-1}(\meanvec -\niwnmu)\big\} \\
    \times
    | \covmat |^{-(\niwiwalpha +q +1)/2}
    \exp\{-\frac{1}{2} tr(\niwiwtau\covmat^{-1})\} \nonumber,
    \label{eq:niw}
\end{align}
where
\begin{align*}
\niwconstsymb(\niwiwalpha, \niwiwtau, \niwnnu) = \frac{2^{\niwiwalpha q/2}\Gamma_{q}(\niwiwalpha/2)}{|\niwiwtau|^{\niwiwalpha/2}}  \frac{(2\pi)^{q/2}}{\niwnnu^{\frac{1}{2}}}.
\end{align*}
$\Gamma_q$ is the multivariate gamma function defined $\text{ for } a>(q-1)/2$ as
\begin{align*}
    \Gamma_q(a) = \pi^{q(q-1)/4}\prod_{i=0}^{q-1}\Gamma(a-\frac{i}{2}).
\end{align*}

The posterior of $(\meanvec,\covmat)$ obtained after updating this prior with $\testdataset[\testdatasize]$
has the parameters
\begin{align*}
\niwnnustsymb=\niwnnust,\, \niwnmustsymb=\niwnmust,\, \niwiwalphastsymb=\niwiwalphast \text{ and } \niwiwtaustsymb=\niwiwtaust,
\end{align*}
where $\sumsq=\sum_{i=1}^\testdatasize(\testdatapnt^i - \testdatameanest)(\testdatapnt^i - \testdatameanest)'$.
The posterior predictive distribution of $\testdatapnt^{\testdatasize+1}$, where the parameters are integrated out according to (2) is now written as
\begin{align}
    \int_{\precmatrestr[{\graph[\jtnode]}]} \int_{\mathbb R^{q}}\normsymb(\testdatapnt^{\testdatasize + 1}|\meanvec, \covmat)\giwsymb(\meanvec, \covmat | \niwnmustsymb, \niwnnustsymb ,\niwiwtaustsymb, \niwiwalphastsymb)d\meanvec d\covmat.
\end{align}
This is the non-central multivariate \emph{t}-distribution  with density
\begin{align*}
    \tdenssymb(\testdatapnt^{\testdatasize+1}|\tpreddfsymb, \niwnmustsymb, \tpredprecsymb)
    & = \tdens{\tpredprecsymb}{q}{\tpreddfsymb}{\niwnmustsymb}{\testdatapnt^{\testdatasize+1}},
\end{align*}
where $\tpreddfsymb=\tpreddf{q}$ and $\tpredprecsymb=\tpredprec{q}$, see e.g. \cite[p.~441]{BayesianTheory}.

\subsection{Hyper Markov laws over decomposable graphs} % (fold)
\label{sub:hyper_markov_laws_over_decomposable_graphs}
% subsection hyper_markov_laws_over_decomposable_graphs (end)
In this we discuss how to obtain a predictive distribution of full dimension, $\p$, that is not restricted to one specific clique. 
For this purpose, we first specify a joint prior law for $(\meanvec, \covmat)$, and show that it is strong hyper Markov with respect to $\graph$.
Analogously to the above construction of the normal inverse Wishart law, we let the marginal law of $\covmat$ be the so called \emph{hyper inverse Wishart law}, with $\niwiwalpha>\p-1$  degrees of freedom and precision matrix $\niwiwtau$ satisfying
\begin{align}
     \sum_{\jtnode \in \jtnodeset} [\niwiwtau_{\jtnode}]^0 - \sum_{\sep \in \sepset} [\niwiwtau_{\sep}]^0 =\niwiwtau,
     \label{eq:matdec}
\end{align}
where the notation $[A]^0$ is obtained from the matrix $A$ by padding it with zeros to get the correct dimensions.
This law was derived in \cite{dawid1993} as the law obtained from hyper combinations of $\iwishsymb(\niwiwtau_{\jtnode}, \niwiwalpha)$ laws defined individually for each clique. 
Further, we specify the distribution of $\meanvec$ conditional on $\covmat$ as $\normsymb(\niwnmu, \frac{1}{\niwnnu}\covmat)$, where $\niwnmu\in \mathbb R^\p,\niwnnu>0$.
We call the resulting law for the \emph{hyper normal inverse Wishart law} and denote it by $\giwsymb_{\graph}(\niwnmu,\niwnnu, \niwiwtau, \niwiwalpha)$.
The density can be verified (see \ref{sec:gnw}) to follow the CSF property as
\begin{align}
    \giwsymb_{\graph}(\meanvec,\covmat \mid \niwnmu,\niwnnu, \niwiwtau, \niwiwalpha) =
    \frac{\prod_{\jtnode \in \jtnodeset} \giwsymb(\meanvec_{\jtnode},\covmat_{\jtnode} \mid \niwnmu_{\jtnode},\niwnnu, \niwiwtau_{\jtnode}, \niwiwalpha) }
          {\prod_{\sep \in \sepset} \giwsymb(\meanvec_{\sep},\covmat_{\sep} \mid \niwnmu_{\sep},\niwnnu, \niwiwtau_{\sep}, \niwiwalpha)}.
          \label{eq:gnw}
\end{align}
In order to show that the strong hyper Markov we restate Proposition 5.1 from \cite{dawid1993}.
\begin{proposition}\cite[Proposition 5.1]{dawid1993}
If the prior law $\mathcal L(\param)$ is hyper Markov over $\graph$ then the joint distribution of $(\testdatapntrv, \param)$ satisfies, for any decomposition $(A,B)$ of $\graph$,
\begin{align}
\indstate{(\testdatapntrv[A], \param[A])}{(\testdatapntrv[B], \param[B])}{(\testdatapntrv[A\cap B], \param[A\cap B])}{}.
    \label{eq:weakjoint}
\end{align}
If $\mathcal L(\param)$ is strong hyper Markov, it also satisfies
\begin{align}
    \indstate{(\testdatapntrv[A], \param[A])}{(\testdatapntrv[B], \param[B|A])}{\testdatapntrv[A\cap B]}{}.
\end{align}
\label{prop:5.1}
\end{proposition}
\begin{proposition}
The hyper normal inverse Wishart law is strong hyper Markov.
\end{proposition}
\begin{proof}
In light of \cite[Proposition 5.1]{dawid1993}, where $(X,\param)$ is substituted for $(\meanvec, \covmat)$,  it follows from \eqref{eq:weakjoint} that the distribution of $(\meanvec, \covmat)$ is weak hyper Markov.
The clique wise $\giwsymb(\niwnmu_{\jtnode}, \niwnnu ,\niwiwtau_{\jtnode}, \niwiwalpha)$ laws define full exponential families which are conjugate to the sampling distribution, thus by \cite[Proposition 5.9]{dawid1993}, also the strong hyper Markov property holds.
\end{proof}
By \cite[Corollary 5.2]{dawid1993} the posterior is strong hyper Markov as well and can be calculated using data corresponding to each clique.
Further as stated in \cite{dawid1993}, by substituting $\testdatapntrv$ for $\testdatapntrv^{\testdatasize+1}$ and $\param$ for $(\meanvec, \covmat)$ in Proposition 5.1, and letting $(\meanvec, \covmat)\sim \giwsymb(\niwnmustsymb, \niwnnustsymb ,\niwiwtaustsymb, \niwiwalphastsymb)$, it follows from the weak union and the decomposition property of the semi-graphoid axioms that the posterior predictive distribution for $\testdatapntrv^{\testdatasize+1}$ is Markov.
The clique-wise distributions are pairwise consistent $\tdenssymb(\tpreddfsymb, \niwnmustsymb_{\jtnode}, {\boldsymbol \Upsilon_{\jtnode}^\ast})$ distributions, yielding the CSF property of the density  
\begin{align}
    \tdenssymb_{\graph[]}(\testdatapnt^{\testdatasize+1}|\tpreddfsymb, \niwnmustsymb_{\jtnodeset}, {\boldsymbol \Upsilon_{\jtnodeset}^\ast}) = \frac{\prod_{\jtnode \in \jtnodeset} \tdenssymb({\testdatapnt_{\jtnode}^{\testdatasize + 1}}\mid\tpreddfsymb, \niwnmustsymb_{\jtnode}, {\boldsymbol \Upsilon_{\jtnode}^\ast})}
          {\prod_{\sep \in \sepset} \tdenssymb({\testdatapnt_{\sep}^{\testdatasize + 1}}\mid\tpreddfsymb, \niwnmustsymb_{\sep}, {\boldsymbol \Upsilon_{\sep}^\ast})}.
\end{align}
In the following this distribution will be referred to as the  \emph{non-central graph t-distribution} with $\tpreddfsymb$ degrees of freedom, mean vector $\niwnmustsymb$ and scale parameter $\tpredprecsymb \in \precmatrestr[\graph]$.

\subsection{Bayesian model averaging} % (fold)
\label{sub:graphical_bayesian_predictive_classifier}

In the above consideration, the graph is assumed to be fixed.
In the present study, we take a BMA approach and incorporate uncertainty about the model by regarding $\graph$ as a random variable, see e.g. \cite{madigan1995graphicalmodels}.
\cite{StructuralMarkov} introduced the family of structural Markov graph laws from which we choose the simplest, the uniform law defined as $p(\graph) = 1/|\graphsp|$ for all $\graph \in \graphsp$.
The full hierarchical structure of data generation is then structured as follows
\begin{align*}
    \graph  & \sim Unif(\graphsp) \\
    \meanvec,\covmat \mid \graph  & \sim \giwsymb_{\graph}(\niwnmu,\niwnnu, \niwiwtau, \niwiwalpha)\\
    \testdatapntrv[i] \mid \meanvec,\covmat & \sim \normsymb(\meanvec ,\covmat),
\end{align*}
for $i=1,\dots,\testdatasize+1$.
In this setting, the predictive distribution in (2) also includes marginalizing over the space of decomposable graphs as
\begin{align*}
    \sum_{\graph \in \graphsp} \int_{\precmatrestr[{\graph}]} \int_{\mathbb R^{p}}
 f(\testdatapnt^{\testdatasize+1} | \meanvec,\covmat,\graph )
 f(\meanvec,\covmat, \graph | \testdataset[\testdatasize_{\class}] ) d\meanvec d\covmat.
\end{align*}
This reduces (see \ref{sub:derivation_of_graphpred}) to the following expression
\begin{align}
f(\testdatapnt^{\testdatasize+1} | \testdataset[\testdatasize_{\class}],\class,\bf{\mathcal  C}^n)
& = \sum_{\graph \in \graphsp}p(\graph | \testdataset[\testdatasize_{\class}])\tdenssymb_{\graph[]}(\testdatapnt^{\testdatasize + 1}|\tpreddfsymb, \niwnmustsymb, \tpredprecsymb),
\label{eq:graphpred}
\end{align}
where 
\begin{align}
p(\graph | \testdataset[\testdatasize_{\class}])
& = \frac{f(\testdataset[\testdatasize_{\class}] |\graph) p(\graph)}{ \sum\limits_{\graph' \in \graphsp}f(\testdataset[\testdatasize_{\class}] |\graph') p(\graph')}
\label{eq:graphpost}
\end{align}
is the graph posterior and
\begin{align}
    f(\testdataset[\testdatasize_{\class}]|\graph)=(2\pi)^{-\testdatasize\p/2}
    \frac{\niwconstsymb_{\graph[]}(\niwiwalphastsymb, \niwiwtaustsymb, \niwnnustsymb )}{\niwconstsymb_{\graph[]}(\niwiwalpha, \niwiwtau, \niwnnu)}
    \label{eq:graphmarg}
\end{align}
is the marginal likelihood of $\graph$ derived in \ref{sub:derivation_of_graphmarg} where
\begin{align*}
        \niwconstsymb_{\graph}(\niwiwalpha, \niwiwtau, \niwnnu) = \frac{\prod_{\jtnode\in \jtnodeset}\niwconstsymb_{\graph[\jtnode]}(\niwiwalpha, \niwiwtau_{\jtnode}, \niwnnu)}{\prod_{\sep\in \sepset}\niwconstsymb_{\graph[\sep]}(\niwiwalpha, \niwiwtau_{\sep}, \niwnnu)}
\end{align*}
is the normalizing constant in $\giwsymb_{\graph}(\meanvec,\covmat\mid\niwnmu, \niwnnu ,\niwiwtau, \niwiwalpha)$.

The problem with \eqref{eq:graphpost} and \eqref{eq:graphpred} from a practical point of view, is that the number of decomposable graphs with $\p$ nodes grows proportional to $\sum_r{\p \choose r} 2^{r(\p-r)}$, see \cite{countchordal} and \cite{cta} for exact enumeration and estimation of this quantity.
Consequently exact computation of the marginalizing factor in \eqref{eq:graphpost} is intractable. 
To tackle this issue, we propose the particle Gibbs scheme with systematic refreshment introduced in \cite{nontempspmcmc} to give an approximation $\hat p(\graph|\testdataset[\testdatasize_{\class}])$ of $p(\graph|\testdataset[\testdatasize_{\class}])$, which is then plugged into \eqref{eq:graphpred} to give a predictive classifier in the sense of (3) explicitly written as
\begin{align}
\sum_{\graph \in \graphsp}\hat p(\graph | \testdataset[\testdatasize_{\class}])\tdenssymb_{\graph[]}(\testdatapnt^{\testdatasize + 1}|\tpreddfsymb, \niwnmustsymb, \tpredprecsymb).
\end{align}

% subsection graphical_bayesian_predictive_classifier (end)

\section{Particle Gibbs with systematic refreshment} % (fold)
\label{sec:bayesian_graph_structure_learning_with_the_cta_particle_gibbs_sampler}
% -*- root: main.tex -*-

In this section we describe the \emph{particle Gibbs with systematic refreshment} (PG) sampling scheme established in \cite{nontempspmcmc}, in the present context of approximating $p(\graph|\testdataset[\testdatasize_{\class}])$.
The PG \emph{sampler} constructs, using SMC, a Markov kernel $\PG{p}$ leaving an extended version of the distribution $p(\graph|\testdataset[\testdatasize_{\class}])$, denoted by $\targ{1:p}$ invariant.
%The $\targ{1:p}$ has the special property of having $p(\graph|\testdataset)$ as marginal,
We proceed this section by defining $\targ{1:p}$ by a procedure called \emph{temporalization}  and show its key property of having $p(\graph|\testdataset[\testdatasize_{\class}])$ as marginal distribution.

\subsubsection*{Step I} We first specify the target distribution of interest as the graph posterior
\begin{align}
\p(\graph \mid\testdataset[\testdatasize_{\class}]).
\end{align}

\subsubsection*{Step II } In this step we exploit  the junction tree representation for decomposable graphs.
%Recall that a junction tree for a specific graph $\graph$ is not unique in general, following the notation in \cite{doi:10.1198/jcgs.2009.07129} we denote the number of equivalent junction trees for $\graph$ by $\ntrees{\graph}$.
%On the other hand, each junction tree $\tree$ has exactly one underlying decomposable graph which we denote by $\trgr(\tree)$.
Note that, having a distribution, $p(\tree)$ over the set of junction trees directly induces a distribution over the underlying decomposable graphs since
\begin{align}
p(\graph)=\sum \limits_{\tree: \trgr(\tree)=\graph} p(\tree),
\label{eq:pmarggraph}
\end{align}
for any $\graph\in \graphsp$.
%Let $\trsp[s_m]$ be the space of junction trees with internal nodes $s_m$.
We use the junction tree posterior, introduced in \cite{Green01032013} defined as
\begin{align*}
    p(\tree \mid \testdataset[\testdatasize_{\class}]) &= p(\trgr(\tree) \mid \testdataset[\testdatasize_{\class}]) \times \frac{1}{\ntrees{\trgr(\tree)}}
%&= \frac{p(\testdataset \mid \trgr(\tree)) \pi (\trgr(\tree))}{\sum_{\graph} p(\testdataset \mid \graph) \pi (\graph)} \times \frac{1}{\ntrees{\trgr(\tree)}} \\
%&= \frac{p(\testdataset \mid \trgr(\tree[m])) }{\sum_{\graph} p(\testdataset \mid \graph) } \times \frac{1}{\ntrees{\trgr(\tree[m])}} \\
%&= \frac{p(\testdataset \mid \trgr(\tree)) }{\sum_{\graph} p(\testdataset \mid \graph) } \times \frac{1}{\ntrees{\trgr(\tree)}} \\
%&= \frac{1}{K}p(\testdataset \mid \trgr(\tree[m])) \times \frac{1}{\ntrees{\trgr(\tree[m])}} \\
%&= \frac{1}{K}p^\ast(\tree \mid \testdataset) \\
%&\propto p(\testdataset \mid \trgr(\tree)) \times \frac{1}{\ntrees{\trgr(\tree)}},
\propto p(\testdataset[\testdatasize_{\class}] \mid {\trgr(\tree)})  \times \frac{ 1}{\ntrees{\trgr(\tree)}},
\label{eq:jtpost}
\end{align*}
for any $\tree \in \trsp$.
That is, each junction representation of any specific graph is assigned equal probability.
With this definition, by combining
%\eqref{eq:pmarggraph} and \eqref{eq:jtpost},
the above equations, as shown in \cite{doi:10.1198/jcgs.2009.07129} the induced graph distribution is the graph posterior of interest since
\begin{align*}
   \sum \limits_{\tree: \trgr(\tree)=\graph} p(\tree\mid\testdataset[\testdatasize_{\class}])
              = \sum \limits_{\tree: \trgr(\tree)=\graph} p(\trgr(\tree) \mid \testdataset[\testdatasize_{\class}]) \times \frac{1}{\ntrees{\trgr(\tree)}}
              %&= \sum \limits_{\tree: \trgr(\tree)=\graph} p(\graph \mid \testdataset) \times \frac{1}{\ntrees{\graph}} \\
              % &= \ntrees{\graph} \times p(\graph \mid \testdataset) \times \frac{1}{\ntrees{\graph}} \\
              = p(\graph \mid \testdataset[\testdatasize_{\class}]).
              %\label{eq:graphpost}
\end{align*}
%The idea of the algorithm is to approximate $p(\graph\mid\testdataset)$.
%where $K=\sum_{\graph} p(\testdataset \mid \graph)$ is the normalizing constant in the unnormalized junction tree posterior $p^\ast(\tree \mid \testdataset)=p(\testdataset \mid \trgr(\tree)) \times \frac{1}{\ntrees{\trgr(\tree)}}$.

%Let, for all $m =1,\dots,\p, \combsp{m}$ be the space of all subsets of $\{1,\dots,\p\}$ with cardinality $m$.

\subsubsection*{Step III} In this step, we extend the junction tree representation and its distribution by a variable corresponding to the nodes in the underlying graph.
%, and then since we are using sequential algorithm by considering paths of pairs of junction trees and there internal nodes.
% We begin by introducing the joint space of junction trees and corresponding internal nodes by the following setup.
Let, for all $m \in \intvect{1}{p}$, $\combsp{m}$ be the space of all \emph{$m$-combinations} of elements in $\intvect{1}{p}$.
%In particular, $\combsp{p} = \{Ê\intvect{1}{p} \}$.
An element $\comb{m} \in \combsp{m}$ is of form $\comb{m} = (\comb[1]{m}, \ldots, \comb[m]{m})$ where $\{\comb[\ell]{m} \}_{\ell = 1}^m \subseteq \intvect{1}{p}$ are distinct.
For $(\ell, \ell') \in \intvect{1}{m}^2$ such that $\ell \leq \ell'$, we denote $\comb[\ell:\ell']{m} \eqdef (\comb[\ell]{m}, \ldots, \comb[\ell']{m})$.
Denote by $\trsp[\comb{m}]$ the space of junction trees with underlying graph nodes in $\comb{m} \in \combsp{m}$.
We define, for all $m=1,\dots,\p$, the extended state spaces of junction trees with internal nodes in $\combsp{m}$ by
\begin{align*}
    \xsp{m} \eqdef \bigcup_{\comb{m} \in \combsp{m}}
    \left( \{\comb{m} \} \times \trsp[\comb{m}] \right).
\end{align*}
%and, for a given discrete probability distribution $\combmeas{m}$ on $\combsp{m}$, extended target distributions
Next, we define discrete probability distribution $\combmeas{m}$ on $\combsp{m}$, for the extended target distributions
$$
    \targ{m}( \x{m})
    =  \frac{\untarg{m}( \x{m})}{\sum \limits_{\x{m} \in \xsp{m}} \untarg{m}(\x{m}) },
$$
%in $\probmeas{\xfd{m}}$,
where
$$
    \untarg{m}( \x{m})
%    = \untarg{m}(\rmd \comb{m}, \rmd \tree[m])
%    \eqdef \untarg[marg]{\comb{m}}(\rmd \tree[m])   \, \combmeas{m}(\rmd \comb{m})
    \eqdef
 \combmeas{m}( \comb{m}) \times  \frac{p( \testdataset[\testdatasize_{\class}] | \trgr(\tree[\comb{m}]) ) }{\ntrees{\trgr(\tree[\comb{m}])}}.
$$
%Here we have chosen to write $\tree[m]$ instead of $\tree[\comb{m}]$ in order to avoid double subscript notation.
%Note that each measure $\untarg{m}(\rmd \x{m})$ has a density $\untarg{m}(\x{m}) = \untarg[marg]{\comb{m}}(\tree[m]) \combmeas{m}(\comb{m})$ (by abuse of notation, we reuse the same symbol) w.r.t. $\cm{\rmd \x{m}}$, the counting measure on $\xsp{m}$.
Moreover, since  $\combsp{p} = \{\intvect{1}{p}\}$ we note that $\combmeas{p} = \delta_{\intvect{1}{p}}$, implying that $p(\tree[\comb{p}] |\testdataset[\testdatasize_{\class}])$ is the marginal distribution of $\targ{p}$ with respect to the $\tree[\comb{p}]$ component.
It is crucial that the distributions $\{\combmeas{m} \}_{m = 1}^p$ satisfies the recursion
\begin{align*}
    \combmeas{m + 1}(\comb{m+1}) = \sum_{\comb{m}\in \combsp{m}} \combkernelpath{m}(\comb{m}, \comb{m+1})\combmeas{m}(\comb{m}),
\end{align*}
where
\begin{equation*} \label{eq:def:comb:kernel:path}
    \combkernelpath{m}(\comb{m},  \comb{m + 1})
    \eqdef \delta_{\comb{m}}( \comb[1:m]{m + 1})
    \, \combkernel{m}(\comb[1:m]{m + 1}, \comb[m + 1]{m + 1})
\end{equation*}
with $\combkernel{m}$ being a Markov transition kernel density \footnote{All densities on finite countable spaces in this paper are defined with respect to the counting measure on the same space.} from
$\combsp{m}$ to $\intvect{1}{p}$ such that $\combkernel{m}(\comb{m}, j) = 0$
for all $j \in \comb{m}$.
Here we let for $j \in \intvect{1}{p} \setminus \comb{m}$, 
\begin{align*}
\combkernel{m}(\comb{m}, j) = \frac{1}{p-m},
\end{align*}
where $m\ge2$.
Further, by letting
\begin{align*}
     \combmeas{1}( \comb{1}) = \frac{1}{\p},
 \end{align*}
we obtain at each step $m$, $\combmeas{m}(\comb{m})$ as the uniform distribution over the nodes in $ \intvect{1}{p} \setminus \comb{m}$.

\subsubsection*{Step IV} The last step in the temporalization is the extension to the path space $\xsp{1:m}\eqdef \times_{\ell=1}^m \xsp{\ell}$.
On this space we define a proposal kernel $\prop{m}$ of form
\begin{equation*}% \label{eq:proposal:special:form}
    \prop{m}(\x{m},  \x{m + 1})
    = \combkernelpath{m}(\comb{m},  \comb{m + 1}) \,
    \prop[\comb{m}, \comb{m + 1}]{m}(\tree[m],  \tree[m + 1]),
\end{equation*}
where $ \combkernelpath{m}$ is defined in \eqref{eq:def:comb:kernel:path} and for all $(\comb{m}, \comb{m + 1}) \in \combsp{m} \times \combsp{m + 1}$, $\prop[\comb{m}, \comb{m+1}]{m}: \trsp[\comb{m}] \times  \trsp[\comb{m+1}] \rightarrow [0, 1]$ is the transition kernel in the so called Christmas tree algorithm (CTA) established in \cite{cta}.

The extension of $\untarg{m}$ to $\xsp{1:m}$, is achieved through the backward versions of the kernels $\prop{m}$, denoted by $\bk[]{m} : \xsp{m+1} \times \xsp{m} \rightarrow [0, 1]$, where for each $m$,
\begin{align*}
    %\label{eq:support:condition:bk}
    \bk{m}(\x{m + 1}, \x{m}) \eqdef |\retrosupp{m}(\x{m + 1})|^{-1}
    \1_{\retrosupp{m}(\x{m + 1})}(\x{m})
    \quad (\x{m + 1} \in \xsp{m+1}),%\supp(\untarg{m + 1})),
\end{align*}
i.e., $\bk{m}(\x{m + 1}, \cdot)$ is the uniform distribution
over $\retrosupp{m}(\x{m + 1})$, the support of $\prop{m}(\cdot, \x{m + 1})$.
The resulting distribution on $\xsp{1:m}$ is defined as
\begin{equation*}% \label{eq:def:extended:target}
    \untargpath{m}(\x{1:m}) \eqdef
   %\untarg{m}( \x{m})
   \untarg{m}(\x{m})
    \prod_{\ell = 1}^{m - 1} \bk[]{\ell}(\x{\ell+1}, \x{\ell}).
\end{equation*}
Under this construction, since each $\bk{m}$ is Markovian, the marginal distribution of $\x{\p}$ is exactly the distribution $p(\tree[\comb{\p}]|\testdataset[\testdatasize_{\class}])$.

%In the following, we briefly discuss how to sample from the extended target $\targ{1:p}$, having the distribution $\maintarg$ of interest as a marginal distribution, using \emph{Markov chain Monte Carlo} (MCMC) methods.
%The \emph{particle Gibbs} (PG) \emph{sampler} constructs, using SMC, a Markov kernel $\PG{p}$ leaving $\targ{1:p}$ invariant.
We use the notation $\disc(\{ a_\ell \}_{\ell = 1}^\N)$ to denote the categorical probability distribution induced by a set $\{ a_\ell \}_{\ell = 1}^\N$ of positive (possibly unnormalised) numbers; thus, writing $W \sim \disc(\{ a_\ell \}_{\ell = 1}^\N)$ means that the variable $W$ takes the value $\ell \in \intvect{1}{\N}$ with probability $a_\ell / \sum_{\ell' = 1}^\N a_{\ell'}$.

The algorithm is as follows. 
It is initiated by one draw, $\X{1:p}[1] \sim \disc(\{ \wgt{\p}{\ell} \}_{\ell = 1}^\N)$, where the weights $\{ \wgt{m}{\ell} \}_{\ell = 1}^\N$ are generated by the standard SMC algorithm stated in Algorithm \ref{alg:SMC:update}.
Then the PG kernel is applied to generate samples $\X{1:p}[2],\X{1:p}[3],\dots,\X{1:p}[\Nmcmc]$.
Algorithmically, the more or less only difference between the PG kernel and the standard SMC algorithm is that the PG kernel, which is described in detail in Algorithm~\ref{alg:particle:Gibbs:kernel}, evolves the particle cloud \emph{conditionally} on a fixed reference trajectory specified \emph{a priori}; this \emph{conditional SMC} algorithm is constituted by Lines~1--16 in Algorithm~\ref{alg:particle:Gibbs:kernel}.
After having evolved, for $p$ time steps, the particles of the conditional SMC algorithm, the PG kernel draws randomly a particle from the last generation (Lines~17--19).
 % where $\disc(\{ a_\ell \}_{\ell = 1}^\N)$ denotes the categorical probability distribution induced by a set $\{ a_\ell \}_{\ell = 1}^\N$ of positive (possibly unnormalised) numbers; thus, writing $W \sim \disc(\{ a_\ell \}_{\ell = 1}^\N)$ means that the variable $W$ takes the value $\ell \in \intvect{1}{\N}$ with probability $a_\ell / \sum_{\ell' = 1}^\N a_{\ell'}$.
It then samples a new genealogical history of the selected particle back to the first generation using the backward kernel (Lines~20--22), and returns the traced path (Line~23).
The backward sampling procedure is the so called refreshment step which serves to improve mixing.
For a complete presentation of the PG sampler and related results see \cite{nontempspmcmc}  and \cite{andrieu2010particle}.

It can be shown that $\PG{p}$ is $\targ{1:p}$-reversible and thus leaves $\targ{1:p}$ invariant, see \cite[Proposition~8]{chopin:singh:2015}.
% Interestingly, reversibility holds true for any particle sample size $\N \ge 2$.
Thus, on the basis of $\PG{p}$, the PG sampler generates a Markov chain $\{\X{1:p}[\ell] \}_{\ell=1}^\Nmcmc$ according to
 \begin{align*}
    \X{1:p}[1] \stackrel{\PG{p}}{\longrightarrow}
    \X{1:p}[2] \stackrel{\PG{p}}{\longrightarrow}
    \X{1:p}[3] \stackrel{\PG{p}}{\longrightarrow}
    \cdots
    \stackrel{\PG{p}}{\longrightarrow}
    \X{1:p}[\Nmcmc],
 \end{align*}
where  $\X{1:p}[l]\in \xsp{\p}$, $l=1,\dots,\Nmcmc$.
%  and returns $\frac{1}{\Nmcmc}\sum_{\ell = 1}^\Nmcmc h(\X{1:p}[\ell])$ as an estimate of $E[h]$ for any
%  %$\targ{1:p}$-integrable objective 
%  function $h$ defined on $\xfd{1:p}$, where $\Nmcmc$ denotes the MCMC sample size.
% In particular, in the case here, where the 
% %objective 
% function $h$ depends on the argument $\tree[p]$ only,
An unbiased estimator of $p(\tree\mid \testdataset[\testdatasize_{\class}])$ is obtained as
\begin{align*}
    \hat p(\tree\mid \testdataset[\testdatasize_{\class}]) = \frac{1}{\Nmcmc}\sum_{\ell = 1}^\Nmcmc \1_{\tree}(\Z{p}[\ell]),
    \label{eq:mcmc_estimator}
\end{align*}
where each $\Z{p}[\ell]$ variable is extracted, on Line~18, at iteration $\ell - 1$ of Algorithm~\ref{alg:particle:Gibbs:kernel}.
The estimator of $p(\graph\mid\testdataset[\testdatasize_{\class}])$ now follows directly as $\hat p(\graph\mid\testdataset[\testdatasize_{\class}]) = \sum \limits_{\tree: \trgr(\tree)=\graph} \hat p(\tree\mid\testdataset[\testdatasize_{\class}])$.
\begin{algorithm}[H] \label{alg:SMC:update}
     \KwData{$\{ (\epart{m}{i}, \wgt{m}{i}) \}_{i = 1}^\N$}
     \KwResult{$\{ (\epart{m + 1}{i}, \wgt{m + 1}{i}) \}_{i = 1}^\N$}
     \For{$i \gets 1, \ldots, \N$}{
         draw $\ind{m + 1}{i} \sim \disc( \{ \wgt{m}{\ell} \}_{\ell = 1}^\N)$\;
         draw $\combpart{m + 1}{i} \sim \combkernelpath{m}
         (\combpart{m}{\ind{m + 1}{i}}, \rmd \comb{m + 1})$\;
         draw $\treepart{m + 1}{i} \sim \prop[\combpart{m}{\ind{m + 1}{i}}, \combpart{m + 1}{i}]{m}
         (\treepart{m}{\ind{m + 1}{i}}, \rmd \tree[m + 1])$\;
         set $\epart{m + 1}{i} \gets (\combpart{m + 1}{i}, \treepart{m + 1}{i})$\;
         set $\displaystyle \wgt{m + 1}{i} \gets \frac{p(\testdataset[\testdatasize_{\class}] | \trgr(\zpart{\comb{m+1}}{i} )) \ntrees{\trgr(\zpart{\comb{m}}{\ind{m + 1}{i}} )}
         \bk{m}(\epart{m}{\ind{m + 1}{i}}, \epart{m + 1}{i})}
         {p(\testdataset[\testdatasize_{\class}] |{\trgr(\zpart{\comb{m}}{\ind{m + 1}{i}} )}) \ntrees{\trgr(\zpart{\comb{m+1}}{i} )}
         \prop[\combpart{m}{\ind{m + 1}{i}}, \combpart{m + 1}{i}]{m}
         (\zpart{m}{\ind{m + 1}{i}}, \zpart{m + 1}{i})}$\;
         %set $\displaystyle \wgt{m + 1}{i} \gets \frac{\untarg[marg]{\combpart{m + 1}{i}}
         %(\zpart{m + 1}{i}) \bk{m}(\epart{m}{\ind{m + 1}{i}}, \epart{m + 1}{i})}
         %{\untarg[marg]{\combpart{m}{\ind{m + 1}{i}}}(\zpart{m}{\ind{m + 1}{i}})
         %\prop[\combpart{m}{\ind{m + 1}{i}}, \combpart{m + 1}{i}]{m}
         %(\zpart{m}{\ind{m + 1}{i}}, \zpart{m + 1}{i})}$\;
     }
     \caption{SMC update.}
\end{algorithm}
%\bigskip
%\newpage
\begin{algorithm}[H] \label{alg:particle:Gibbs:kernel}
     \KwData{a reference trajectory $\x{1:p} \in \xsp{1:p}$}
     \KwResult{a draw $\X{1:p}$ from $\PG{p}(\x{1:p},\cdot)$} %$\PG{p}(\x{1:p}, \rmd \x{1:p}')$}
     \For{$i \gets 1, \ldots, \N - 1$}{
         %draw $\combpart{1}{i} \sim \unifdens{\{1,\dots,\p\}}$ \;
         draw $\combpart{1}{i} \sim \combmeas{1}(\comb{1})$ \;
         set $\zpart{1}{i} \gets (\{\combpart{1}{i}\}, \emptyset)$ \;
         %draw $\zpart{1}{i} \sim \partinit[\combpart{1}{i}](\rmd \tree[1])$\;
         set $\epart{1}{i} \gets (\combpart{1}{i}, \zpart{1}{i})$\;
     }
     set $\epart{1}{\N} \gets \x{1}$\;
     \For{$i \gets 1, \ldots, \N$}{
         set $\wgt{1}{i} \gets 1$\;
        % / \partinit[\combpart{1}{i}](\zpart{1}{i})$\;
     }
     \For{$m \gets 1, \ldots, p - 1$}{
         \For{$i \gets 1, \ldots, \N - 1$}{
             draw $\ind{m + 1}{i} \sim \disc( \{ \wgt{m}{\ell} \}_{\ell = 1}^\N)$\;
             draw $\combpart{m + 1}{i} \sim \combkernelpath{m} (\combpart{m}{\ind{m + 1}{i}}, \cdot)$\;
             %draw $s \sim \unifdens{ \{1,\dots,\p\} \setminus \combpart{m}{\ind{m + 1}{i}} }$\;
             %$\combpart{m + 1}{i} \gets \combpart{m}{i} \cup s$ \;

             draw $\zpart{m + 1}{i} \sim \prop[\combpart{m}{\ind{m + 1}{i}}, \combpart{m + 1}{i}]{m}
             (\zpart{m}{\ind{m + 1}{i}}, \cdot)$\;
             set $\epart{m + 1}{i} \gets (\combpart{m + 1}{i}, \zpart{m + 1}{i})$\;
         }
         set $\epart{m + 1}{\N} \gets \x{m + 1}$\;
         \For{$i \gets 1, \ldots, \N$}{
             % set $\displaystyle \wgt{m + 1}{i} \gets \frac{\untarg[marg]{\combpart{m + 1}{i}}
             % (\zpart{m + 1}{i}) \bk{m}(\epart{m}{\ind{m + 1}{i}}, \epart{m + 1}{i})}
             % {\untarg[marg]{\combpart{m}{\ind{m + 1}{i}}}(\zpart{m}{\ind{m + 1}{i}})
             % \prop[\combpart{m}{\ind{m + 1}{i}}, \combpart{m + 1}{i}]{m}
             % (\zpart{m}{\ind{m + 1}{i}}, \zpart{m + 1}{i})}$\;
             set $\displaystyle \wgt{m + 1}{i} \gets \frac{p(\testdataset[\testdatasize_{\class}] | \trgr(\zpart{\comb{m+1}}{i} )) \ntrees{\trgr(\zpart{\comb{m}}{\ind{m + 1}{i}} )} \bk{m}(\epart{m}{\ind{m + 1}{i}}, \epart{m + 1}{i})}
             {p( \testdataset[\testdatasize_{\class}]| \trgr(\zpart{\comb{m}}{\ind{m + 1}{i}} )) \ntrees{\trgr(\zpart{\comb{m+1}}{i} )}
             \prop[\combpart{m}{\ind{m + 1}{i}}, \combpart{m + 1}{i}]{m}
             (\zpart{m}{\ind{m + 1}{i}}, \zpart{m + 1}{i})}$\;
         }
     }
     draw $\gen{p} \sim \disc( \{ \wgt{p}{\ell} \}_{\ell = 1}^\N )$\;
     set $\Z{p} \gets \zpart{p}{\gen{p}}$\;
     set $\X{p} \gets (\intvect{1}{p}, \Z{p})$\;

     \For{$m \gets p - 1, \ldots, 1$}{
         draw $\X{m} \sim \bk{m}(\X{m + 1}, \cdot)$\;
     }
     set $\X{1:p} \gets (\X{1}, \ldots, \X{p})$\;

     % \For{$m \gets p - 1, \ldots, 1$}{
     %     set $\gen{m} \gets \ind{m}{\gen{m + 1}}$\;
     %     set $\X{m} \gets \epart{m}{\gen{m + 1}}$\;
     % }
     % set $\X{1:p} \gets (\X{1}, \ldots, \X{p})$\;
     \Return{$\X{1:p}$}
     \caption{One transition of PG.}
\end{algorithm}

% \begin{algorithm}[H] \label{alg:particle:Gibbs:refreshment}
%      \KwData{a reference trajectory $\x{1:p} \in \xsp{1:p}$}
%      \KwResult{a draw $\X{1:p}$ from $\PG{p}\G{p}(\x{1:p}, \rmd \x{1:p}')$}
%      draw, using Algorithm~\ref{alg:particle:Gibbs:kernel},
%      $\X{1:p}' \sim \PG{p}(\x{1:p}, \rmd \x{1:p}')$\;
%      set $\X{p} = (\intvect{1}{p}, \Z{p}) \gets \X{p}'$\;
%      \For{$m \gets p - 1, \ldots, 1$}{
%          draw $\X{m} \sim \bk{m}(\X{m + 1}, \rmd \x{m})$\;
%      }
%      set $\X{1:p} \gets (\X{1}, \ldots, \X{p})$\;
%      \Return{$\X{1:p}$}
%      \caption{One transition
%      of PG with systematic refreshment.}
% \end{algorithm}

% subsection bayesian_graph_structure_learning_with_the_cta_particle_gibbs_sampler (end)

\section{Numerical study} % (fold)
\label{sec:simulation_study}
% -*- root: main.tex -*-
We demonstrate the performances of the suggested BMA classifier by one realistic- and three synthetic datasets, illustrating different typical classification scenarios.
In each of the examples the underlying graph distribution was estimated by the PG sampler, where the number of particles, $\N$ were set to $50$ and the number of Gibbs samples, $M$ where set to $2000$.
The burn-in period where deduced by visual inspection of likelihood traces of the sampled graphs.
The CTA proposal kernels $\{\prop[]{m}\}_{m=1}^\p$ of the SMC algorithm has two tuning parameters, $\alpha$ and $\beta$ which were both set to $0.5$, reflecting an assumption of moderately sparse graphs.
We refer to \cite{cta} for a more detailed description of how of these parameters influence the sparsity of the generated junction trees.

The performance of our classifier were compared to 11 different out-of-the-box classifiers.
% \emph{3-NN} is the 3 nearest neighbors classifier.
% \emph{Gaussian process} is the Gaussian process classifier has the squared exponential kernel with length-scale parameter set to 1.0.
% \emph{Decision tree} has maximum depth 5.
% \emph{Naive Bayes} is the standard Naive Bayes classifier.
% \emph{Linear SVM} is the linear support vector machine with penalty parameter 0.025 for the error term.
% \emph{Neural net} is a neural network with $L_2$ penalty set to 1.0.
% \emph{RBF SVM} is the support vector machine with th radial basis function kernel and penalty parameter 1.0.
% and the results are summarized in Table \ref{tab:classification-results}.
We also included the BMA classifier where the graph distributions was set to be a point mass at the true graph respectively.
For further details about the implementation and parameterization of these classifiers, the reader is referred to the python library which can be obtained from the third author.

%  and the distance parameter, $\delta$ in the SMC algorithm used in the Gibbs sampler were set to $p$, in order to reflect that no assumption of ordering between the variables is made.
%The estimated graph distributions where burn-in for each Gibbs chain was set from visual inspection of log-likelihood of the graphs in the corresponding chains.

\subsection{Synthetic data} % (fold)
\label{sub:simulations}
The number of nodes in each of the synthetic examples were fixed to $p=50$
%The graphs $\graph[c]$ for $c=1,\dots, |C|$ in each of the examples A, B and C, was
and the graphs were generated by the CTA with $\alpha=0.5$ and $\beta=0.5$.
For each example and each class $\class$, the data were sampled from $\normsymb_{\graph[\class]}(\meanvec[\class], \covmat_\class)$, where the subscript $\class$ now indicates class belonging.
We defined $\covmat[\class]$ so as to fulfill
\begin{align*}
    (\covmat_\class)_{ij} = \begin{cases}
        \sigma^2, &\text{ if } i=j\\
        \rho\sigma^2, &\text{ if } (i,j) \in \graph[\class] \\
    \end{cases}
\end{align*}
and $(\covmat^{-1}_\class)_{ij} =0$ if $(i,j) \notin \graph[\class]$.
This is the graphical intraclass structure considered in \cite{doi:10.1198/jcgs.2009.07129}.
Here, we have fixed the variance and correlation parameters $\sigma^2$ and $\rho$ to 1.0 and 0.5 respectively.
The underlying graph and class centroids, $\meanvec[\class]$ for the two scenarios are described below.

% \subsubsection*{A) Two classes with equal graphs} % (fold)
% \label{sub:two_classes_with_equal_graphs}
% %In this example, the data were generated from two classes where the underlying graphs both have the adjacency matrix in Figure \ref{fig:A}.
% The class centroids where separated as $\meanvec[1] = \mathbf 0$ and $\meanvec[2] = \Delta \times \mathbf 1$, where $\mathbf 0$ and $\mathbf 1$ are the vectors of zeros and ones respectively and $\Delta=0.5$.

\subsubsection*{A) Two classes with different graphs} % (fold)
\label{sub:two_classes_with_different_graph}
%Here, the data were generated from two classes with distinct the graphs in Figure \ref{fig:B}.
The class centroids where separated as $\meanvec[1] = \mathbf 0$ and $\meanvec[2] = \Delta \times \mathbf 1$, where $\mathbf 0$ and $\mathbf 1$ are the vectors of zeros and ones respectively and $\Delta=0.0001$.
%Similarly to A) the class centroids where separated as $\meanvec[1] = \mathbf 0$ and $\meanvec[2] = \Delta \times \mathbf 1$ but here $\Delta=0.0001$.

\subsubsection*{B) Three classes with different graphs} % (fold)
\label{sub:three_classes_with_different_graphs}
% subsection three_classes_with_different_graphs (end)=
%The adjacency matrices of the graphs are found in Figure \ref{fig:C}.
The class centroids in this example where separated as $\meanvec[1] = \mathbf 0$, $\meanvec[2] = ((i \mod 2) \times \Delta)_{i=1}^p$ and $\meanvec[3]=((i+1 \mod 2) \times \Delta)_{i=1}^p$ and $\Delta=0.0001$.

These choices of separating the class centroids showed to reflect narrowness between the classes.

For each of the scenarios, the correct classification rate was calculated on 10 independently generated datasets each consisting of $50$ test samples for each class.
The training datasets consisted of $n=51$ and $n=300$ samples for each class.
The correct classification probability was then estimated by their means and summarized in Table \ref{tab:classification-results}.

The hyper parameters for the hyper normal inverse Wishart prior were set to $\niwnnu=1, \niwiwtau=\mathbf I$ and $ \niwiwalpha=\p$ in all classes for each of the examples.
For $\niwnmu_\class$ we used the empirical mean computed separately in each of the classes.

%The class centroid hyper parameter were set as $\niwnmu_c = \meanvec[c]$.

%Our suggested BMA classifier shows better results than the standard predictive classifier in each of the examples [extend].

The results of the graph posterior estimation by the PG sampler for one of the 10 dataset replicates from example A and B are shown in Figure \ref{fig:B}-\ref{fig:C}.
Each column corresponds to one class.
%Example A) and C) shown in Figure \ref{fig:A} and Figure \ref{fig:C} has similar sparsity pattern for their graphs and show similar results.
%For A), since the underlying graphs were equal, the graph posterior was estimated using data from both of the classes, for  B) and C) it was estimated separately for each class.
The first row in each figure show the adjacency matrices for the underlying graphs in each of the classes and the second row shows the estimated heatmaps.
A dark color at position $i,j$, counted from the top left corner, indicates high marginal posterior probability of the edge $(i,j)$.
It is interesting to note that even though the pattern in the heatmaps for the $\testdatasize=51$ case does not resemble those in the the corresponding underlying graph so well, the BMA classifier still performs better than the standard Bayesian predictive classifier as seen in Table \ref{tab:classification-results}.
Also, with this small amount of data one could hardly expect to see the true pattern in the posterior heatmaps.
For $\testdatasize=300$, we observe better correspondence between the heatmap and the true underlying graph, and the true classification probability is also higher.
However, the impact of the graph seem to degrease when the number of data increases, since the standard Bayesian predictive classifier shows comparable results to our BMA approach, (0.89 to 0.93).
For the $\testdatasize=51$ case, there is a greater gap between the standard approach and the BMA approach (0.73 to 0.82).

The third row shows the log-likelihood for the sampled graphs (blue line), the true underlying graph (green line) and the complete graph (red line), where the complete graph corresponds to the standard Bayesian posterior predictive classifier.
As the log-likelihood for the complete graph is substantially lower than those sampled by the PG sampler for both classes, the plot illustrates the motivation of the BMA approach.
Remarkably, the log-likelihood for the true underlying graph is also lower than those samples by the PG sampler in the $\testdatasize=51$ case.
This is explained by the relatively small amount of data that we are using, for $\testdatasize=300$ the plots has the expected relation.

%The heat map is expected to show a similar pattern to the true underlying graph.
The last row shows the estimated auto-correlations of the number of edges in the graphs (\emph{graph size}) generated by the PG sampler.
The dependence seem to decline to zero after about 100-200 lags.
Both the autocorrelation and the heatmaps were estimated after a burn-in period, deduced from the log-likelihood plots.

Throughout each of the examples, the BMA classifier showed better results than the standard Bayesian predictive approach using the similar parametrization.
%This is in line with the theoretical results about ... given in \cite{}
Also the BMA classifier out performed the rest of the classifiers as well.
The BMA classifier with point mass at the true graph is the best choice in each of the examples.
%Using the MAP estimate instead of the full graph distribution does not seem to affect the results remarkably.
%However, other values of their corresponding parametrization might change the results.
%$\alpha$ and $\beta$

%The model was then trained on the training sets and the miss-classification rate was calculated on the corresponding test set.
%The miss-classification probability was then estimated by averaging the miss-classification rate of calculated on the tests sets where the model was training on a corresponding training set.
%The prediction error was then estimated by averaging the miss-classification on 10 test sets containing $n=50$ samples each, where each model was trained on 10 corresponding training sets containing $n=51$ samples.

\begin{table}[]
\centering
\begin{tabular}{lllll}
%\hline
%&\multicolumn{4}{l}{Classification rate} \\\hline
                       & Walking             & A, $\testdatasize=300$                       & A, $\testdatasize=51$                       & B, $\testdatasize=51$\\ \hline
Bayes pred             & 0.809 (0.050)         & 0.888 (0.041)             & 0.725 (0.040)             & 0.613 (0.052)          \\%\hline
BMA (true)             &  -                  & 0.948 (0.015)             & 0.904 (0.037)             & 0.877 (0.029)          \\%\hline
\bf{BMA (PG)}          & {\bf0.848} (0.059)    & {\bf 0.933} (0.025)       & {\bf 0.816} (0.047)       & {\bf 0.788} (0.038)    \\%\hline
%BMA (PG-MAP)          & 0.85 (0.06)         & 0.927 (0.03)              & 0.83 (0.04)             & 0.78 (0.04)          \\%\hline
3-NN                   & 0.825 (0.053)         & 0.716 (0.054)             & 0.670 (0.040)             & 0.558 (0.026)          \\%\hline
Gauss proc             & 0.829 (0.046)         & 0.700 (0.047)             & 0.652 (0.029)             & 0.538 (0.028)          \\%\hline
Dec tree (5)           & 0.761 (0.045)         & 0.548 (0.050)             & 0.543 (0.063)             & 0.373 (0.047)          \\%\hline
Naive Bayes            & 0.740 (0.071)         & 0.480 (0.026)             & 0.504 (0.042)             & 0.340 (0.036)          \\%\hline
Linear SVM             & 0.644 (0.070)         & 0.471 (0.041)             & 0.487 (0.032)             & 0.358 (0.050)          \\%\hline
Neural net             & 0.774 (0.098)         & 0.825 (0.047)             & 0.631 (0.050)             & 0.521 (0.036)          \\%\hline
%RBF SVM               & 0.57 (0.04)         & 0.500 (0.00)                & 0.51 (0.01)             & 0.34 (0.01)          \\%\hline
AdaBoost               & 0.791 (0.051)         & 0.488 (0.053)             & 0.496 (0.036)             & 0.331 (0.036)          \\%\hline
LDA                    & 0.666 (0.044)         & 0.482 (0.031)             & 0.490 (0.038)             & 0.348 (0.047)          \\%\hline
QDA                    & 0.741 (0.095)         & 0.889 (0.041)             & 0.535 (0.032)             & 0.391 (0.022)          \\%\hline
Rand forest            & 0.700 (0.078)         & 0.541 (0.055)             & 0.526 (0.057)             & 0.348 (0.040)          \\%\hline
\end{tabular}
\vspace{2mm}
\caption{Estimated probabilities (standard errors) of correct classification  averaged over the 10 test sets for the walking / Nordic walking data and the synthetic datasets.}
\label{tab:classification-results}

\end{table}

%  \begin{figure}
%     \centering
%     \begin{subfigure}[t]{0.33\textwidth}
%         \includegraphics[width=\textwidth]{figures/2classes_p50_same_graph_adjmat_class_0}
%         %\caption{True adjacency matrix.}
%         \label{fig:two_classes_same_graph_adjmat}
%     \end{subfigure}

%    \begin{subfigure}[t]{0.33\textwidth}
%        %\includegraphics[width=\textwidth]{{figures/2classes_p50_same_graph_m05_reps10_tmp_rep_5_edge_heatmap_burnin_1000_class_0}}
%     %\caption{Heatmap estimate of the underlying adjacency matrix.}
%     \label{fig:two_classes_with_same_graph_heatmap}
%    \end{subfigure}
%     %\caption{Dataset A.}
%     %\label{fig:two_classes_same_graph}
% % \end{figure}

% %  \begin{figure}
% %     \centering
%     \begin{subfigure}[t]{0.33\textwidth}
%         \includegraphics[width=\textwidth]{figures/2classes_p50_same_graph_m05_reps10_rep_5_autocorr_size_burnin_500_class_0}
%     %\caption{Autocorrelation for the number of edges in the graphs generated by PGibbs.}
%     \label{fig:two_classes_with_same_graph_heatmap}
%     \end{subfigure}

%     \begin{subfigure}[t]{0.33\textwidth}
%         \includegraphics[width=\textwidth]{{figures/2classes_p50_same_graph_m05_reps10_rep_5_ll_class-group_0}}
%         %\caption{Log-likelihood for the graphs generated by PGibbs (blue), the true graph (green) and the complete graph (red).}
%     \label{fig:two_classes_with_different_graph}
%     \end{subfigure}
%     \caption{Dataset A.}%: summary of one of the 10 trajectories generated by the PG sampler.}
%     \label{fig:A}
% \end{figure}

% subsection two_classes_with_equal_graphs (end)

\begin{figure}
    \centering
    \begin{subfigure}[t]{0.33\textwidth}
        \includegraphics[width=\textwidth]{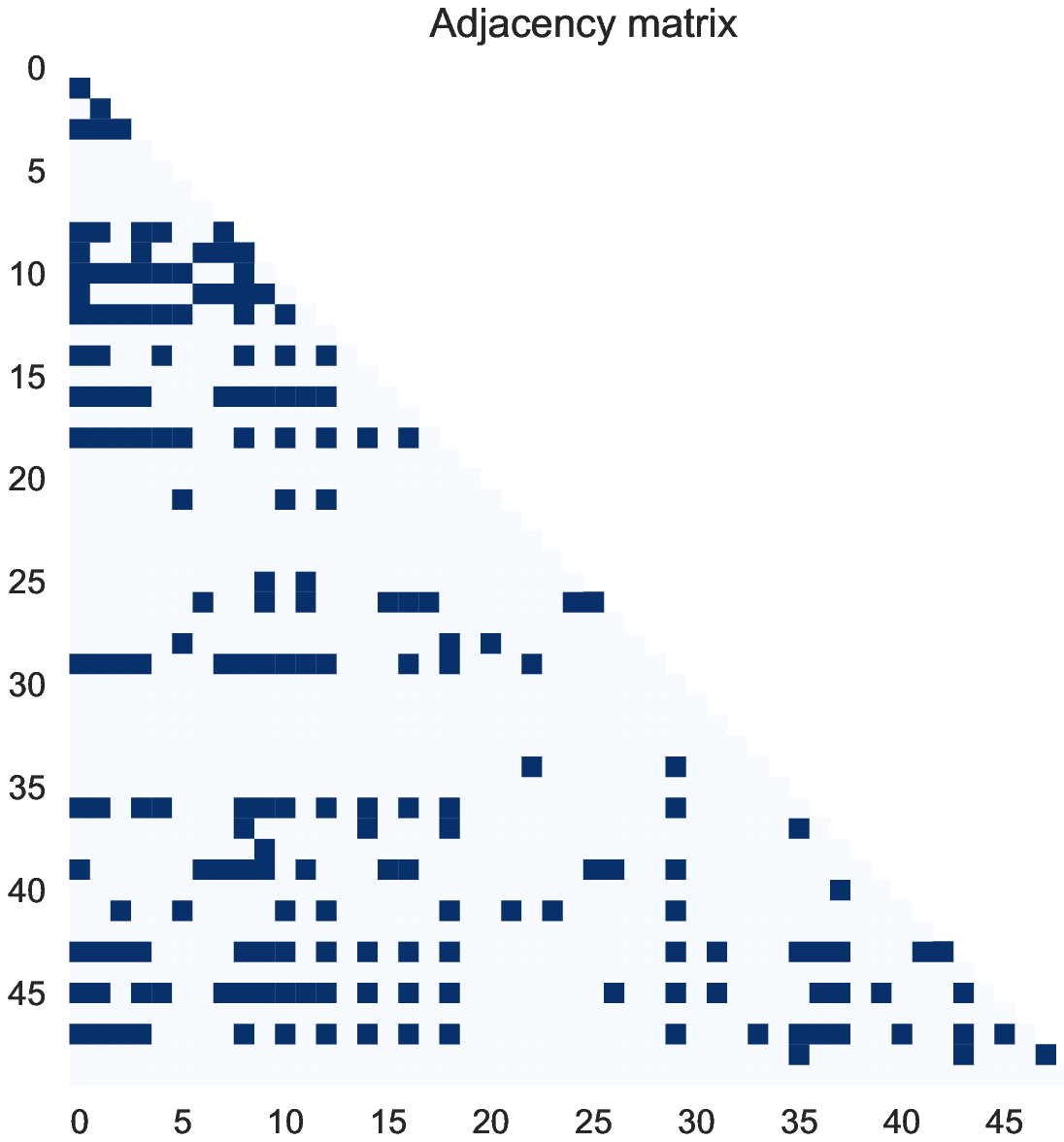}
        %\caption{}
        %\label{fig:two_classes_same_graph_adjmat}
    \end{subfigure}
   ~
   \begin{subfigure}[t]{0.33\textwidth}
       \includegraphics[width=\textwidth]{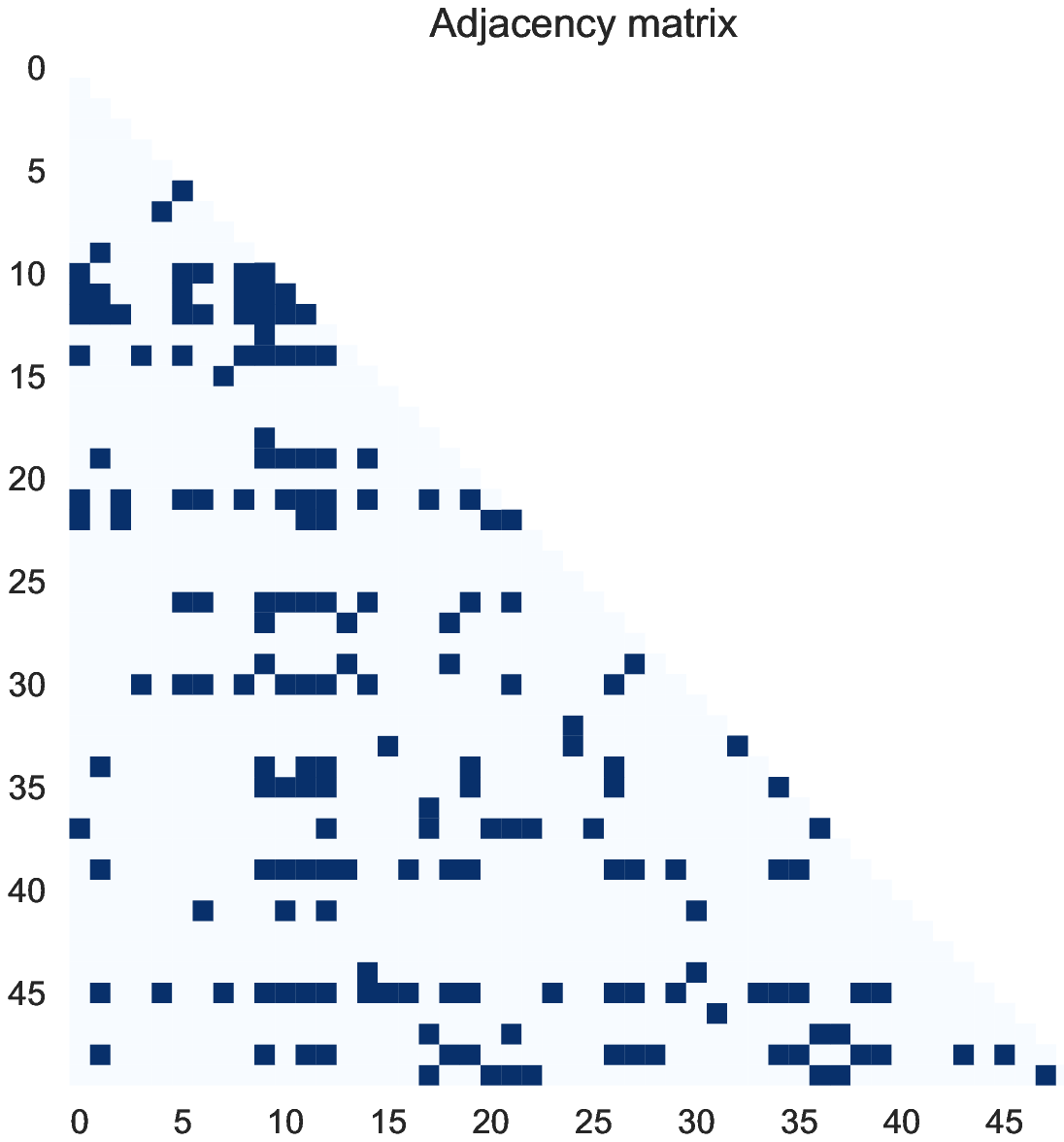}
    %\caption{}
    %\label{fig:3_classes_heatmap1}
   \end{subfigure}

    \begin{subfigure}[t]{0.33\textwidth}
        \includegraphics[width=\textwidth]{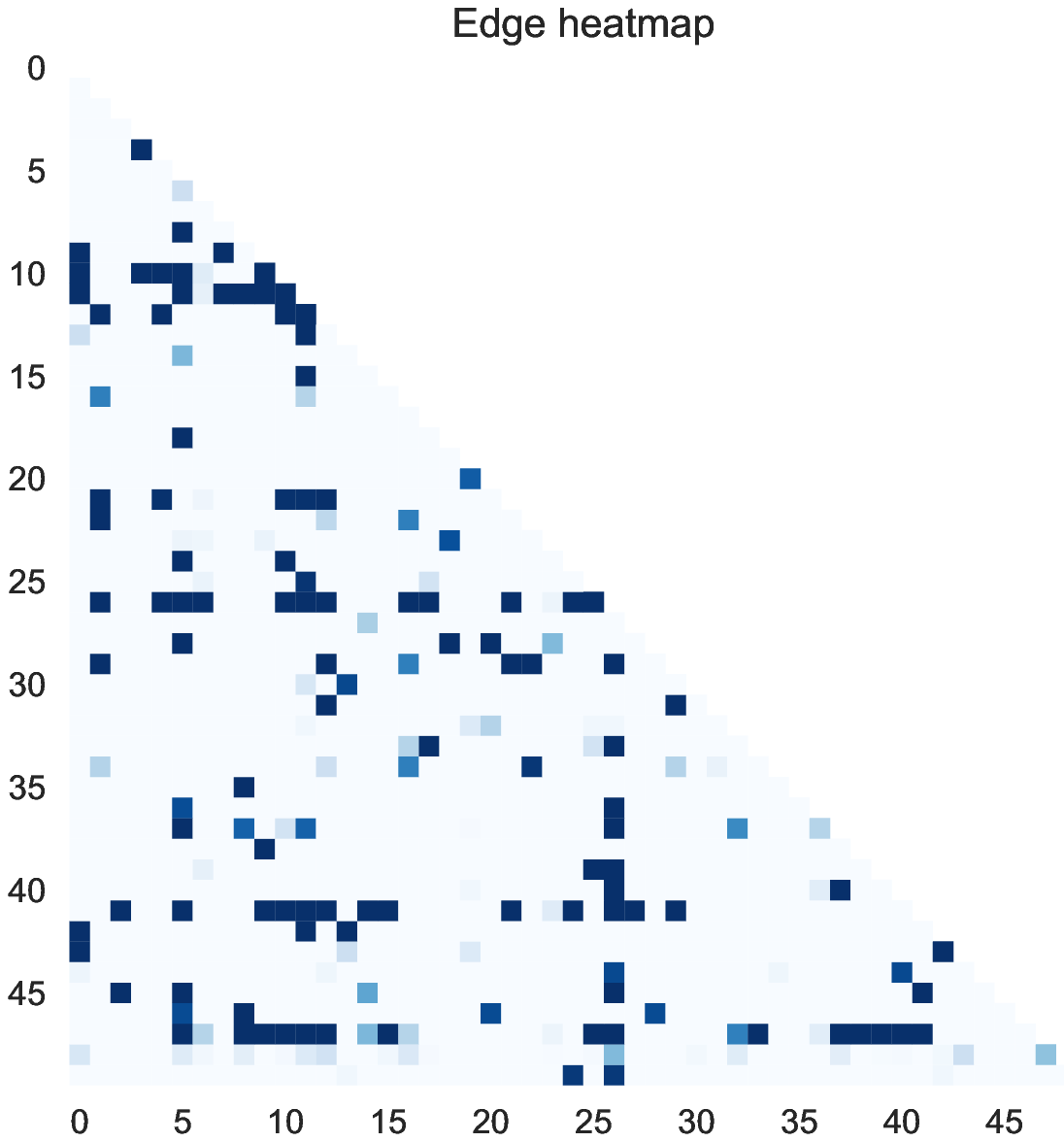}
        %\caption{}
        %\label{fig:two_classes_same_graph_adjmat}
    \end{subfigure}
   ~
   \begin{subfigure}[t]{0.33\textwidth}
       \includegraphics[width=\textwidth]{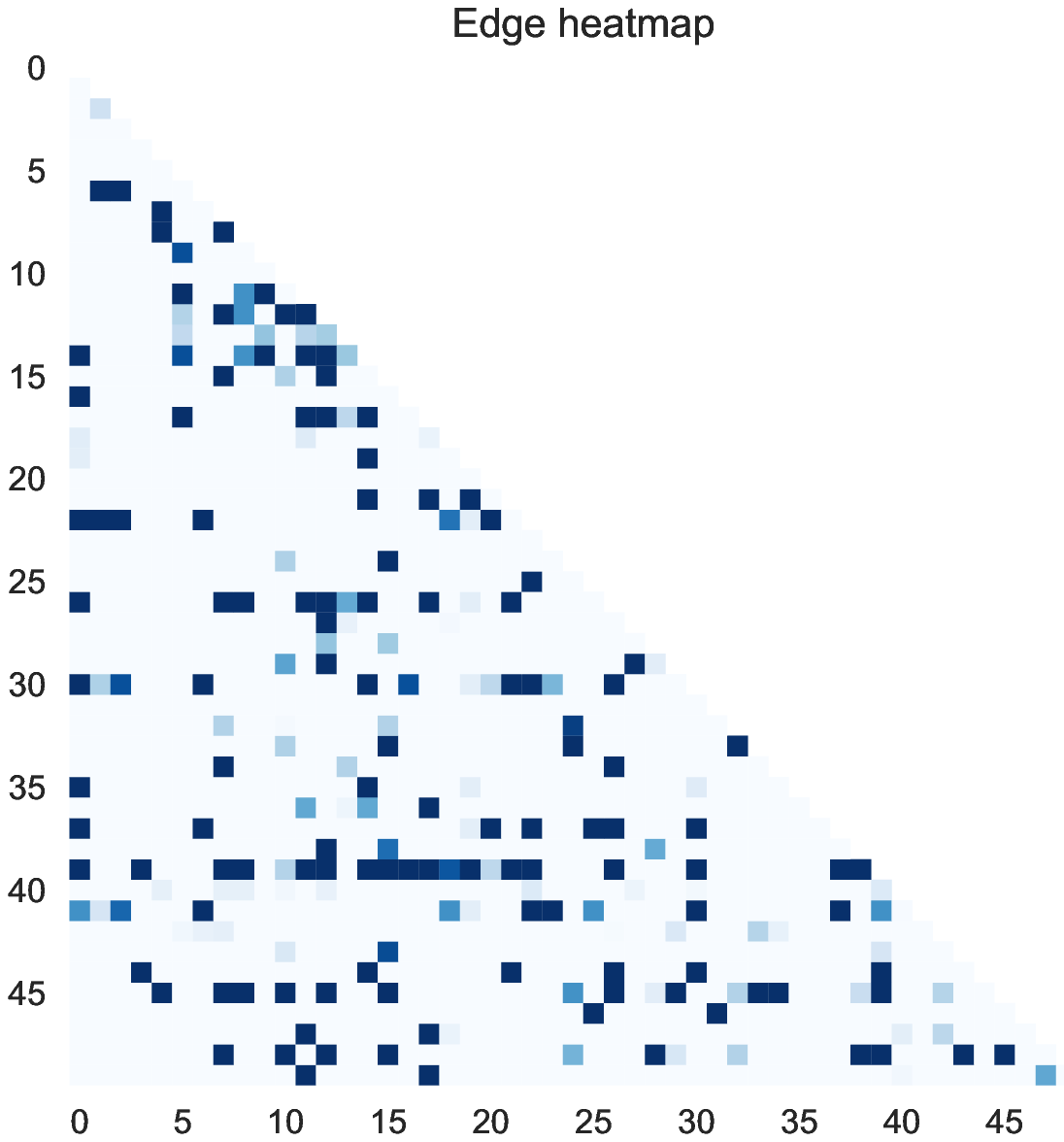}
    %\caption{}
    %\label{fig:3_classes_heatmap1}
   \end{subfigure}

    \begin{subfigure}[t]{0.33\textwidth}
        \includegraphics[width=\textwidth]{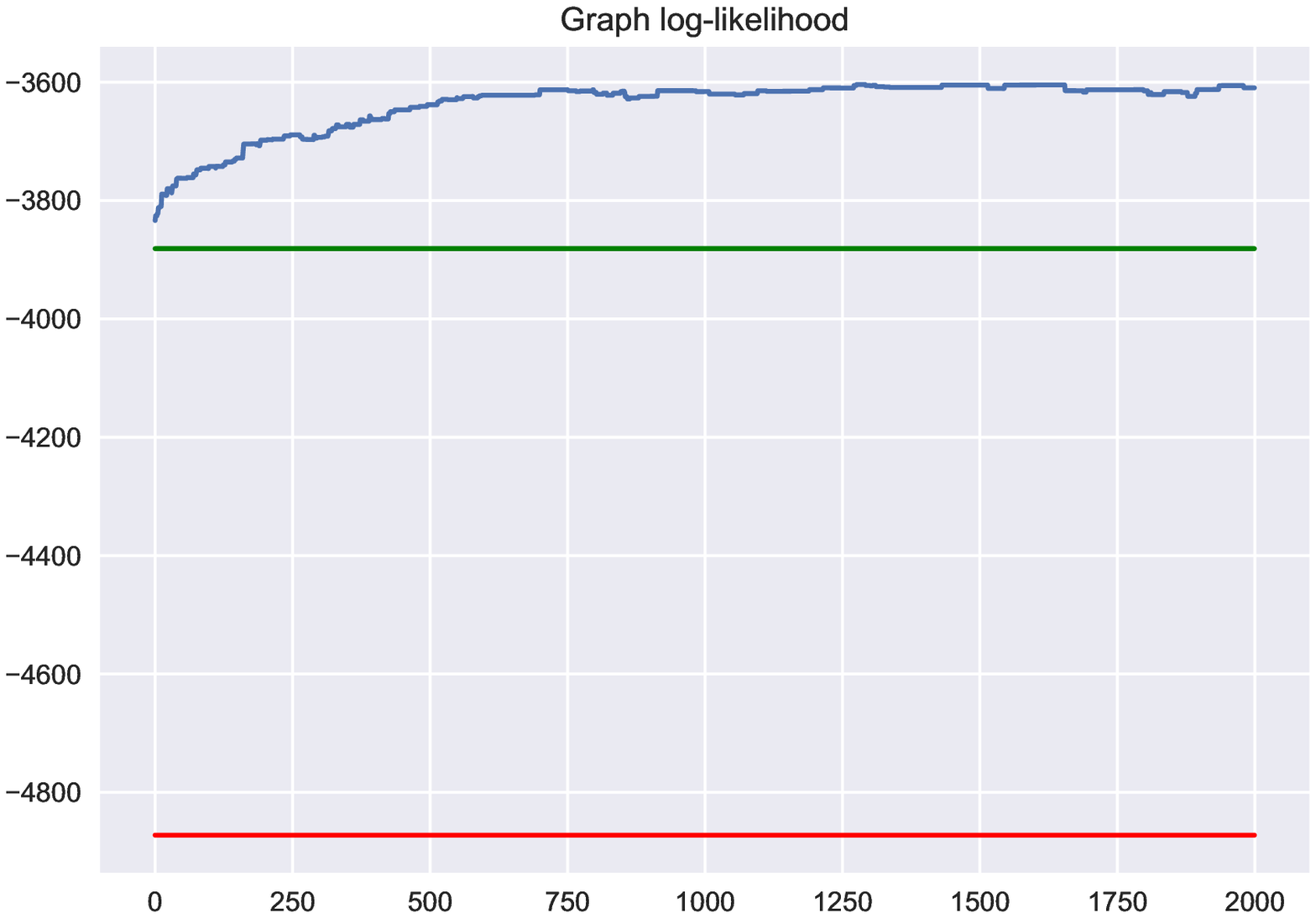}
        %\caption{}
        %\label{fig:two_classes_same_graph_adjmat}
    \end{subfigure}
   ~
   \begin{subfigure}[t]{0.33\textwidth}
       \includegraphics[width=\textwidth]{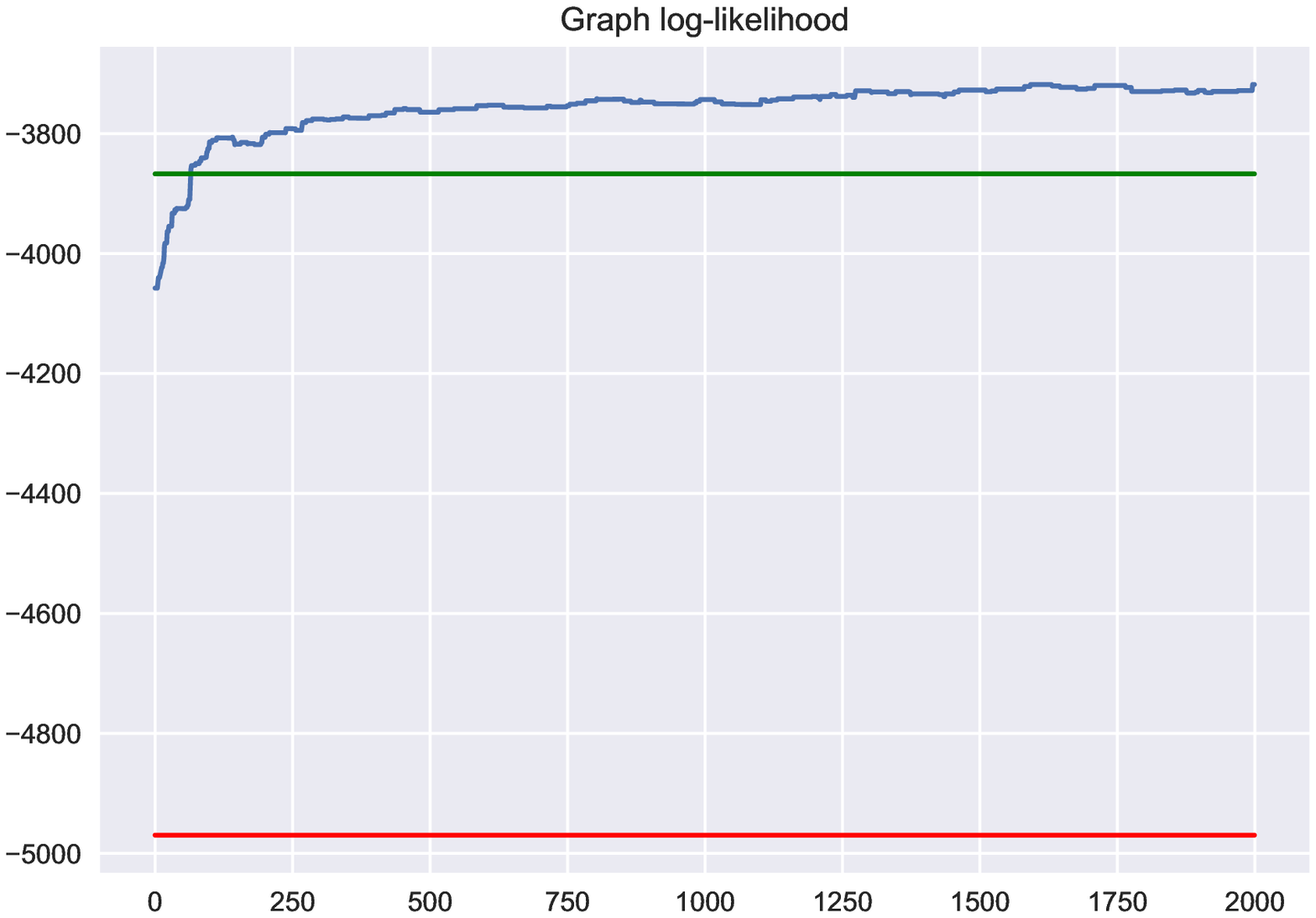}
      %\caption{}
    \label{fig:3_classes_heatmap1}
   \end{subfigure}

    \begin{subfigure}[t]{0.33\textwidth}
        \includegraphics[width=\textwidth]{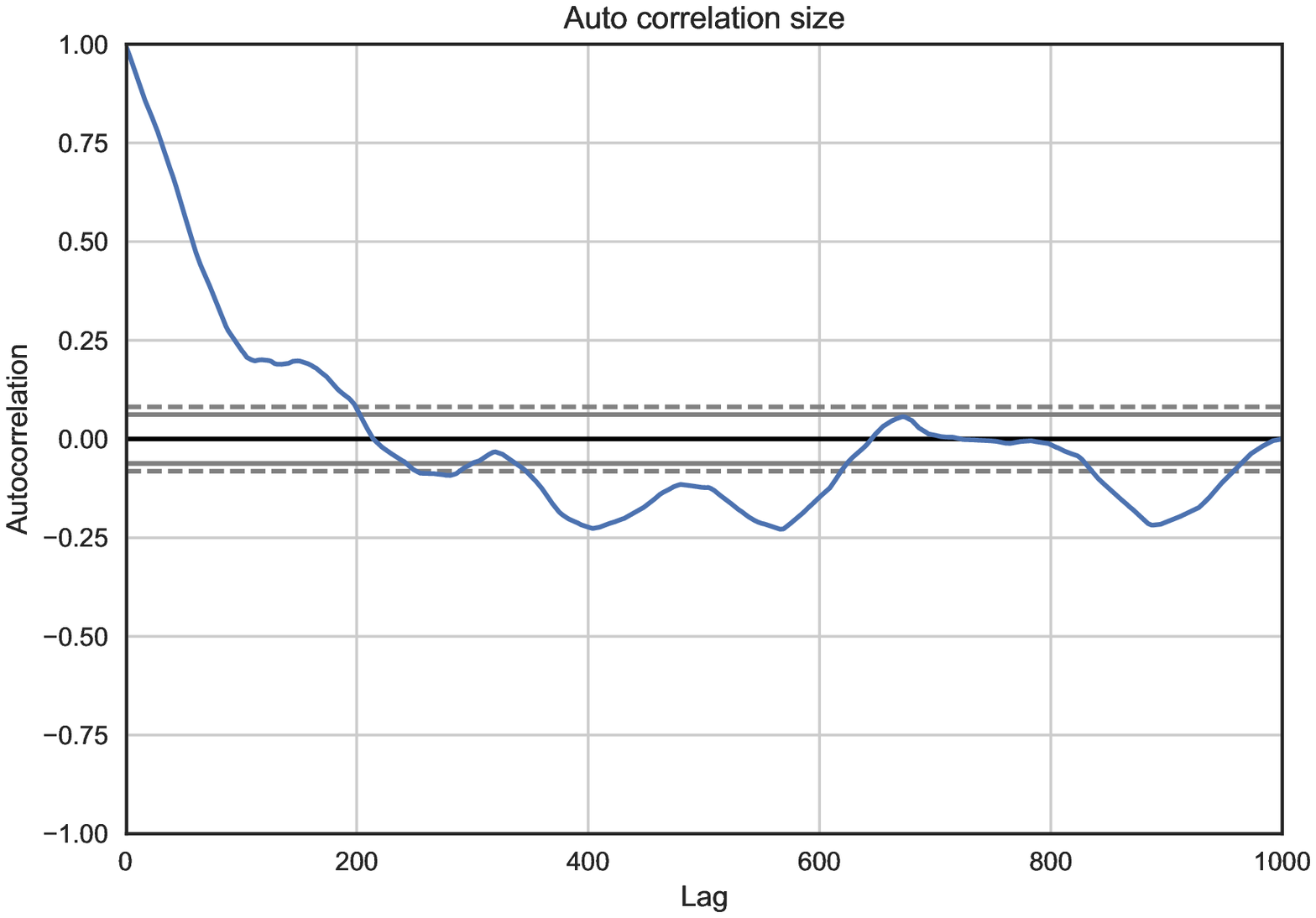}
        %\caption{}
        \label{fig:two_classes_same_graph_adjmat}
    \end{subfigure}
   ~
   \begin{subfigure}[t]{0.33\textwidth}
       \includegraphics[width=\textwidth]{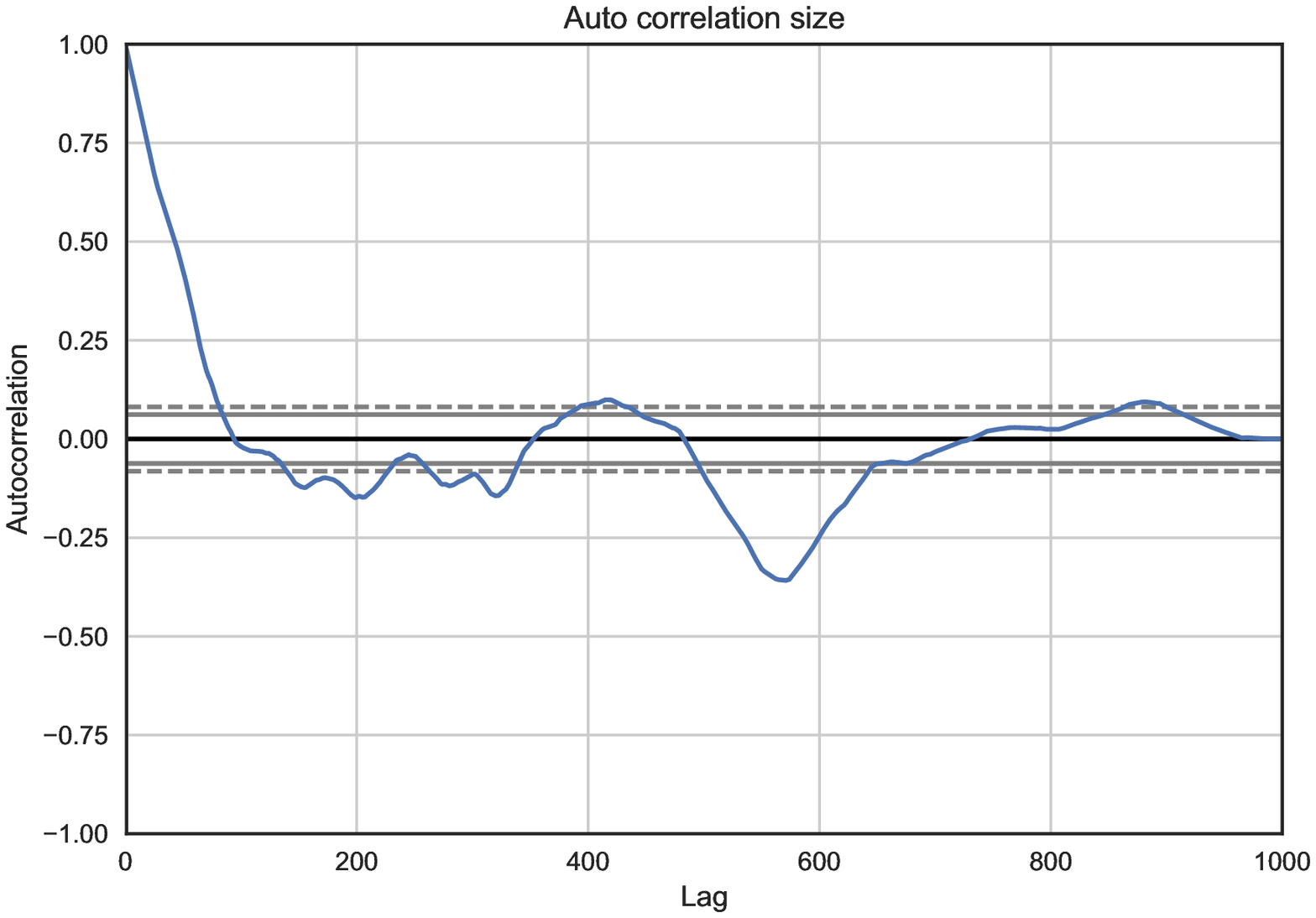}
      %\caption{}
    \label{fig:3_classes_heatmap1}
   \end{subfigure}

    \caption{Dataset A.}%: summary of one of the 10 trajectories generated by the PG sampler.}
    \label{fig:B}
\end{figure}

\begin{figure}
    \centering
    \begin{subfigure}[t]{0.33\textwidth}
        \includegraphics[width=\textwidth]{figures/2classes_p50_diff_graphs_adjmat_class_0.eps}
        %\caption{}
        %\label{fig:two_classes_same_graph_adjmat}
    \end{subfigure}
   ~
   \begin{subfigure}[t]{0.33\textwidth}
       \includegraphics[width=\textwidth]{{figures/2classes_p50_diff_graphs_adjmat_class_1}}
    %\caption{}
    %\label{fig:3_classes_heatmap1}
   \end{subfigure}

    \begin{subfigure}[t]{0.33\textwidth}
        \includegraphics[width=\textwidth]{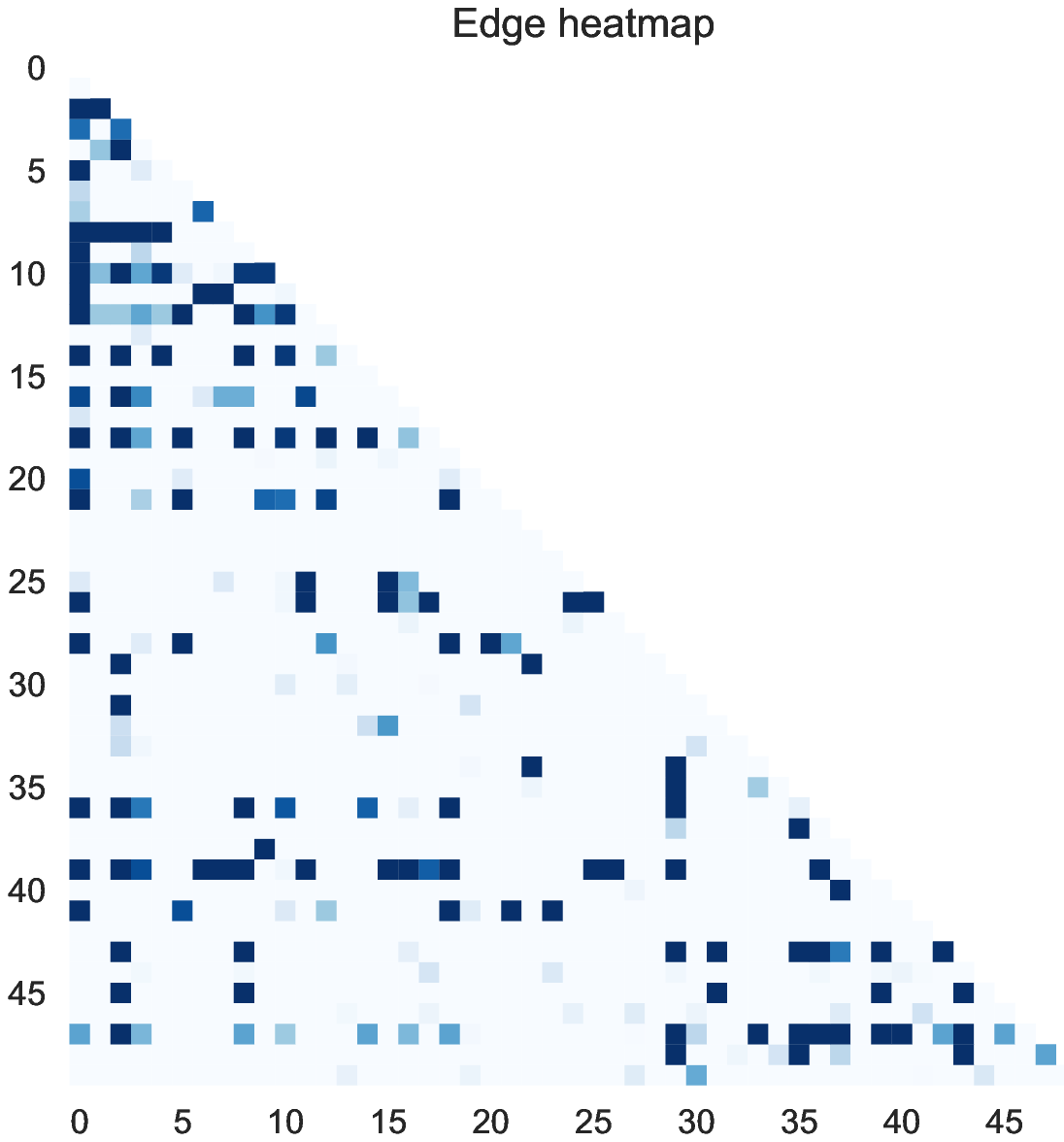}
        %\caption{}
        %\label{fig:two_classes_same_graph_adjmat}
    \end{subfigure}
   ~
   \begin{subfigure}[t]{0.33\textwidth}
       \includegraphics[width=\textwidth]{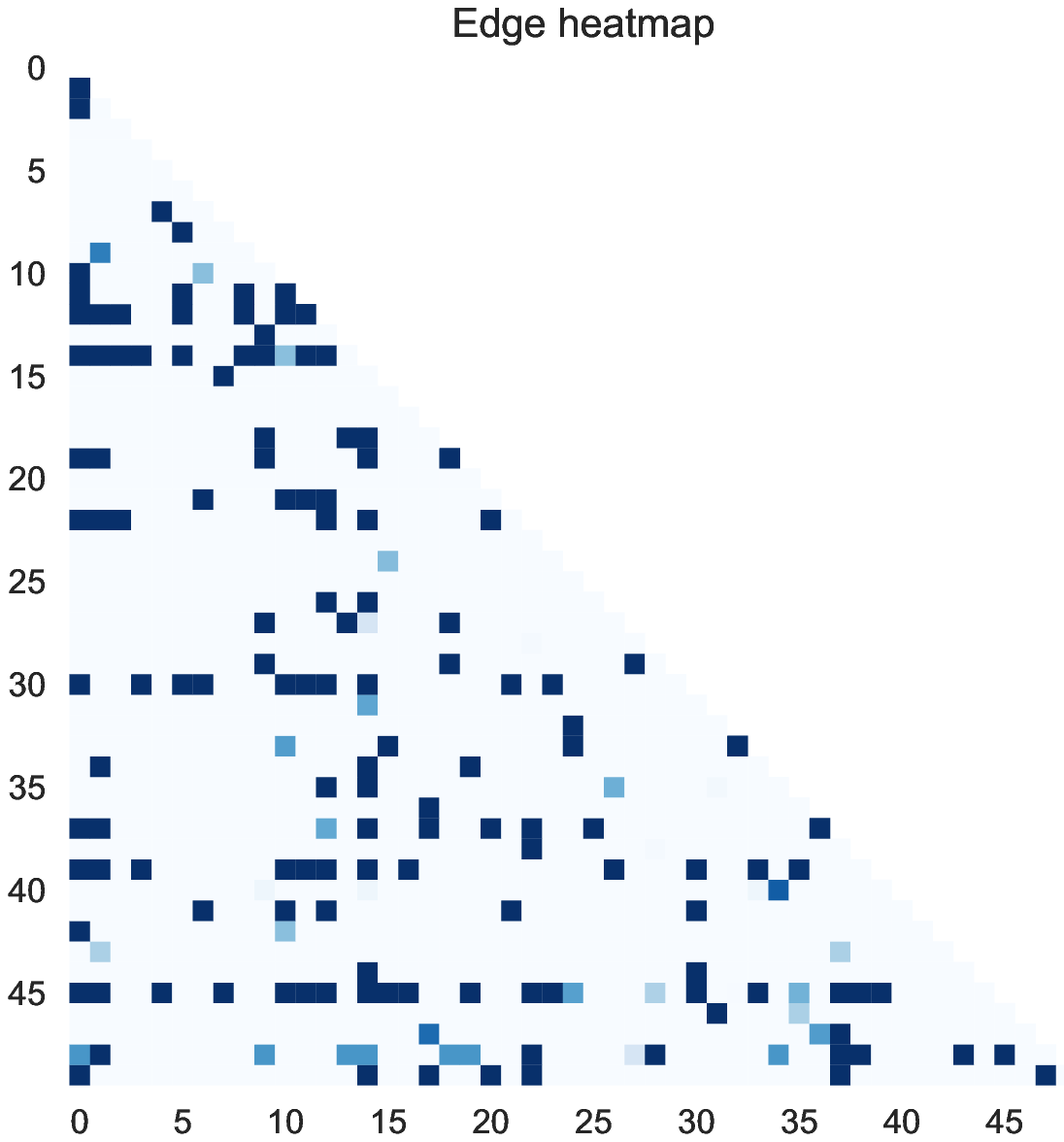}
    %\caption{}
    %\label{fig:3_classes_heatmap1}
   \end{subfigure}

    \begin{subfigure}[t]{0.33\textwidth}
        \includegraphics[width=\textwidth]{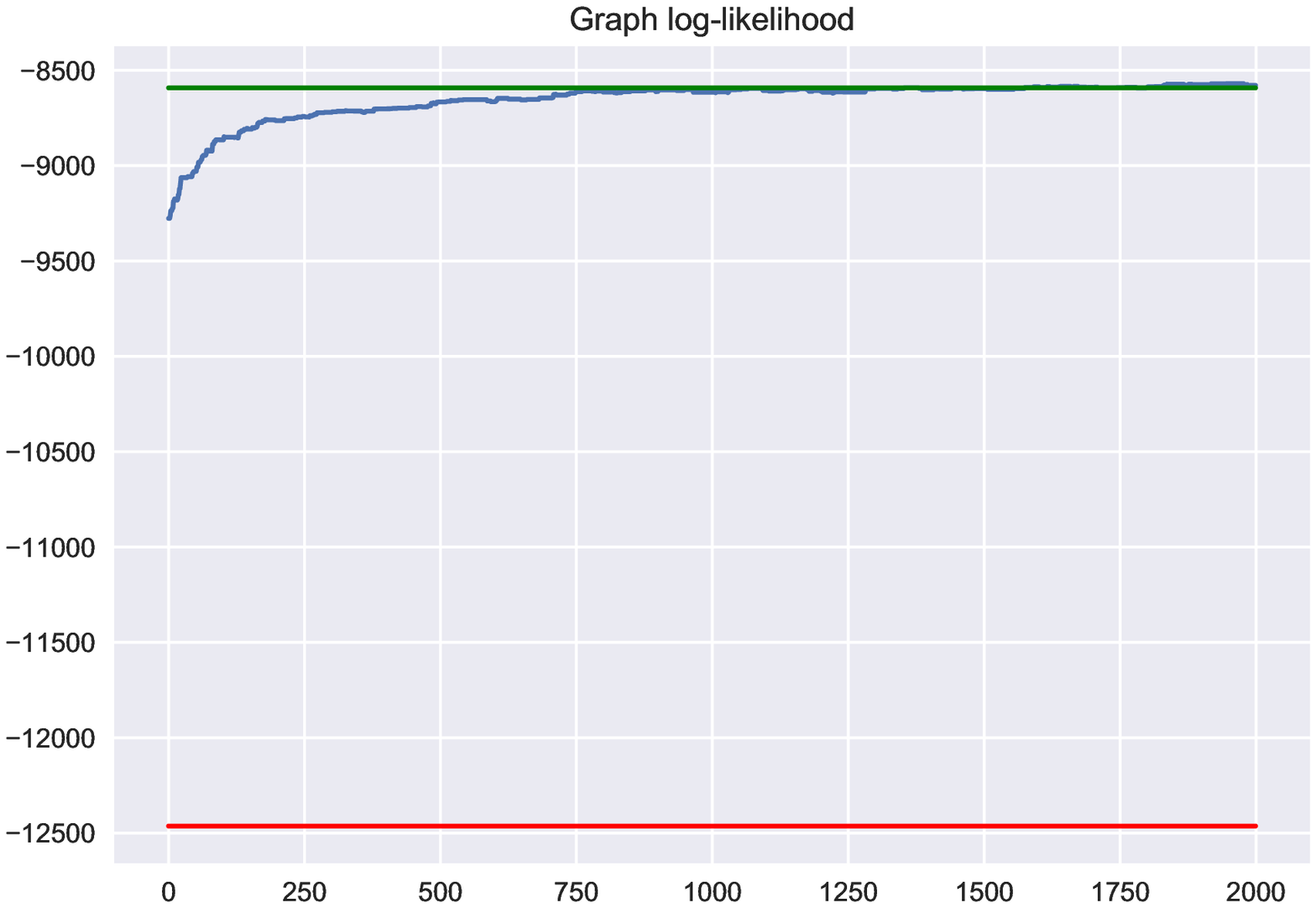}
        %\caption{}
        %\label{fig:two_classes_same_graph_adjmat}
    \end{subfigure}
   ~
   \begin{subfigure}[t]{0.33\textwidth}
       \includegraphics[width=\textwidth]{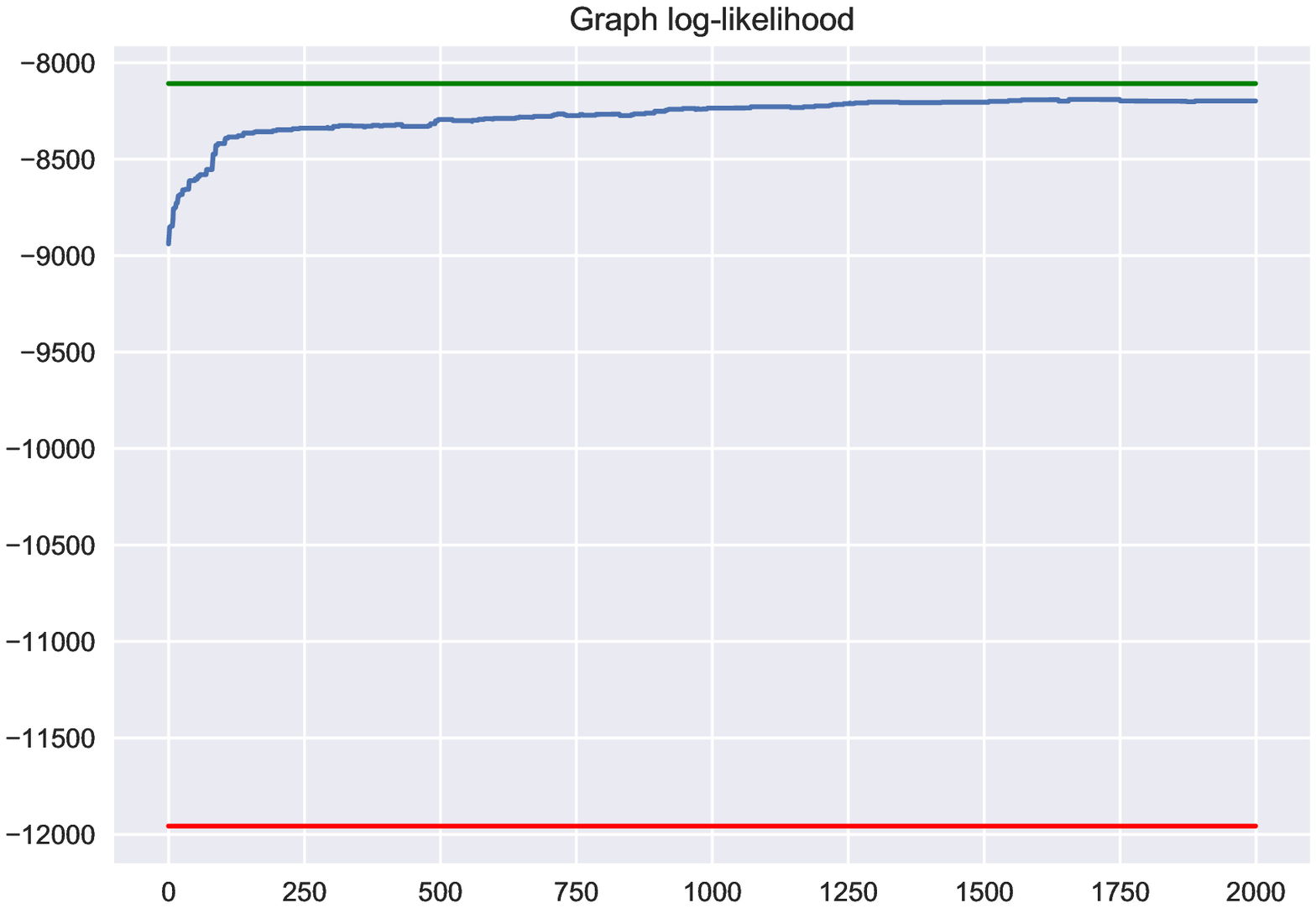}
      %\caption{}
    \label{fig:3_classes_heatmap1}
   \end{subfigure}

    \begin{subfigure}[t]{0.33\textwidth}
        \includegraphics[width=\textwidth]{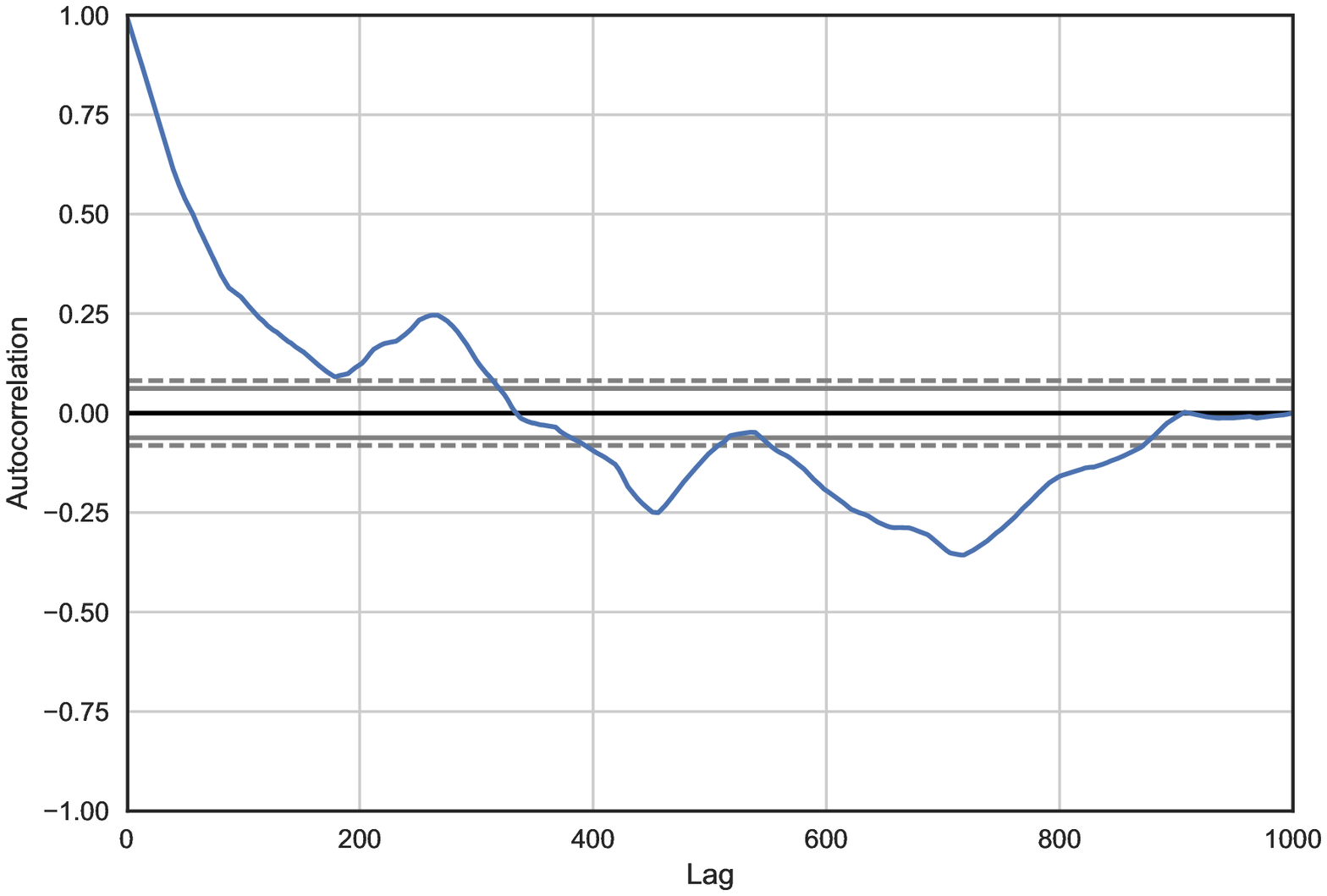}
        %\caption{}
        \label{fig:two_classes_same_graph_adjmat}
    \end{subfigure}
   ~
   \begin{subfigure}[t]{0.33\textwidth}
       \includegraphics[width=\textwidth]{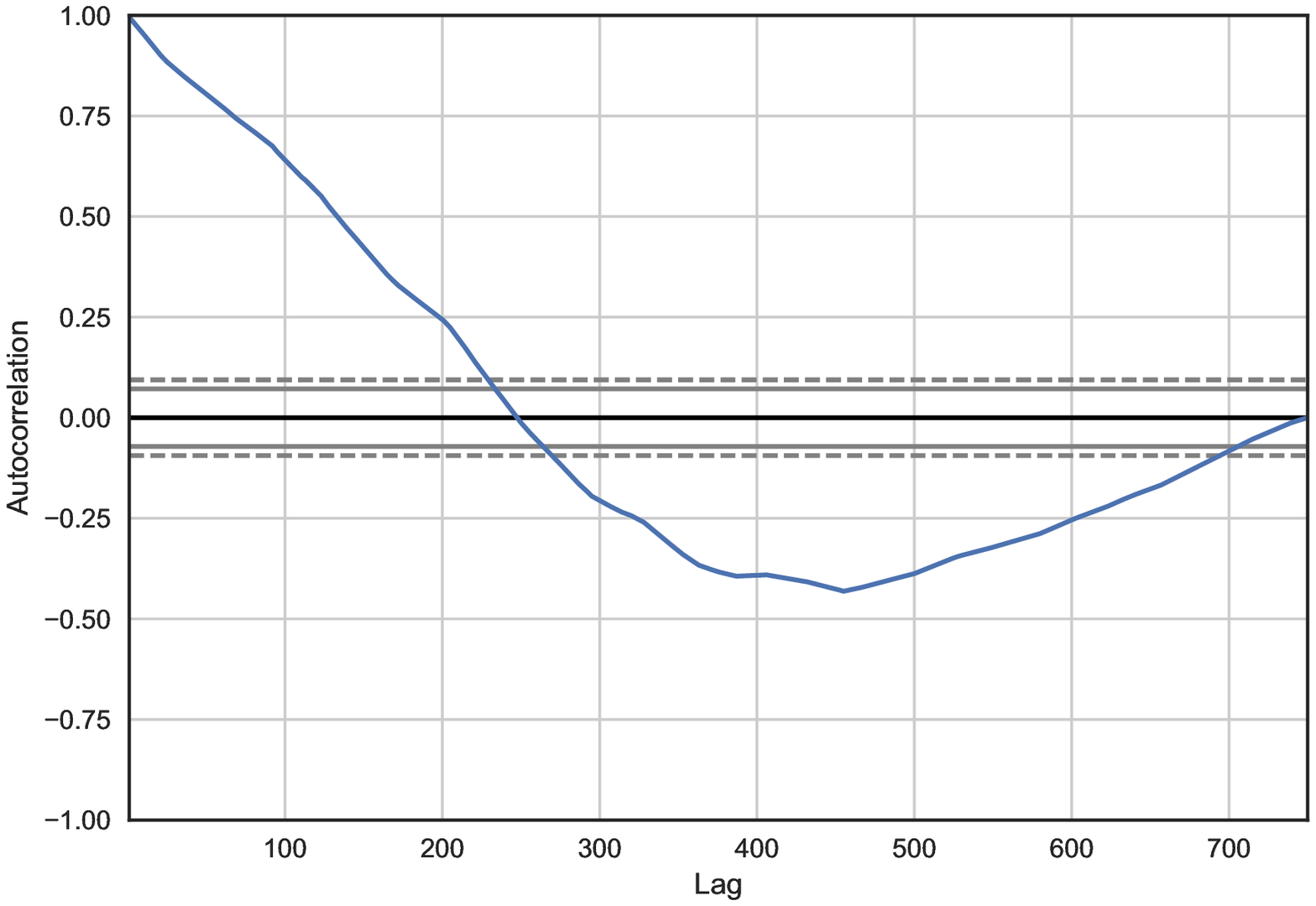}
      %\caption{}
    \label{fig:3_classes_heatmap1}
   \end{subfigure}

    \caption{Dataset A, $\testdatasize=300$.}%: summary of one of the 10 trajectories generated by the PG sampler.}
    \label{fig:B300}
\end{figure}

\begin{figure}
    \centering
    \begin{subfigure}[t]{0.33\textwidth}
        \includegraphics[width=\textwidth]{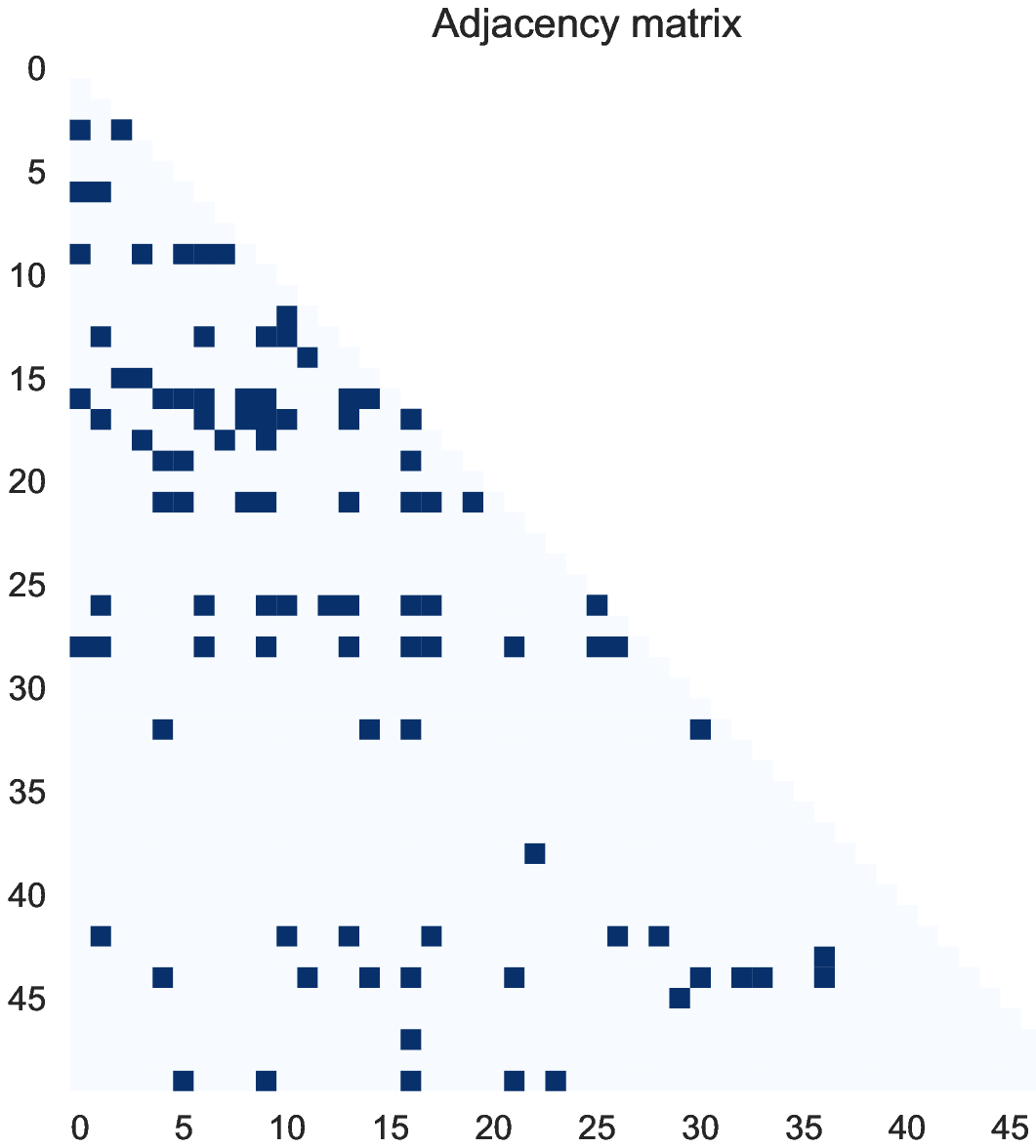}
        %\caption{}
        %\label{fig:two_classes_same_graph_adjmat}
    \end{subfigure}
   ~
   \begin{subfigure}[t]{0.33\textwidth}
       \includegraphics[width=\textwidth]{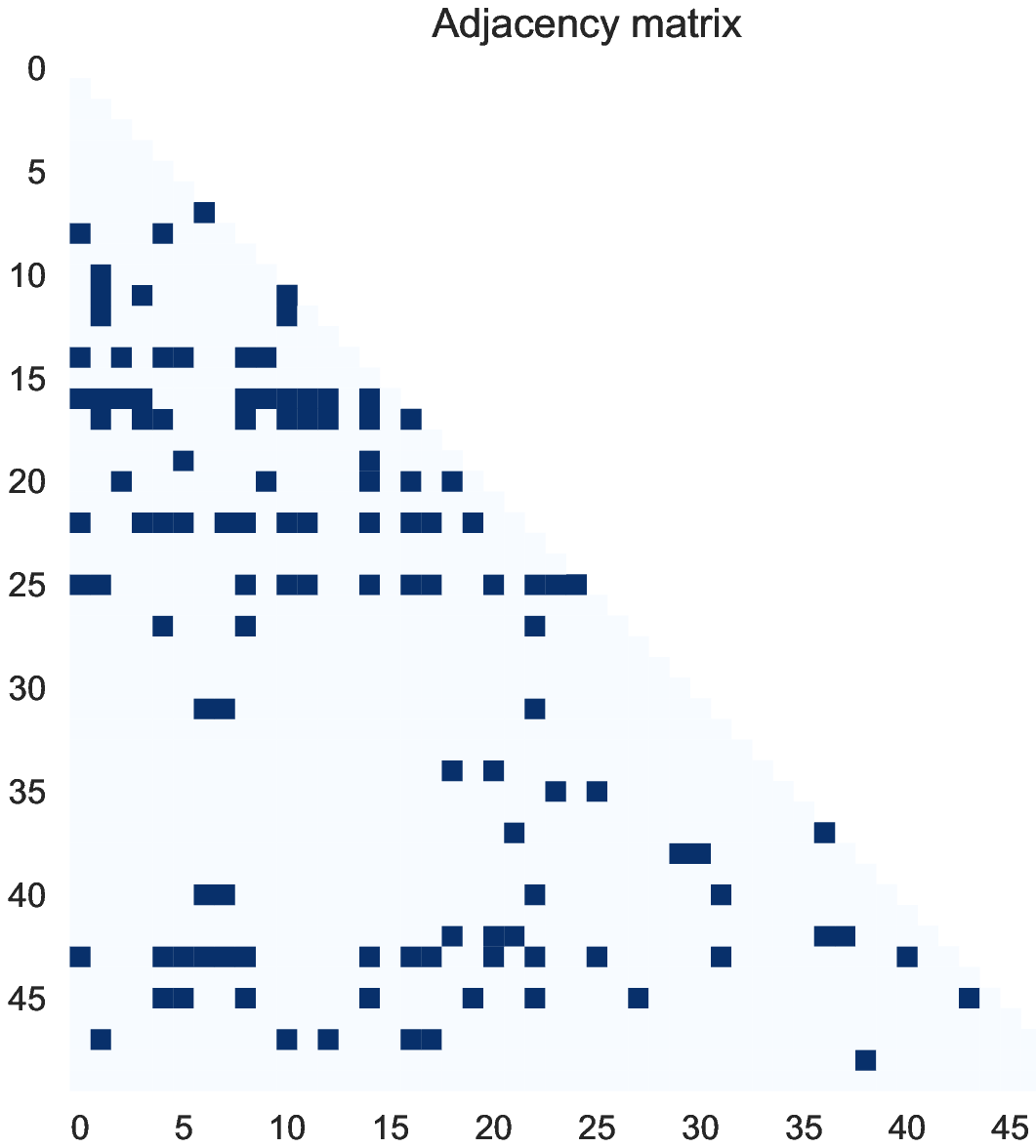}
    %\caption{}
    %\label{fig:3_classes_heatmap1}
   \end{subfigure}
    ~
   \begin{subfigure}[t]{0.33\textwidth}
       \includegraphics[width=\textwidth]{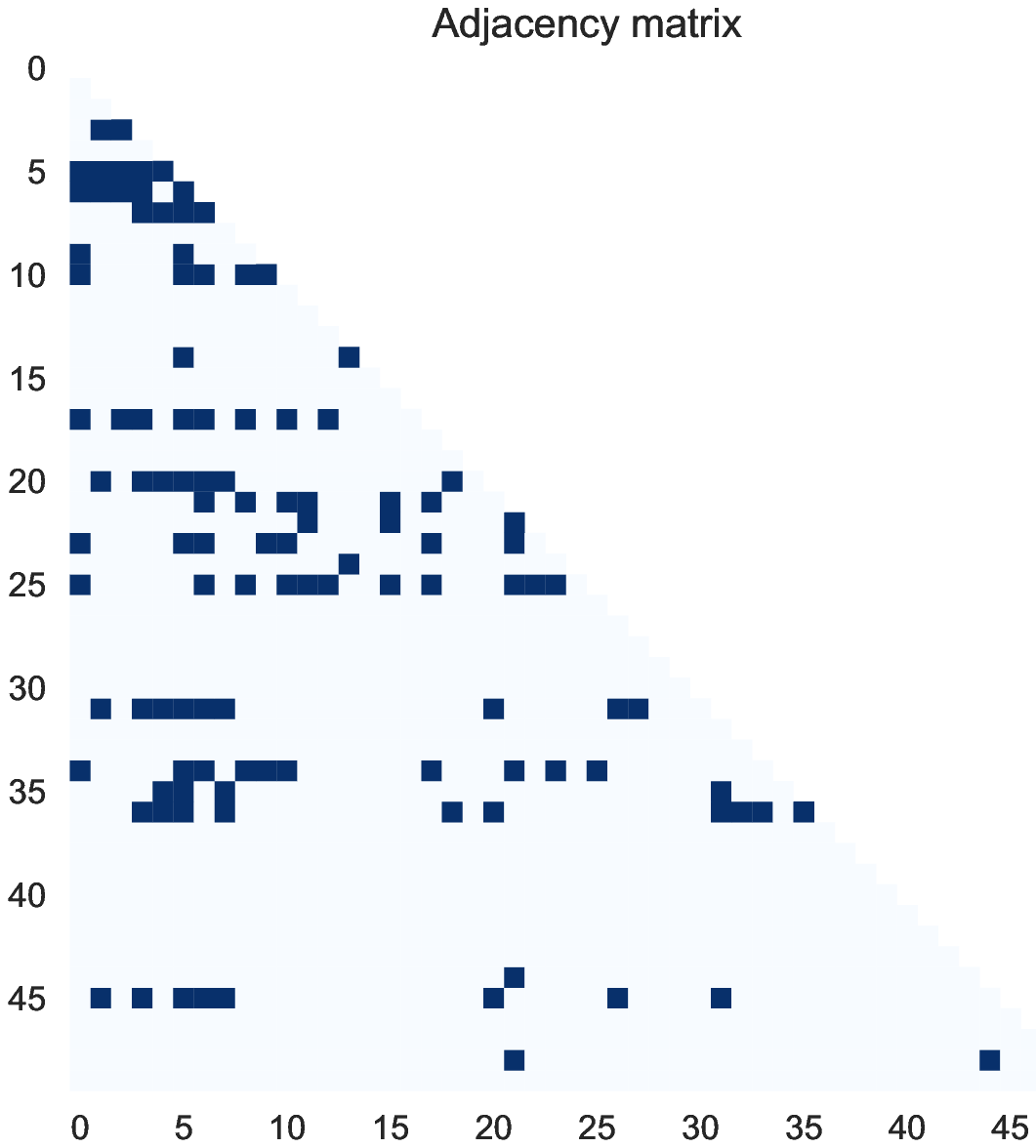}
    %\caption{}
    %\label{fig:two_classes_with_same_graph_heatmap}
   \end{subfigure}

    \begin{subfigure}[t]{0.33\textwidth}
        \includegraphics[width=\textwidth]{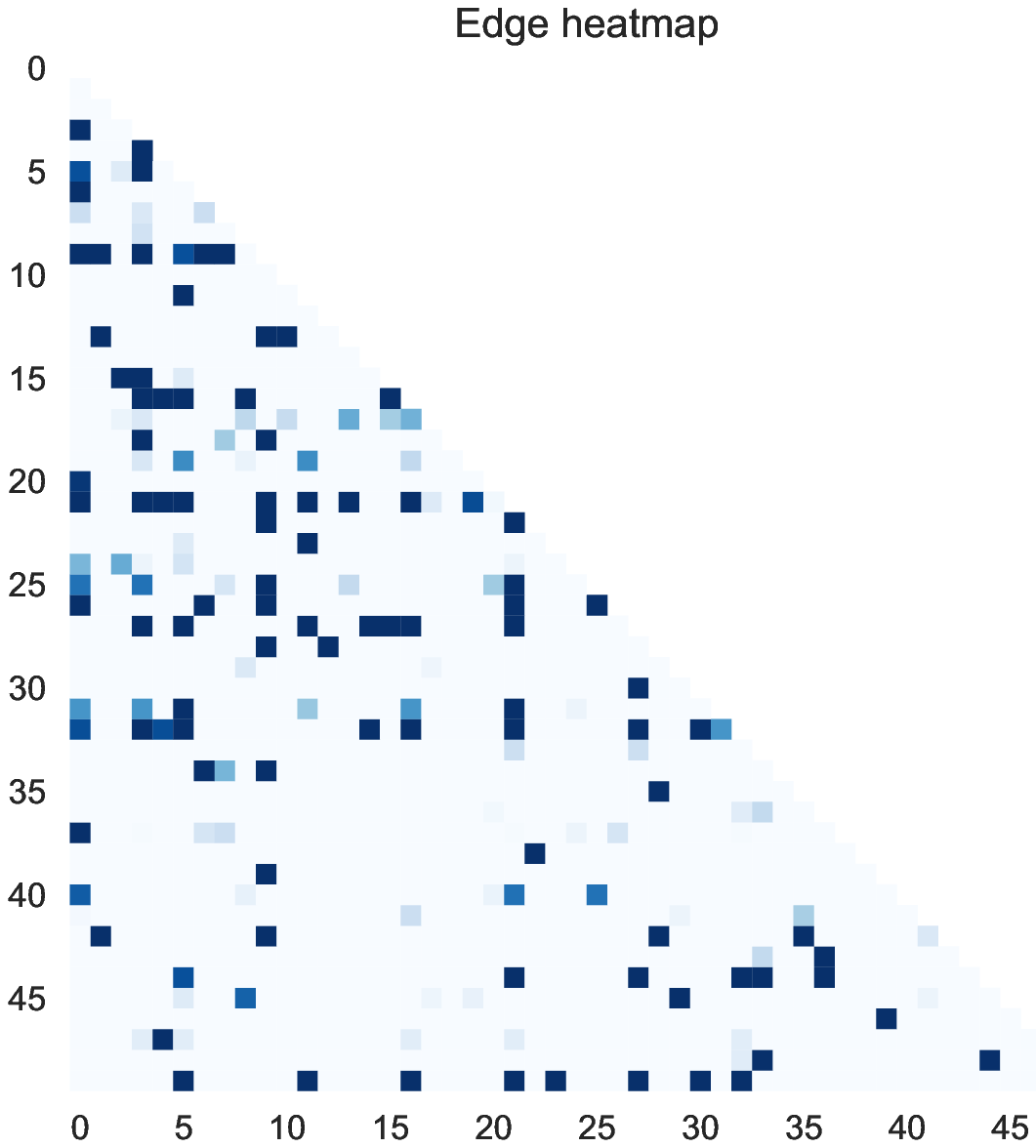}
        %\caption{}
        %\label{fig:two_classes_same_graph_adjmat}
    \end{subfigure}
   ~
   \begin{subfigure}[t]{0.33\textwidth}
       \includegraphics[width=\textwidth]{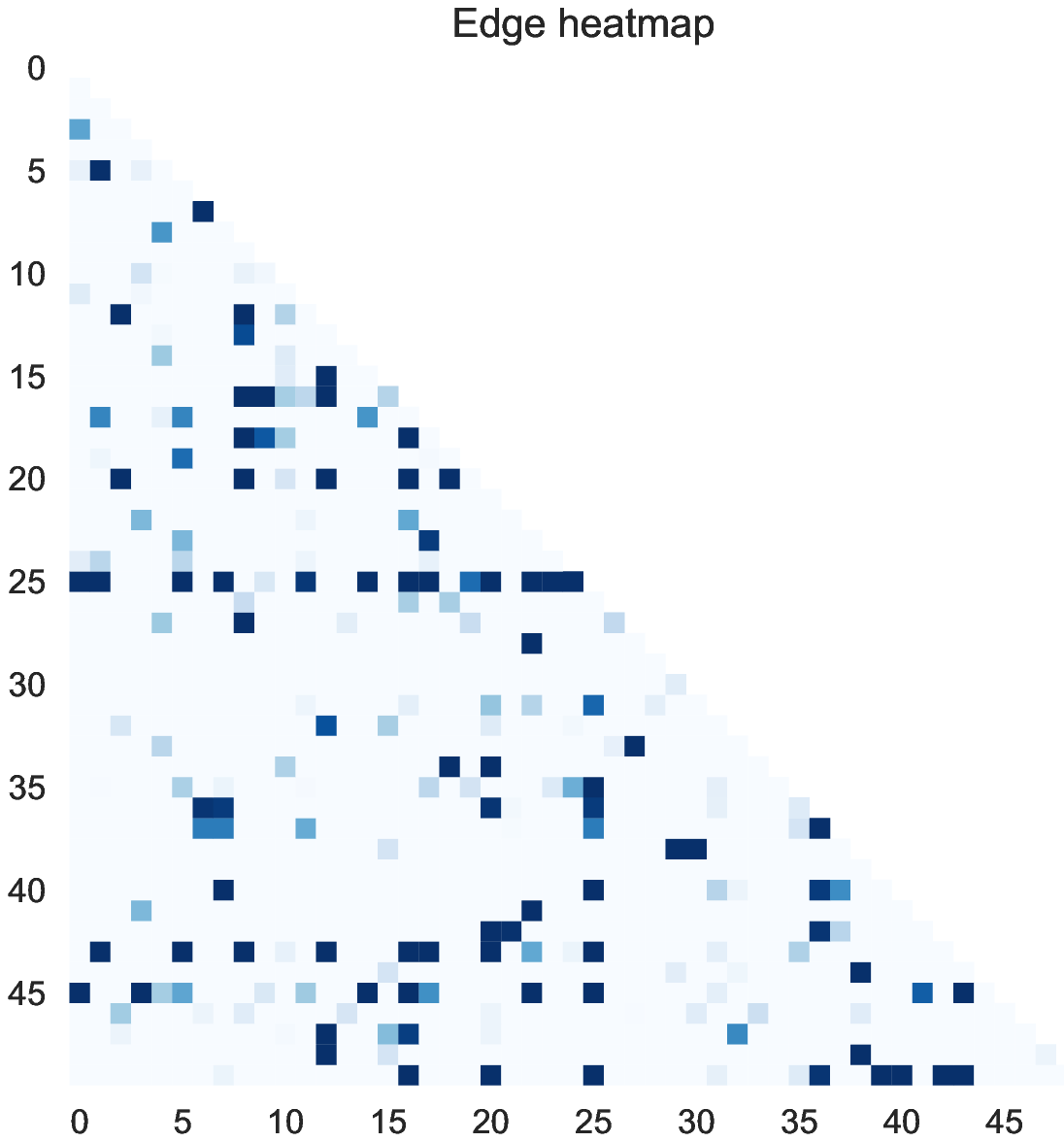}
    %\caption{}
    %\label{fig:3_classes_heatmap1}
   \end{subfigure}
    ~
   \begin{subfigure}[t]{0.33\textwidth}
       \includegraphics[width=\textwidth]{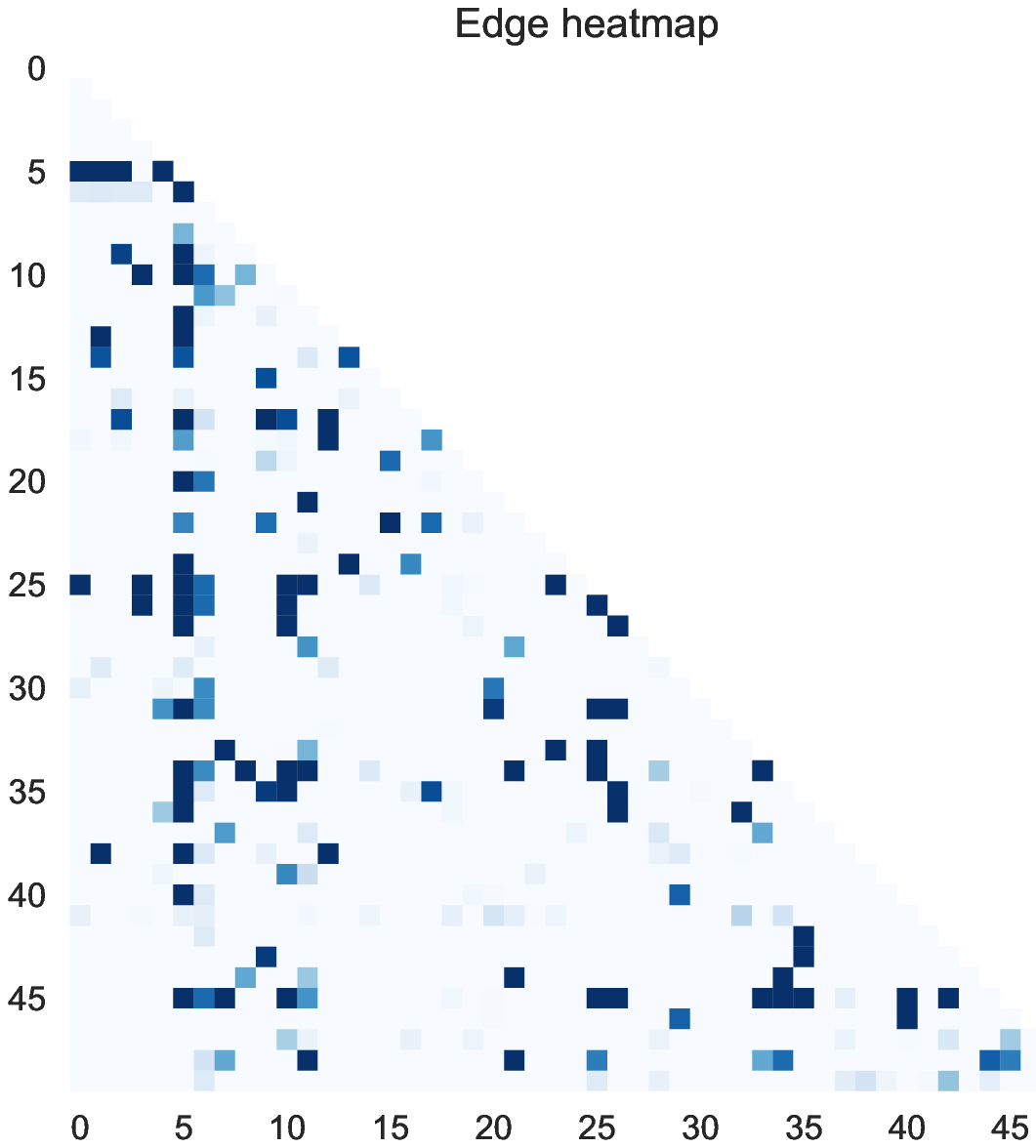}
    %\caption{}
    %\label{fig:two_classes_with_same_graph_heatmap}
   \end{subfigure}

    \begin{subfigure}[t]{0.33\textwidth}
        \includegraphics[width=\textwidth]{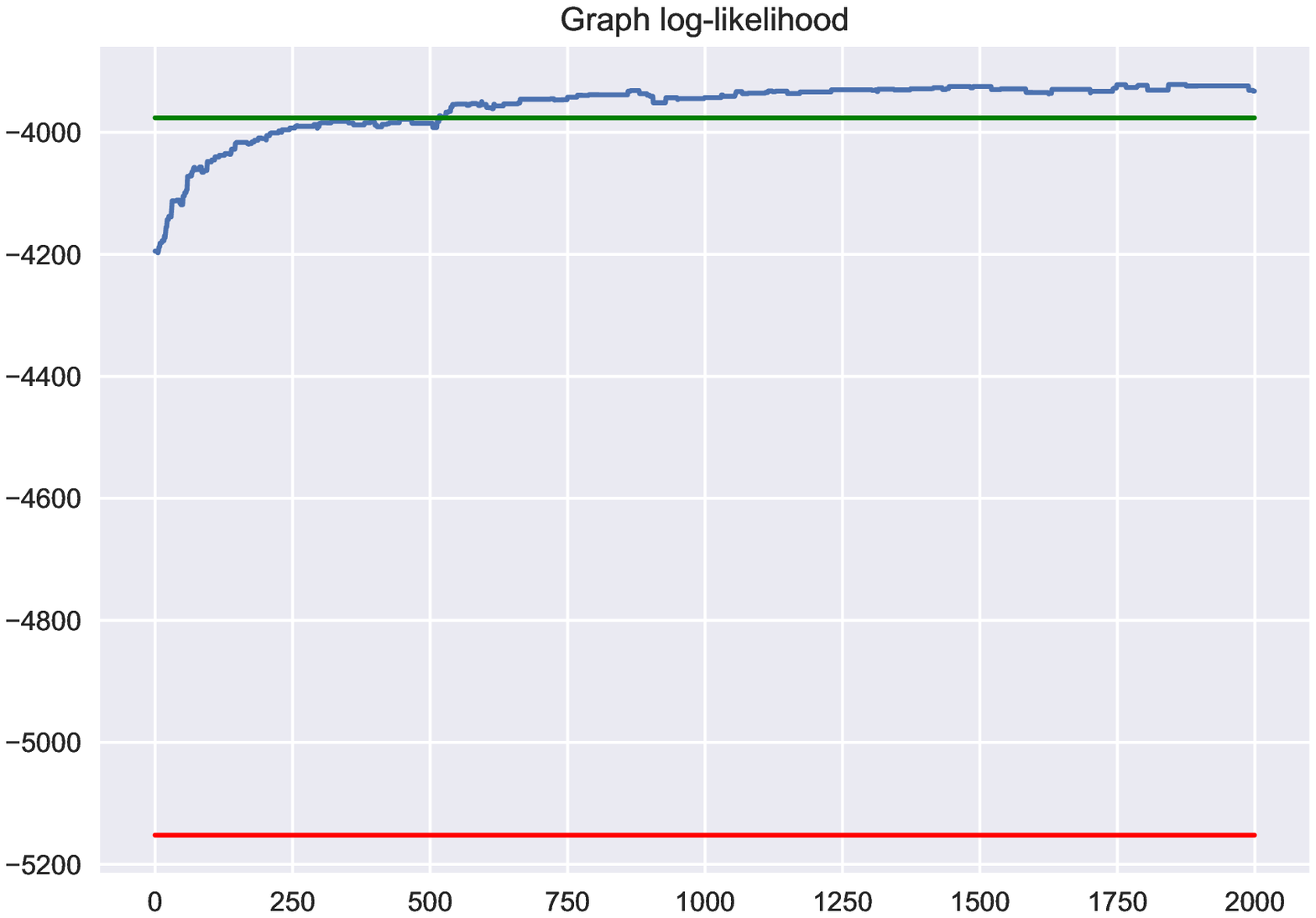}
        %\caption{}
        %\label{fig:two_classes_same_graph_adjmat}
    \end{subfigure}
   ~
   \begin{subfigure}[t]{0.33\textwidth}
       \includegraphics[width=\textwidth]{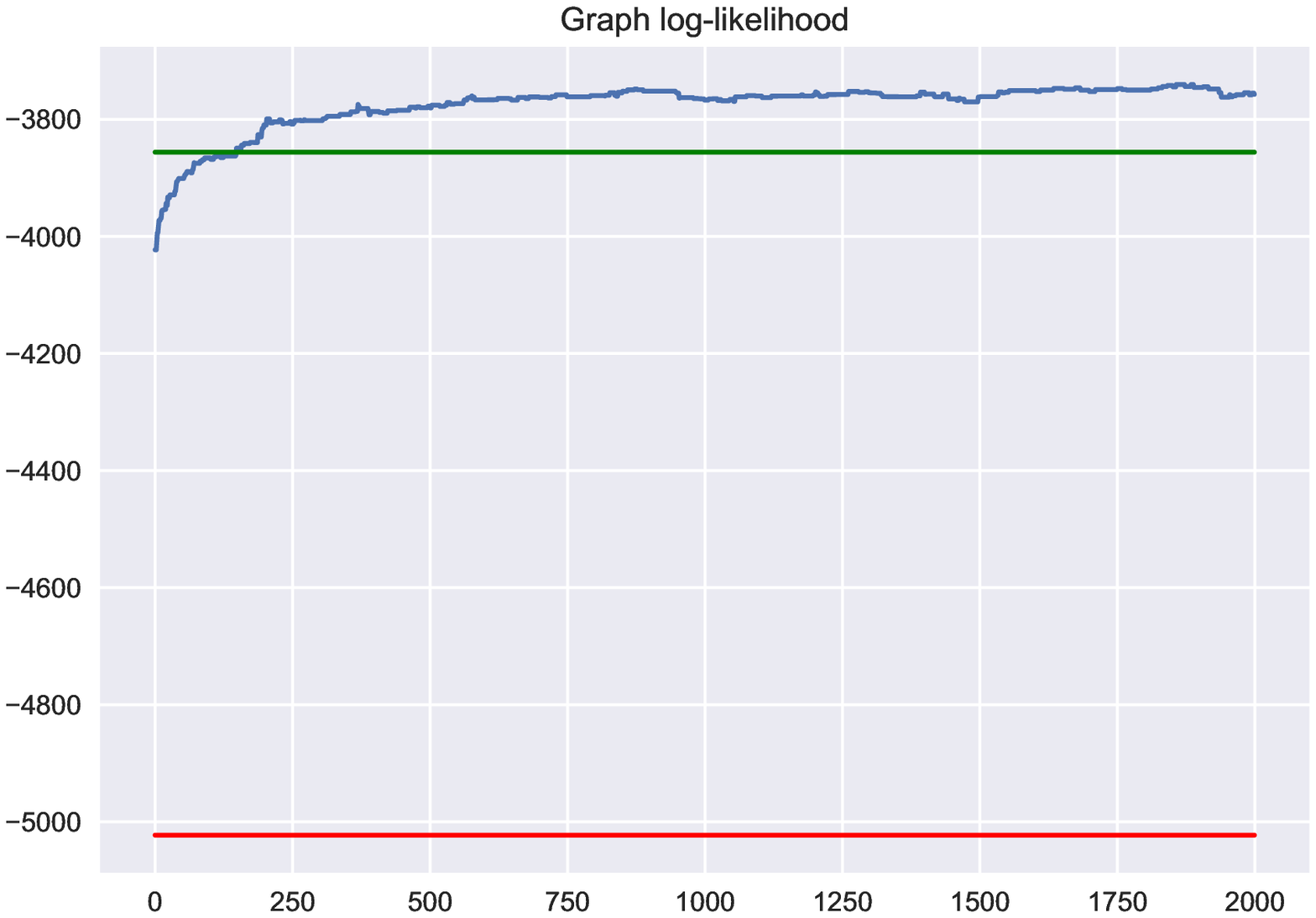}
    %\caption{}
    \label{fig:3_classes_heatmap1}
   \end{subfigure}
    ~
   \begin{subfigure}[t]{0.33\textwidth}
       \includegraphics[width=\textwidth]{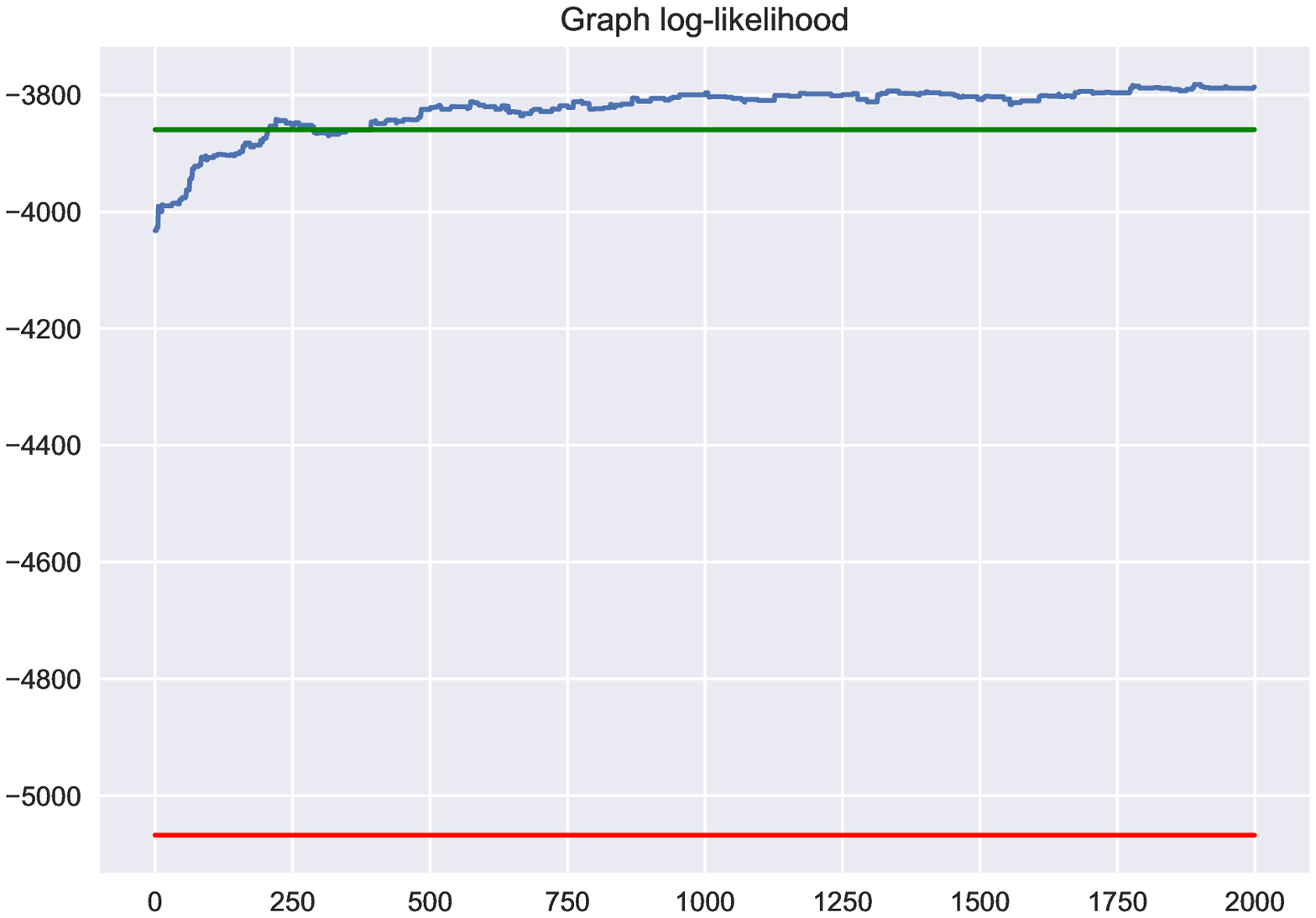}
    %\caption{}
    \label{fig:two_classes_with_same_graph_heatmap}
   \end{subfigure}

    \begin{subfigure}[t]{0.33\textwidth}
        \includegraphics[width=\textwidth]{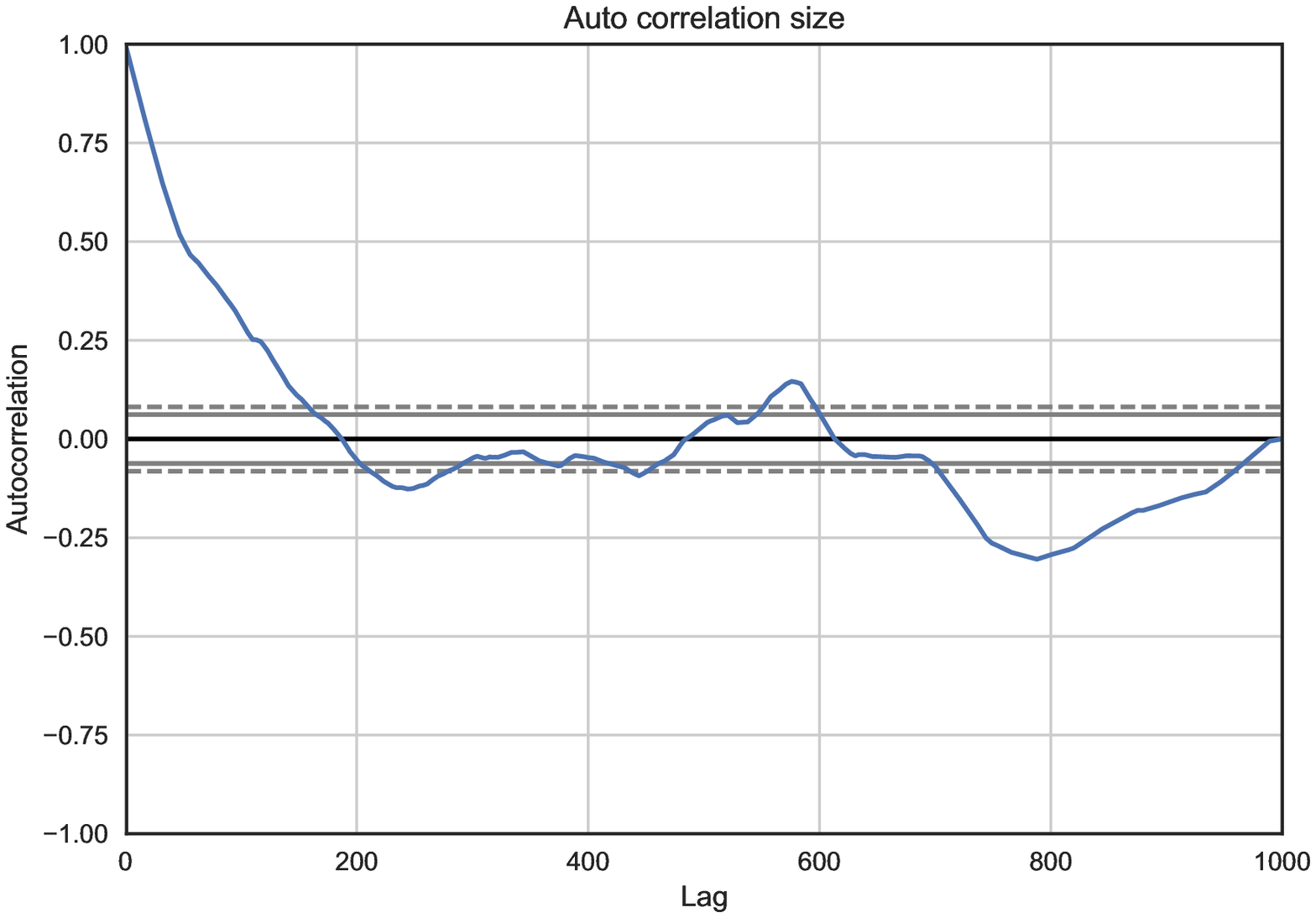}
        %\caption{}
        \label{fig:two_classes_same_graph_adjmat}
    \end{subfigure}
   ~
   \begin{subfigure}[t]{0.33\textwidth}
       \includegraphics[width=\textwidth]{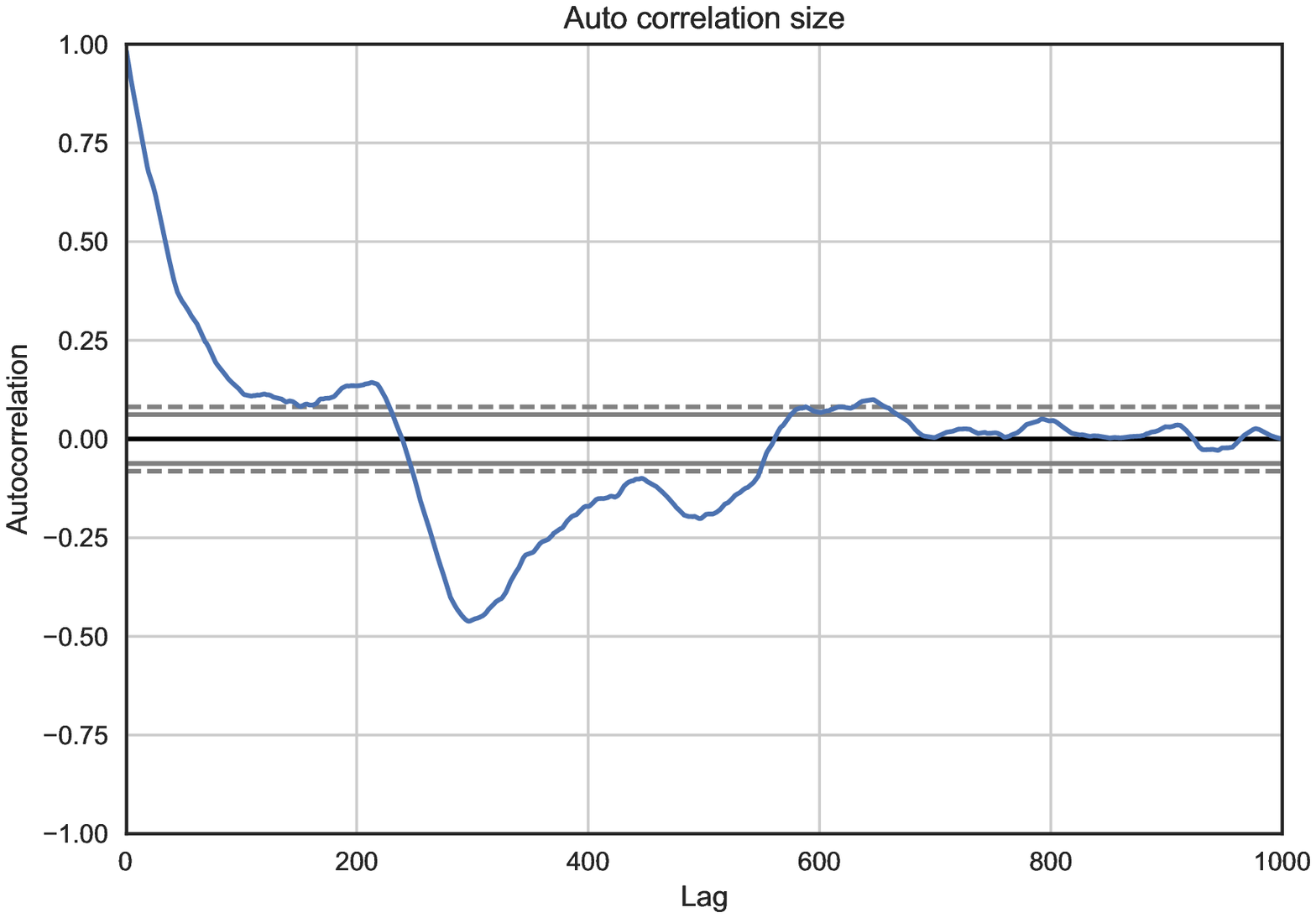}
    %\caption{}
    \label{fig:3_classes_heatmap1}
   \end{subfigure}
    ~
   \begin{subfigure}[t]{0.33\textwidth}
       \includegraphics[width=\textwidth]{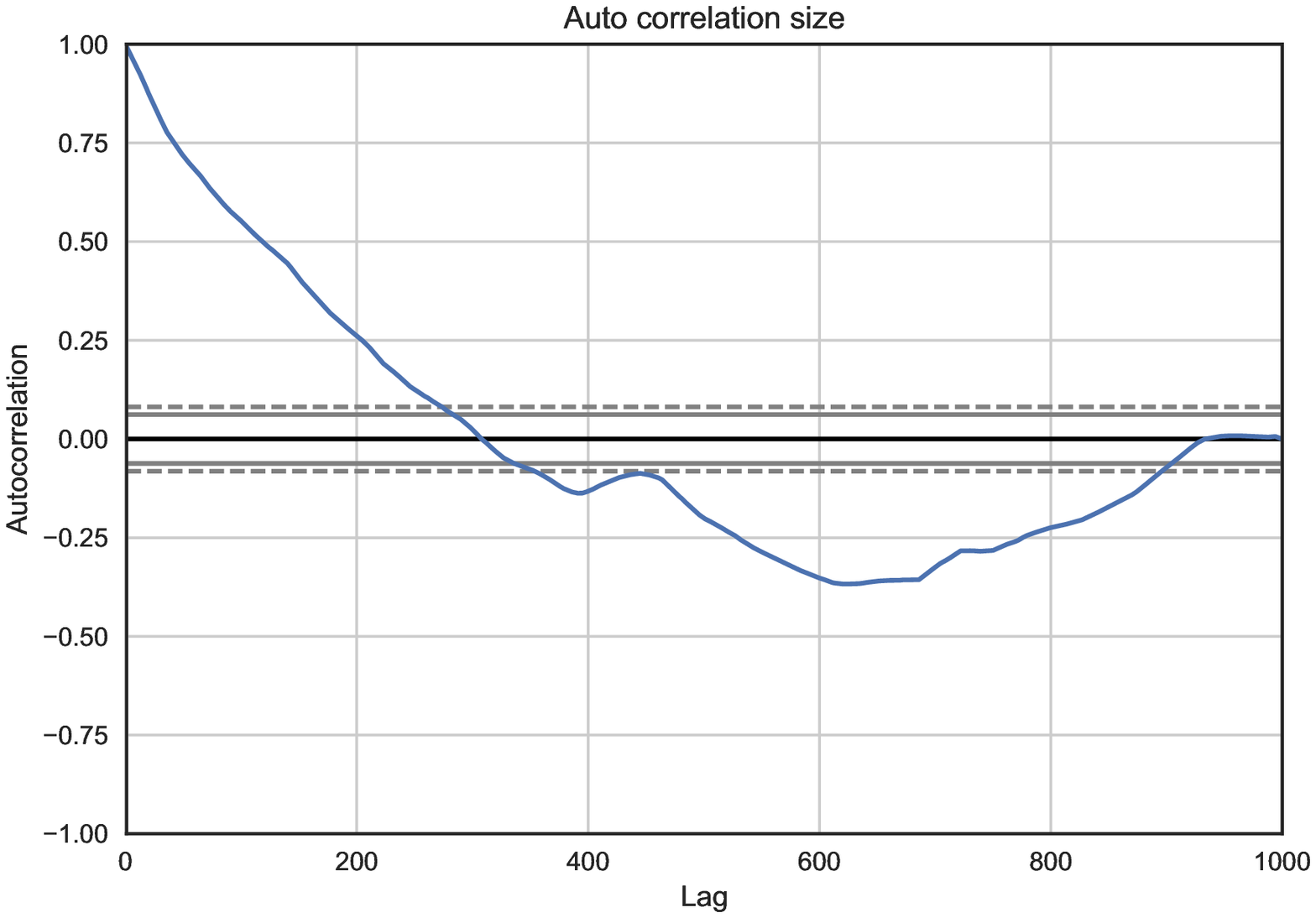}
    %\caption{}
    \label{fig:two_classes_with_same_graph_heatmap}
   \end{subfigure}
    \caption{Dataset B.}%: summary of one of the 10 PG trajectories for the three columns of classes read from top to bottom: true adjacency matrices; posterior edge heatmaps; graph log-likelihoods for PG trajectory (blue), true graph (green) and complete graph (red); auto-correlation for graph size.}
    \label{fig:C}
\end{figure}

\begin{figure}
    \centering
    \begin{subfigure}[b]{0.33\textwidth}
        \includegraphics[width=\textwidth]{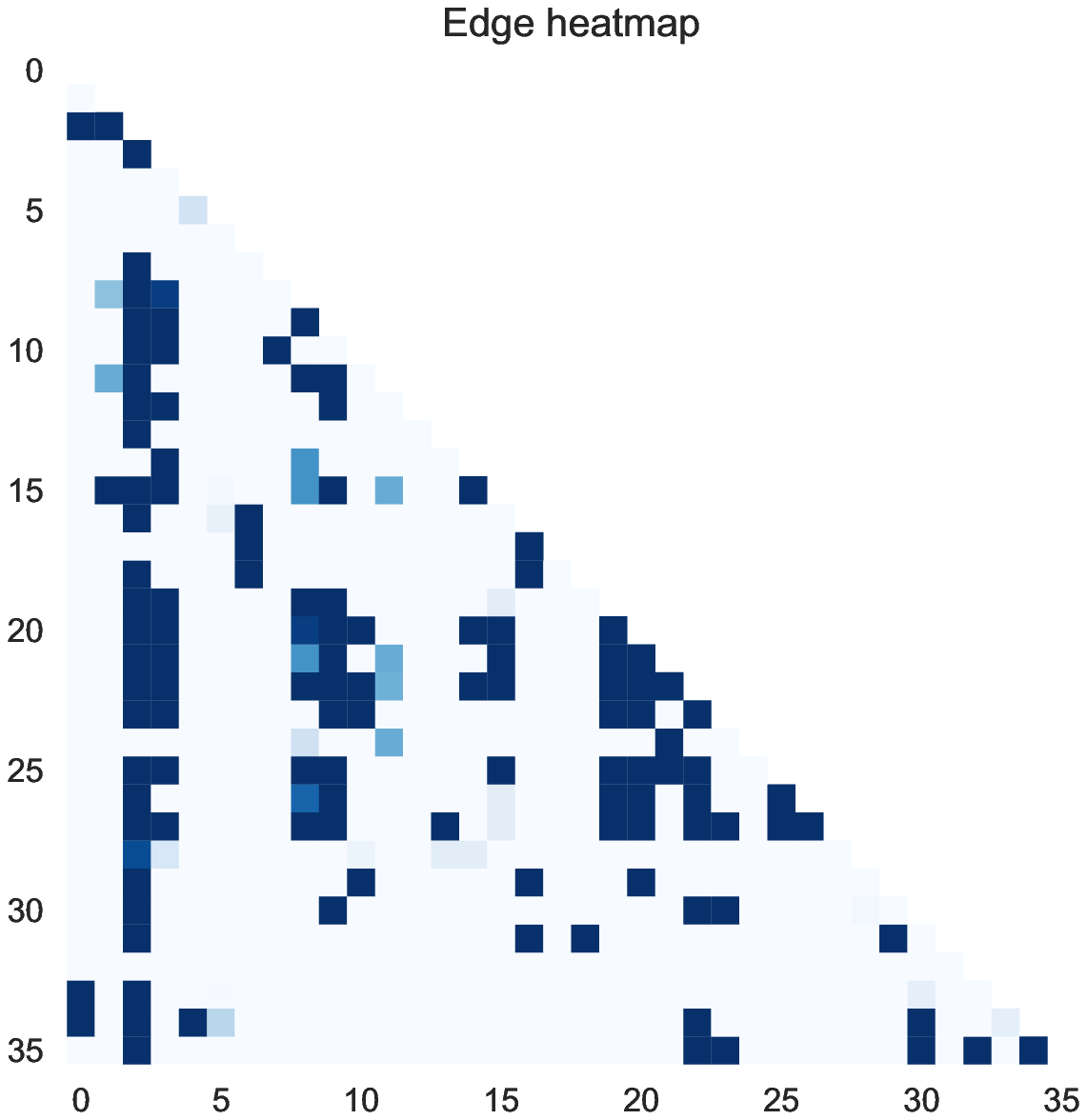}
    \end{subfigure}
    ~
    \begin{subfigure}[b]{0.33\textwidth}
        \includegraphics[width=\textwidth]{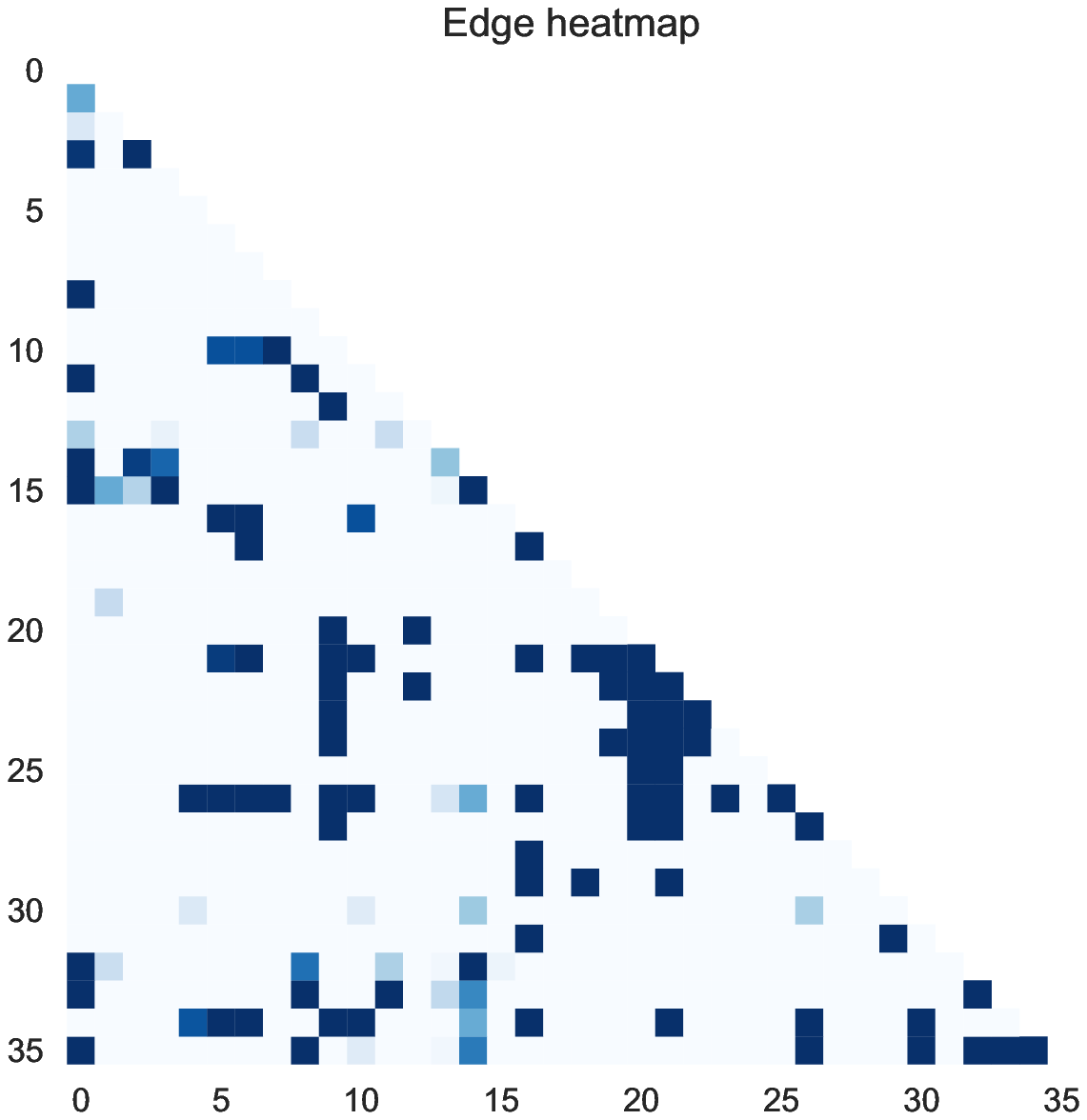}
    \end{subfigure}

    \begin{subfigure}[t]{0.33\textwidth}
        \includegraphics[width=\textwidth]{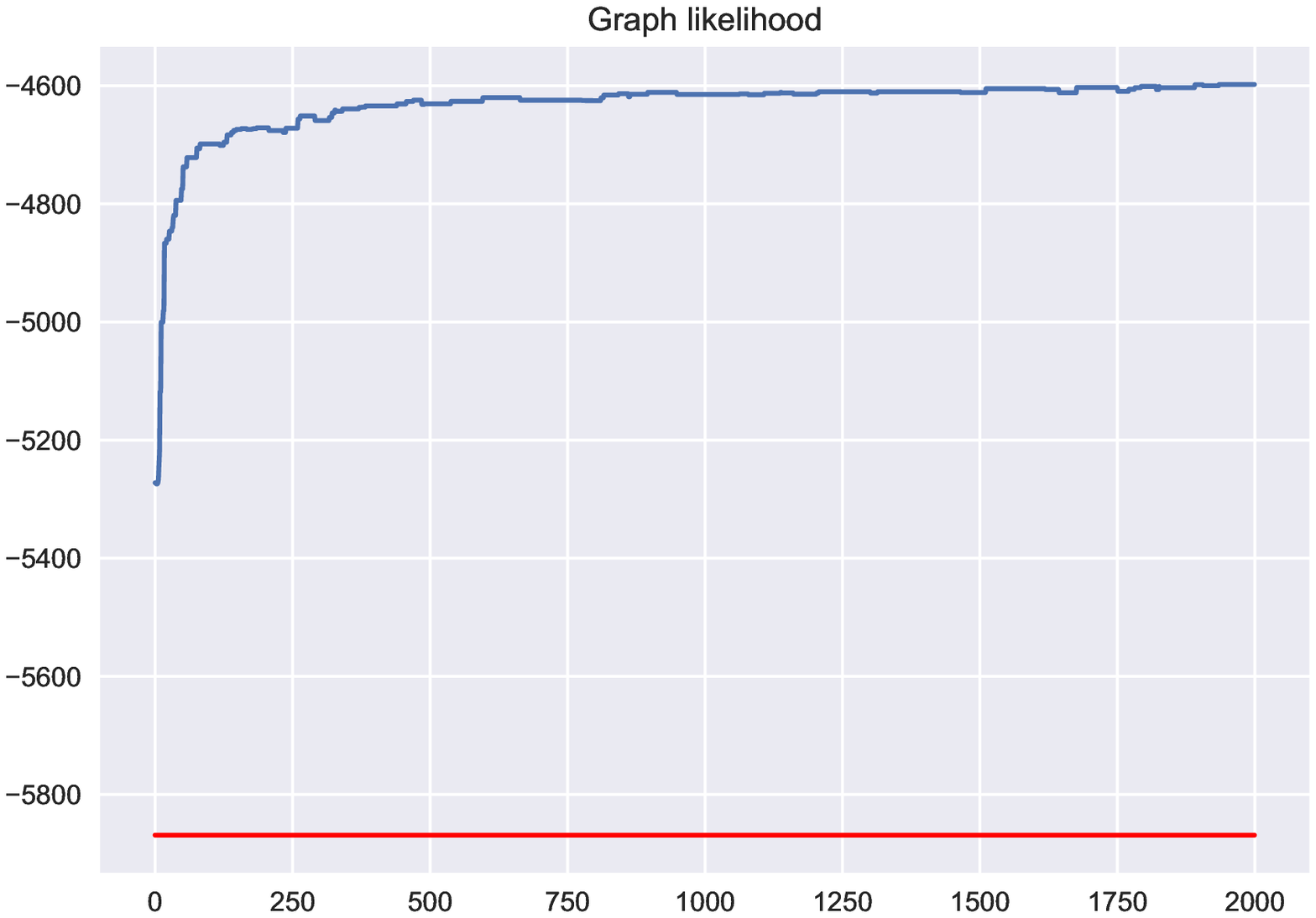}
    \end{subfigure}
    ~
    \begin{subfigure}[b]{0.33\textwidth}
        \includegraphics[width=\textwidth]{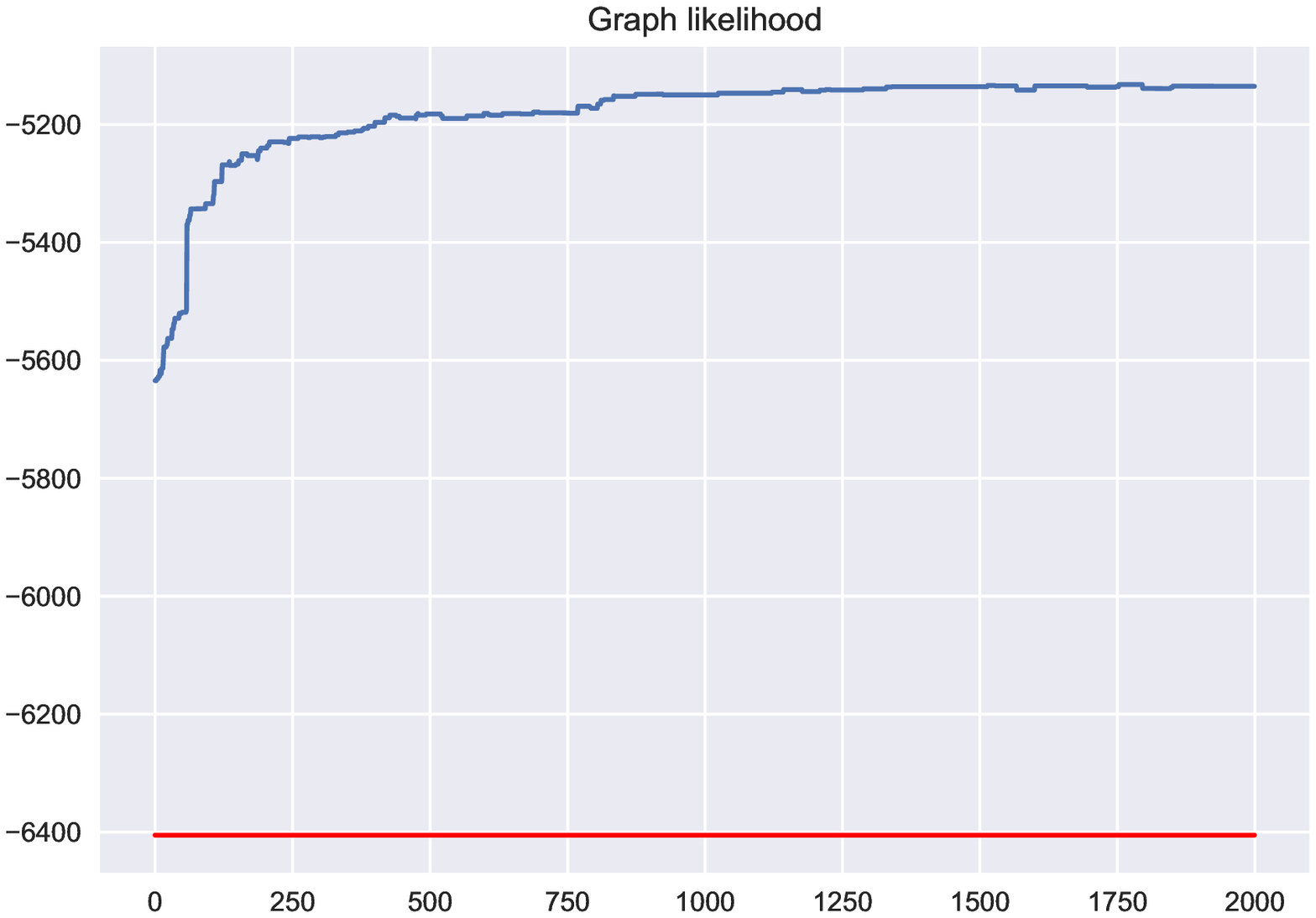}
    \end{subfigure}

    \begin{subfigure}[t]{0.33\textwidth}
        \includegraphics[width=\textwidth]{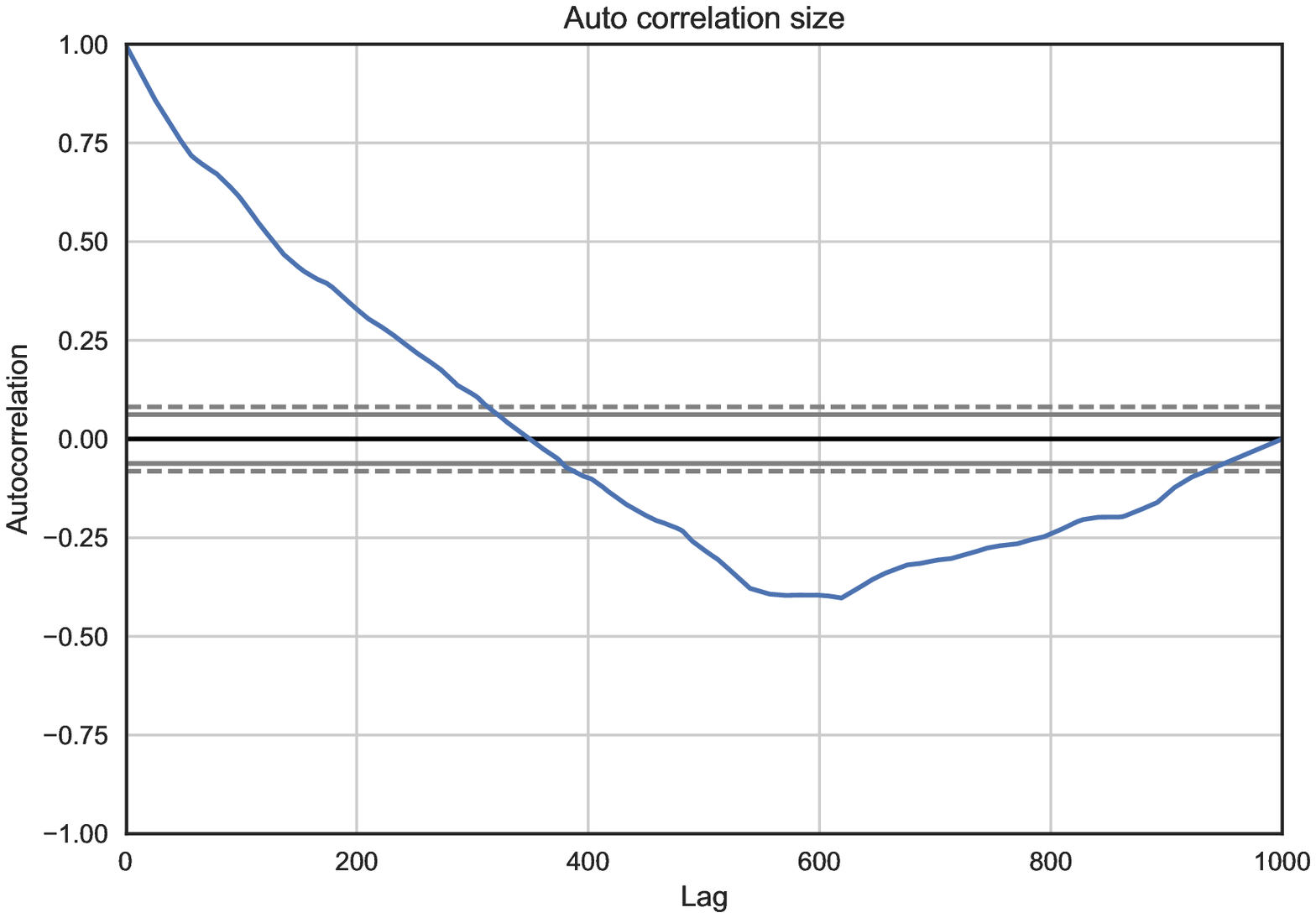}
    \end{subfigure}
    ~
    \begin{subfigure}[t]{0.33\textwidth}
        \includegraphics[width=\textwidth]{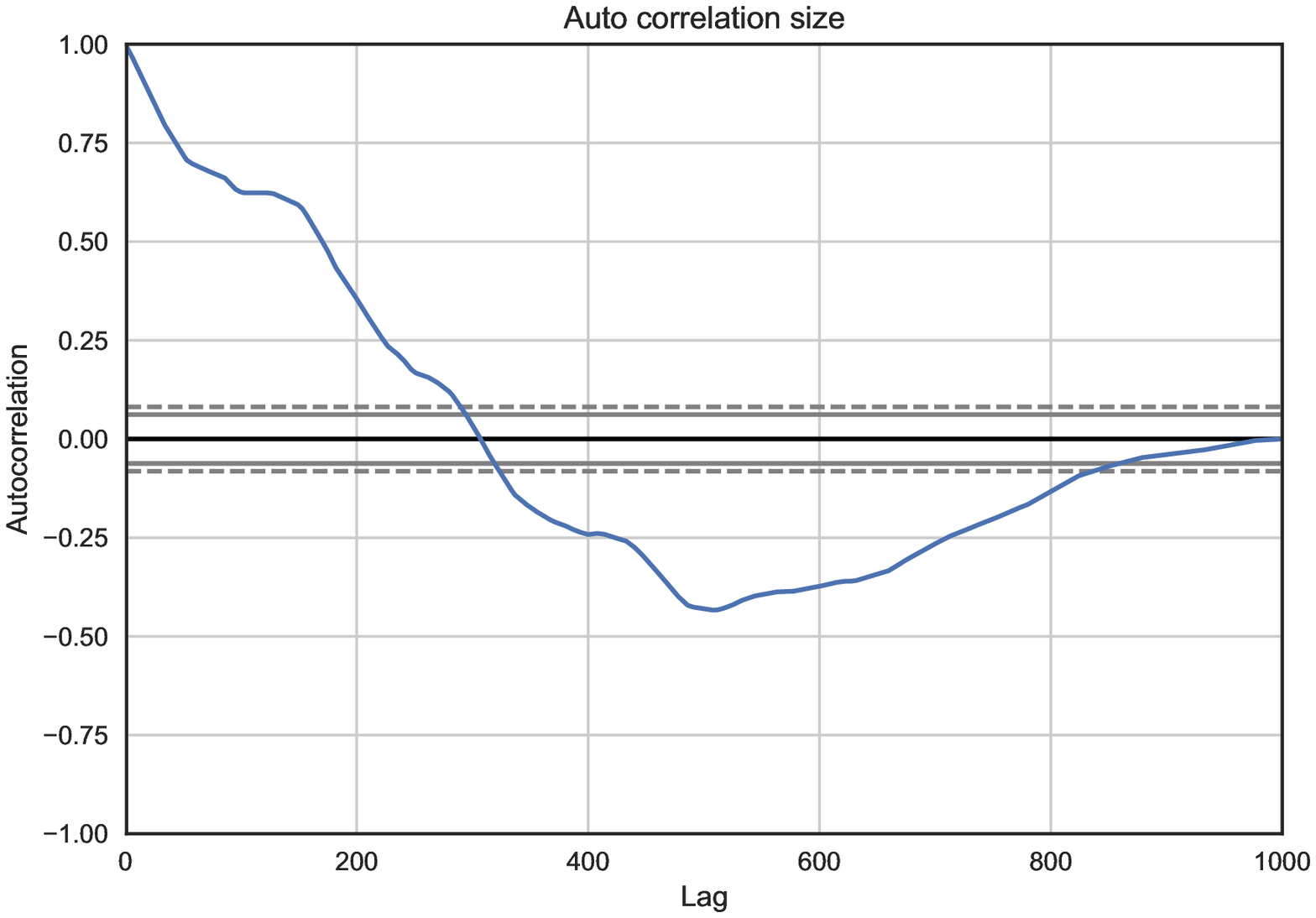}
    \end{subfigure}
    \caption{Walking/Nordic walking.}
    \label{fig:walking}
\end{figure}
% subsection two_classes_with_equal_graphs (end)

% subsection simulations (end)

\subsection{Realistic data: walking and Nordic walking} % (fold)
\label{sub:real_data_example}
This dataset was first introduced in \cite{reiss2012creating,reiss2012introducing} where eighteen different physical activities were measured on eight subjects.
The data was collected at a rate of 100Hz using three wireless IMUs (inertial measurement units) placed on strategic location at the body;
one IMU over the wrist on the dominant arm, one IMU on the chest and one IMU on the dominant side's ankle.
The measurements were then summarized in $52$-dimensional vectors.
However, some of these measurements were categorical and thus dropped here yielding a reduction in dimension to $p=36$.
For this study, we selected the subject  referred to as \emph{Subject1}  and study the two types of activities, \emph{walking} to \emph{Nordic walking} thereby determining a two-class classification problem.
Walking was performed outside
%with moderate to brisk pace
with a speed of 4-6km/h, according to what was suitable for the subject.
Nordic walking was performed on asphaltic terrain, using asphalt pads on the walking poles.

In total the full dataset contains $\approx12000$ data vectors, which enabled us to generate classification data with a similar strategy as in the synthetic data by sampling without replacement 10 replicates of 50 test and $n=40$ training samples from the full dataset for each class.
Consequently, we ignore the possible time dependence between the samples which we could expect.
We set the hyper parameters $\niwnnu, \niwiwtau, \niwiwalpha$  and $\niwnmu_\class$ analogously to the synthetic data examples above.
%For the centroid parameter hyper parameter, since the true centroids are unknown, for $\niwnmu_\class$ we used the empirical mean computed separately in each of the two classes.
The results from the PG sampler for the two classes are found in Figure \ref{fig:walking}.
From the log-likelihood plots, we deduce that the PG sampler seem to have reach a stationary distribution after about 1000 iterations.
The results of correct classification are found in Table \ref{tab:classification-results}.
We see that also in this realistic scenario where the underlying graph structure is unknown, our suggested particle Gibbs BMA classifier outperforms the standard predictive classifier as well as the out-of-the-box classifiers.

% subsection real_data_example (end)

\section{Conclusion} % (fold)
\label{sec:summary}
% -*- root: main.tex -*-
We have constructed a theoretical framework for Bayesian predictive classification where the model  uncertainty is taken into account, by regarding the graph structure underlying the data as random and by using the BMA approach. %for the underlying structure in a predictive classification setting.
 % and by incorporating it into the classification process.
Specifically, at the averaging step of the BMA, each predictive posterior for a fresh  $\testdatapnt^{\testdatasize+1}$ was weighted by the graph posterior $p(\graph[\class] \mid \testdataset[\testdatasize_{\class}])$.

Computing $p(\graph[\class] \mid \testdataset[\testdatasize_{\class}])$, except for toy problems, remains a hard task even for moderate $\p$.
Our solution to this problem was to use an approximation obtained by the particle Gibbs sampler suggested in \cite{nontempspmcmc}.

As a part of the graph posterior specification, we have derived the hyper normal inverse Wishart law and showed that it has the strong hyper Markov property.
We also showed that the posterior predictive distribution for $\testdatapnt^{\testdatasize+1}$, obtained by marginalising out the parameters following the hyper normal inverse Wishart law, has the non-central graph \emph{t}-distribution and concluded that this distribution is Markov.
%The obtained predictive distribution was then weighted by the corresponding graph probability to build the BMA classifier.

The results of the numerical study demonstrated that the particle Gibbs BMA approach to predictive classification indeed improved performance accuracy of the standard Bayesian predictive classifier (corresponding to a point mass at the complete graph) for four synthetic datasets as well as for the realistic walking/Nordic walking dataset, where the true underlying graph structure was unknown.

A comparison with 11 standard out-of-the-box classifiers was also performed, showing that our suggested BMA predictive classifier has systematically superior performance.

Though promising results of the suggested methodology were obtained, a deeper theoretical analysis of the computational complexity along with further numerical validation is needed; for example other more flexible graph priors for both the graph and the parameters are interesting subjects for further study, which is of special importance for applications in high-dimensional problems.

\appendix
\section{} % (fold)
\label{appendix:a}
% -*- root: main.tex -*-
\subsection{Derivation of \eqref{eq:gnw}} % (fold)
\label{sec:gnw}
% section  (end)
Since $\covmat \sim \iwishsymb_{\graph}(\niwiwtau, \niwiwalpha)$, it holds that 
%\begin{align*}
    $\sum_{\jtnode \in \jtnodeset} [\covmat_{\jtnode}]^0 - \sum_{\sep \in \sepset} [\covmat_{\sep}]^0 =\covmat$,
%\end{align*}
which implies that both the determinant and the trace in decomposes in the expression below (see e.g. \cite{lauritzen1996})
\begin{align*}
\normsymb (\meanvec | \niwnmu ,\frac{1}{\niwnnu}\covmat) \iwishsymb_{\graph}(\covmat | \niwiwtau, \niwiwalpha)
% & =
% (2\pi) ^ {-\frac{\p}{2}}| \frac{1}{\niwnnu}\covmat |^{-1/2}
%     %\exp\big \{-\frac{1}{2} tr((\niwiwtau + \niwnnu(\meanvec -\niwnmu)(\meanvec -\niwnmu)' )\covmat)\big \},
%     \exp\{-\frac{\niwnnu}{2} (\meanvec -\niwnmu)\covmat^{-1}(\meanvec -\niwnmu)'\big\} \times \iwishsymb_{\graph}(\covmat \mid \niwiwtau, \niwiwalpha) 
% 	\\
& =
(2\pi) ^ {-\frac{\p}{2}}| \frac{1}{\niwnnu}\covmat |^{-1/2}
    %\exp\big \{-\frac{1}{2} tr((\niwiwtau + \niwnnu(\meanvec -\niwnmu)(\meanvec -\niwnmu)' )\covmat)\big \},
    \exp\{-\frac{\niwnnu}{2} tr((\meanvec -\niwnmu)(\meanvec -\niwnmu)'\covmat^{-1})\big\}  \times  \iwishsymb_{\graph}(\covmat \mid \niwiwtau, \niwiwalpha)
	\\
& =
	\frac{\prod_{\jtnode\in \jtnodeset} \normsymb (\meanvec[\jtnode] | \niwnmu_{\jtnode} ,\frac{1}{\niwnnu}\covmat_{\jtnode})}{\prod_{\sep \in \sepset} \normsymb (\meanvec[\sep] | \niwnmu_{\sep} ,\frac{1}{\niwnnu}\covmat_{\sep})}
  \times  \frac{\prod_{\jtnode \in \jtnodeset} \iwishsymb(\covmat_{\jtnode} \mid \niwiwtau_{\jtnode}, \niwiwalpha) }
          {\prod_{\sep \in \sepset} \iwishsymb(\covmat_{\sep} \mid  \niwiwtau_{\sep}, \niwiwalpha)}
	\\
    &=\frac{\prod_{\jtnode \in \jtnodeset} \giwsymb(\meanvec_{\jtnode},\covmat_{\jtnode} \mid \niwnmu_{\jtnode},\niwnnu, \niwiwtau_{\jtnode}, \niwiwalpha) }
          {\prod_{\sep \in \sepset} \giwsymb(\meanvec_{\sep},\covmat_{\sep} \mid \niwnmu_{\sep},\niwnnu, \niwiwtau_{\sep}, \niwiwalpha)}
	\\
	&=\giwsymb_{\graph}(\meanvec,\covmat \mid \niwnmu,\niwnnu, \niwiwtau, \niwiwalpha).
\end{align*}
%where $\graphnodeset_{\jtnode}$ is defined as the number of nodes in clique $\jtnode$, and similarly for $\sep$.

\subsection{Derivation of \eqref{eq:graphpred}} % (fold)
\label{sub:derivation_of_graphpred}
\cite[p.~1296]{dawid1993} showed that the predictive distribution for strong hyper Markov laws are Markov so that

\begin{align*}
f(\testdatapnt^{\testdatasize+1} | \testdatapnt^{(\testdataindset[\class])},\class,{\bf {\mathcal  C}}^n,\graph) &=
    \int_{\precmatrestr[{\graph}]} \int_{\mathbb R^{p}}
 f(\testdatapnt^{\testdatasize+1} | \meanvec,\covmat )
 f(\meanvec,\covmat| \graph ,\testdataset[{\testdataindset[\class]}] ) d\meanvec d\covmat\\
 &=
 \int_{\precmatrestr[{\graph}]} \int_{\mathbb R^{p}}
\normsymb(\testdatapnt^{\testdatasize + 1}|\meanvec, \covmat)\giwsymb_{\graph}(\meanvec, \covmat | \niwnmustsymb, \niwnnustsymb ,\niwiwtaustsymb, \niwiwalphastsymb) d\meanvec d\covmat\\
 &=
\frac{
\prod_{\jtnode \in \jtnodeset} \int_{\precmatrestr[{\graph}]} \int_{\mathbb R^{p}}
 \normsymb(\testdatapnt^{\testdatasize + 1}|\meanvec[\jtnode], \covmat[\jtnode])\giwsymb(\meanvec[\jtnode], \covmat[\jtnode] | \niwnmustsymb[\jtnode], \niwnnustsymb ,\niwiwtaustsymb[\jtnode], \niwiwalphastsymb) d\meanvec[\jtnode] d\covmat[\jtnode]
}{
\prod_{\sep \in \sepset}
\int_{\precmatrestr[{\graph}]} \int_{\mathbb R^{p}}
 \normsymb(\testdatapnt^{\testdatasize + 1}|\meanvec[\sep], \covmat[\sep])\giwsymb(\meanvec[\sep], \covmat[\sep] | \niwnmustsymb[\sep], \niwnnustsymb ,\niwiwtaustsymb[\sep], \niwiwalphastsymb) d\meanvec[\sep] d\covmat[\sep]
}\\
 &= \tdenssymb_{\graph[]}(\testdatapnt^{\testdatasize + 1}|\tpreddfsymb, \niwnmustsymb, \tpredprecsymb),
\end{align*}
thus we can write
\begin{align*}
f(\testdatapnt^{\testdatasize+1} | \testdatapnt^{(\testdataindset[\class])},\class,\bf{\mathcal  C}^n) &=
  \sum_{\graph \in \graphsp} \int_{\precmatrestr[{\graph}]} \int_{\mathbb R^{p}}
 f(\testdatapnt^{\testdatasize+1} | \meanvec,\covmat,\graph )
 f(\meanvec,\covmat, \graph | \testdataset[{\testdataindset[\class]}] ) d\meanvec d\covmat  \\
&=
\sum_{\graph \in \graphsp} \int_{\precmatrestr[{\graph}]} \int_{\mathbb R^{p}}
 f(\testdatapnt^{\testdatasize+1} | \meanvec,\covmat,\graph )
 f(\meanvec,\covmat| \graph ,\testdataset[{\testdataindset[\class]}] ) f( \graph | \testdataset[{\testdataindset[\class]}] ) d\meanvec d\covmat\\
 &=
\sum_{\graph \in \graphsp} p( \graph | \testdataset[{\testdataindset[\class]}] ) \int_{\precmatrestr[{\graph}]} \int_{\mathbb R^{p}}
 f(\testdatapnt^{\testdatasize+1} | \meanvec,\covmat )
 f(\meanvec,\covmat| \graph ,\testdataset[{\testdataindset[\class]}] ) d\meanvec d\covmat\\
 &=
\sum_{\graph \in \graphsp} p( \graph | \testdataset[{\testdataindset[\class]}] ) f(\testdatapnt^{\testdatasize+1} | \testdatapnt^{(\testdataindset[\class])},\class,{\bf {\mathcal  C}}^n,\graph)\\
%  &=
%  \sum_{\graph \in \graphsp} p( \graph | \testdataset[{\testdataindset[\class]}] ) \int_{\precmatrestr[{\graph}]} \int_{\mathbb R^{p}}
% \normsymb(\testdatapnt^{\testdatasize + 1}|\meanvec, \covmat)\giwsymb_{\graph}(\meanvec, \covmat | \niwnmustsymb, \niwnnustsymb ,\niwiwtaustsymb, \niwiwalphastsymb) d\meanvec d\covmat\\
%  &=
% \sum_{\graph \in \graphsp} p( \graph | \testdataset[{\testdataindset[\class]}] )
% \frac{
% \prod_{\jtnode \in \jtnodeset} \int_{\precmatrestr[{\graph}]} \int_{\mathbb R^{p}}
%  \normsymb(\testdatapnt^{\testdatasize + 1}|\meanvec[\jtnode], \covmat[\jtnode])\giwsymb(\meanvec[\jtnode], \covmat[\jtnode] | \niwnmustsymb[\jtnode], \niwnnustsymb ,\niwiwtaustsymb[\jtnode], \niwiwalphastsymb) d\meanvec[\jtnode] d\covmat[\jtnode]
% }{
% \prod_{\sep \in \sepset}
% \int_{\precmatrestr[{\graph}]} \int_{\mathbb R^{p}}
%  \normsymb(\testdatapnt^{\testdatasize + 1}|\meanvec[\sep], \covmat[\sep])\giwsymb(\meanvec[\sep], \covmat[\sep] | \niwnmustsymb[\sep], \niwnnustsymb ,\niwiwtaustsymb[\sep], \niwiwalphastsymb) d\meanvec[\sep] d\covmat[\sep]
% }\\
 &= \sum_{\graph \in \graphsp}p(\graph | \testdataset[{\testdataindset[\class]}])\tdenssymb_{\graph[]}(\testdatapnt^{\testdatasize + 1}|\tpreddfsymb, \niwnmustsymb, \tpredprecsymb).
\end{align*}

\subsection{Derivation of \eqref{eq:graphmarg}} % (fold)
\label{sub:derivation_of_graphmarg}
We super script a density by $^\ast$ to denote its unnormalized version. 
Since the hyper inverse Wishart is a conjugate prior for the normal model, we have
\begin{align*}
    f(\testdataset[{\testdataindset[\class]}] |\graph) &= 
    \int_{\precmatrestr[{\graph}]} \int_{\mathbb R^{p}}
\normsymb(\testdataset[{\testdataindset[\class]}]|\meanvec, \covmat)\giwsymb_{\graph}(\meanvec, \covmat | \niwnmu, \niwnnu ,\niwiwtau, \niwiwalpha) d\meanvec d\covmat\\
&=    (2\pi)^{-n\p/2}\frac{1}{\niwconstsymb_{\graph}(\niwiwalpha, \niwiwtau, \niwnnu)} \int_{\precmatrestr[{\graph}]} \int_{\mathbb R^{p}}
\normsymb^\ast(\testdataset[{\testdataindset[\class]}]|\meanvec, \covmat)\giwsymb_{\graph}^\ast(\meanvec, \covmat | \niwnmu, \niwnnu ,\niwiwtau, \niwiwalpha) d\meanvec d\covmat\\
&=    (2\pi)^{-n\p/2}\frac{1}{\niwconstsymb	_{\graph}(\niwiwalpha, \niwiwtau, \niwnnu)} \int_{\precmatrestr[{\graph}]} \int_{\mathbb R^{p}}
\giwsymb_{\graph}^\ast(\meanvec, \covmat|\niwnmustsymb, \niwnnustsymb ,\niwiwtaustsymb, \niwiwalphastsymb) d\meanvec d\covmat\\
&= (2\pi)^{-n\p/2}\frac{\niwconstsymb_{\graph}(\niwiwalphastsymb, \niwiwtaustsymb, \niwnnustsymb)}{\niwconstsymb_{\graph}(\niwiwalpha, \niwiwtau, \niwnnu)}.
\end{align*}

% Expression \eqref{eq:graphmarg} has a closed form specified by the ratio of two normalizing constants corresponding to $\gwsymb_{\graph}(\niwnmu,\niwnnu, \niwiwtau, \niwiwalpha)$ and $\gwsymb_{\graph}(\niwnmustsymb,\niwnnustsymb, \niwiwtaustsymb, \niwiwalphastsymb)$, respectively
% \begin{align}
%      f(\testdataset[\testdataindset] |\graph) = \frac{1}{(2\pi)^{\testdatasize\p/2}}\frac{\wishconstsymb_{\graph[]}(\niwiwalphastsymb, \niwiwtaustsymb )}{\wishconstsymb_{\graph[]}(\niwiwalpha, \niwiwtau)},
% \end{align}
% where
% \begin{align}
%     \wishconstsymb_{\graph[]}(\niwiwalphastsymb, \niwiwtaustsymb ) = \frac{\prod_{\jtnode \in \jtnodeset} \wishconstsymb_{\graph[\jtnode]}(\niwiwalphastsymb, \niwiwtaustsymb_{\jtnode} ) }
%                                                                           {\prod_{\sep \in \sepset} \wishconstsymb_{\graph[\sep]}(\niwiwalphastsymb, \niwiwtaustsymb_{\sep} )}
% \end{align}
% where
% \begin{align}
%         J_{\graph[]}(\niwiwalphastsymb, \niwiwtaustsymb,\niwiwalpha, \niwiwtau ) =\frac{\wishconstsymb_{\graph[]}(\niwiwalphastsymb, \niwiwtaustsymb )}{\wishconstsymb_{\graph[]}(\niwiwalpha, \niwiwtau)}
% \end{align}

% subsection derivation_of_eqref (end)

% subsection derivation_of_ref (end)
% section  (end)
% section simulation_study (end)

% \section{Predictive classification} % (fold)
% \label{sec:predictive_classification}
% \input{predictive_classification.tex}

% section realistic_data (end)
\bibliographystyle{abbrvnat.bst}
\bibliography{allbib}

\begin{thebibliography}{31}
\providecommand{\natexlab}[1]{#1}
\providecommand{\url}[1]{\texttt{#1}}
\expandafter\ifx\csname urlstyle\endcsname\relax
  \providecommand{\doi}[1]{doi: #1}\else
  \providecommand{\doi}{doi: \begingroup \urlstyle{rm}\Url}\fi

\bibitem[Andrieu et~al.(2010)Andrieu, Doucet, and
  Holenstein]{andrieu2010particle}
C.~Andrieu, A.~Doucet, and R.~Holenstein.
\newblock {Particle Markov chain Monte Carlo methods}.
\newblock \emph{Journal of the Royal Statistical Society: Series B (Statistical
  Methodology)}, 72\penalty0 (3):\penalty0 269--342, 2010.

\bibitem[Bernardo and Smith(2000)]{BayesianTheory}
J.~Bernardo and A.~Smith.
\newblock \emph{Bayesian Theory}.
\newblock Wiley, Chichester, 2000.

\bibitem[Byrne and Dawid(2015)]{StructuralMarkov}
S.~Byrne and A.~P. Dawid.
\newblock {Structural Markov graph laws for Bayesian model uncertainty}.
\newblock \emph{Annals of Statistics}, 43\penalty0 (4):\penalty0 1647--1681,
  2015.

\bibitem[Chopin and Singh(2015)]{chopin:singh:2015}
N.~Chopin and S.~S. Singh.
\newblock On particle {Gibbs} sampling.
\newblock \emph{Bernoulli}, 21\penalty0 (3):\penalty0 1855--1883, 08 2015.

\bibitem[Clyde and George(2004)]{clyde2004model}
M.~Clyde and E.~I. George.
\newblock Model uncertainty.
\newblock \emph{Statistical science}, pages 81--94, 2004.

\bibitem[Corander et~al.(2013{\natexlab{a}})Corander, Cui, and
  Koski]{corander2013inductive}
J.~Corander, Y.~Cui, and T.~Koski.
\newblock Inductive inference and partition exchangeability in classification.
\newblock In \emph{Algorithmic Probability and Friends. Bayesian Prediction and
  Artificial Intelligence}, pages 91--105. Springer, 2013{\natexlab{a}}.

\bibitem[Corander et~al.(2013{\natexlab{b}})Corander, Cui, Koski, and
  Sir{\'e}n]{corander2013have}
J.~Corander, Y.~Cui, T.~Koski, and J.~Sir{\'e}n.
\newblock {Have I seen you before? Principles of Bayesian predictive
  classification revisited}.
\newblock \emph{Statistics and Computing}, 23\penalty0 (1):\penalty0 59--73,
  Jan 2013{\natexlab{b}}.

\bibitem[Corander et~al.(2013{\natexlab{c}})Corander, Koski, Pavlenko, and
  Tillander]{Corander2013}
J.~Corander, T.~Koski, T.~Pavlenko, and A.~Tillander.
\newblock \emph{Bayesian block-diagonal predictive classifier for Gaussian
  data}, pages 543--551.
\newblock Springer Berlin Heidelberg, Berlin, Heidelberg, 2013{\natexlab{c}}.

\bibitem[Cui et~al.(2016)Cui, Sir{\'e}n, Koski, and Corander]{Cui2016}
Y.~Cui, J.~Sir{\'e}n, T.~Koski, and J.~Corander.
\newblock Simultaneous predictive {Gaussian} classifiers.
\newblock \emph{Journal of Classification}, 33\penalty0 (1):\penalty0 73--102,
  Apr 2016.

\bibitem[Dawid(1979)]{10.2307/2984718}
A.~P. Dawid.
\newblock Conditional independence in statistical theory.
\newblock \emph{Journal of the Royal Statistical Society. Series B
  (Methodological)}, 41\penalty0 (1):\penalty0 1--31, 1979.

\bibitem[Dawid and Fang(1992)]{dawid1992conjugate}
A.~P. Dawid and B.~Q. Fang.
\newblock {Conjugate Bayes discrimination with infinitely many variables}.
\newblock \emph{Journal of multivariate analysis}, 41\penalty0 (1):\penalty0
  27--42, 1992.

\bibitem[Dawid and Lauritzen(1993)]{dawid1993}
A.~P. Dawid and S.~L. Lauritzen.
\newblock Hyper {Markov} laws in the statistical analysis of decomposable
  graphical models.
\newblock \emph{The Annals of Statistics}, 21\penalty0 (3):\penalty0
  1272--1317, 1993.

\bibitem[Dempster(1972)]{1972}
A.~P. Dempster.
\newblock Covariance selection.
\newblock \emph{Biometrics}, 28\penalty0 (1):\penalty0 pp. 157--175, 1972.

\bibitem[Geisser(1964)]{geisser1964posterior}
S.~Geisser.
\newblock Posterior odds for multivariate normal classifications.
\newblock \emph{Journal of the Royal Statistical Society. Series B
  (Methodological)}, pages 69--76, 1964.

\bibitem[Geisser(1966)]{geisser1966predictive}
S.~Geisser.
\newblock Predictive discrimination.
\newblock \emph{Krishnajah, P.R. ed. Multivariate analysis}, 149:\penalty0 163,
  1966.

\bibitem[Geisser(1993)]{geisserpredictive}
S.~Geisser.
\newblock \emph{{Predictive inference: An introduction.}}
\newblock Chapman \& Hall, London, 1993.

\bibitem[Green and Thomas(2013)]{Green01032013}
P.~J. Green and A.~Thomas.
\newblock Sampling decomposable graphs using a {M}arkov chain on junction
  trees.
\newblock \emph{Biometrika}, 100\penalty0 (1):\penalty0 91--110, 2013.

\bibitem[Kass and Raftery(1995)]{kass1995bayes}
R.~E. Kass and A.~E. Raftery.
\newblock Bayes factors.
\newblock \emph{Journal of the american statistical association}, 90\penalty0
  (430):\penalty0 773--795, 1995.

\bibitem[Lauritzen(1996)]{lauritzen1996}
S.~L. Lauritzen.
\newblock \emph{Graphical Models}.
\newblock Oxford University Press, 1996.

\bibitem[Madigan and Raftery(1994)]{10.2307/2291017}
D.~Madigan and A.~E. Raftery.
\newblock Model selection and accounting for model uncertainty in graphical
  models using {Occam's} window.
\newblock \emph{Journal of the American Statistical Association}, 89\penalty0
  (428):\penalty0 1535--1546, 1994.

\bibitem[Madigan et~al.(1995)Madigan, York, and
  Allard]{madigan1995graphicalmodels}
D.~Madigan, J.~York, and D.~Allard.
\newblock Bayesian graphical models for discrete data.
\newblock \emph{International Statistical Review / Revue Internationale de
  Statistique}, 63\penalty0 (2):\penalty0 215--232, 1995.

\bibitem[Nyman et~al.(2016)Nyman, Xiong, Pensar, and
  Corander]{nyman2016marginal}
H.~Nyman, J.~Xiong, J.~Pensar, and J.~Corander.
\newblock Marginal and simultaneous predictive classification using stratified
  graphical models.
\newblock \emph{Advances in Data Analysis and Classification}, 10\penalty0
  (3):\penalty0 305--326, 2016.

\bibitem[Olsson et~al.(2018{\natexlab{a}})Olsson, Pavlenko, and Rios]{cta}
J.~Olsson, T.~Pavlenko, and F.~L. Rios.
\newblock {Sequential sampling of junction trees for decomposable graphs}.
\newblock \emph{ArXiv e-prints}, 2018{\natexlab{a}}.

\bibitem[Olsson et~al.(2018{\natexlab{b}})Olsson, Pavlenko, and
  Rios]{nontempspmcmc}
J.~Olsson, T.~Pavlenko, and F.~L. Rios.
\newblock {Bayesian inference in decomposable graphical models using sequential
  Monte Carlo methods}.
\newblock \emph{ArXiv e-prints}, 2018{\natexlab{b}}.

\bibitem[Paz and Pearl(1985)]{graphoids}
A.~Paz and J.~Pearl.
\newblock \emph{Graphoids: A graph based logic for reasoning about relevance
  relations}.
\newblock UCLA, 1985.

\bibitem[Reiss and Stricker(2012{\natexlab{a}})]{reiss2012creating}
A.~Reiss and D.~Stricker.
\newblock Creating and benchmarking a new dataset for physical activity
  monitoring.
\newblock In \emph{Proceedings of the 5th International Conference on Pervasive
  Technologies Related to Assistive Environments}, page~40. ACM,
  2012{\natexlab{a}}.

\bibitem[Reiss and Stricker(2012{\natexlab{b}})]{reiss2012introducing}
A.~Reiss and D.~Stricker.
\newblock Introducing a new benchmarked dataset for activity monitoring.
\newblock In \emph{Wearable Computers (ISWC), 2012 16th International Symposium
  on Wearable Computers}, pages 108--109. IEEE, 2012{\natexlab{b}}.

\bibitem[Ripley(2007)]{ripley2007pattern}
B.~D. Ripley.
\newblock \emph{Pattern recognition and neural networks}.
\newblock Cambridge university press, 2007.

\bibitem[Speed and Kiiveri(1986)]{10.2307/2241271}
T.~P. Speed and H.~T. Kiiveri.
\newblock {Gaussian Markov} distributions over finite graphs.
\newblock \emph{The Annals of Statistics}, 14\penalty0 (1):\penalty0 138--150,
  1986.

\bibitem[Thomas and Green(2009)]{doi:10.1198/jcgs.2009.07129}
A.~Thomas and P.~J. Green.
\newblock Enumerating the junction trees of a decomposable graph.
\newblock \emph{Journal of Computational and Graphical Statistics}, 18\penalty0
  (4):\penalty0 930--940, 2009.

\bibitem[Wormald(1985)]{countchordal}
N.~C. Wormald.
\newblock Counting labelled chordal graphs.
\newblock \emph{Graphs and Combinatorics}, 1\penalty0 (1):\penalty0 193--200,
  1985.

\end{thebibliography}

\end{document}